\documentclass[11pt]{article}

\usepackage[english]{babel}
\usepackage[normalem]{ulem}
\usepackage{xcolor}

\usepackage{amsmath, amssymb, amsthm}

\usepackage{graphicx}
\usepackage{lineno}
\usepackage{algorithm2e}
\usepackage{textcomp} 
\usepackage[export]{adjustbox}
\usepackage{subcaption}
\usepackage{booktabs}
\usepackage{caption}
\usepackage{url}
\usepackage{nomencl}
\usepackage{multirow}
\usepackage{pifont}
\usepackage{xr-hyper}
\usepackage{csquotes}
\usepackage{setspace}
\usepackage{authblk}
\usepackage{pgf}
\usepackage{adjustbox}
\usepackage[percent]{overpic}
\usepackage[most]{tcolorbox}
\usepackage{booktabs}

\usepackage{hyperref}
\hypersetup{colorlinks=true, linkcolor=blue, citecolor=blue, urlcolor=blue, linkcolor=blue}

\newcommand{\cA}{\mathcal{A}}

\newcommand{\cC}{\mathcal{C}}
\newcommand{\cD}{\mathcal{D}}
\newcommand{\cE}{\mathcal{E}}
\newcommand{\cF}{\mathcal{F}}

\newcommand{\cI}{\mathcal{I}}
\newcommand{\cJ}{\mathcal{J}}

\newcommand{\cN}{\mathcal{N}}

\newcommand{\cP}{\mathcal{P}}

\newcommand{\cR}{\mathcal{R}}
\newcommand{\cS}{\mathcal{S}}

\newcommand{\cV}{\mathcal{V}}
\newcommand{\cW}{\mathcal{W}}
\newcommand{\cX}{\mathcal{X}}

\newcommand{\Iplus}{\mathcal{I}_{>\tau}}

\newcommand{\E}{\mathbb{E}}

\newcommand{\ess}{\mathrm{ESS}}

\DeclareMathOperator*{\argmin}{arg\ min}
\DeclareMathOperator*{\Bern}{Bernoulli}

\DeclareMathOperator*{\Leb}{Leb}

\newtheorem{lemma}{Lemma}[section]
\newtheorem{proposition}[lemma]{Proposition}

\newtheorem{remark}[lemma]{Remark}

\newtcolorbox{siframe}[1][]{
  breakable,
  enhanced,
  colback=white,
  colframe=black,
  boxrule=0.5pt,
  arc=0pt,
  outer arc=0pt,
  left=10pt, right=10pt, top=10pt, bottom=10pt,
  #1
}

\usepackage[
    backend=biber,
    style=numeric-comp,
    maxcitenames=2,
    maxbibnames=12,
    sorting=none
]{biblatex}

\makenomenclature

\usepackage[font=small]{caption}
\usepackage[margin=1in]{geometry}

\title{Nearest Neighbor Functional Curve Completion for Intraday Probabilistic Revisions of Renewable Generation Forecasts}

\date{\today}

\author[1,2,*]{Guillermo Terr\'en-Serrano}
\author[3]{Michael Ludkovski}

\affil[1]{\small{Environmental Markets Lab (emLab), University of California Santa Barbara, Santa Barbara, CA, USA.}}
\affil[2]{\small{Environmental Studies Program, University of California Santa Barbara, Santa Barbara, CA, USA.}}
\affil[3]{\small{Statistics and Applied Probability, University of California Santa Barbara, Santa Barbara, CA, USA.}}
\affil[*]{\small{Correspondence: \url{guillermoterren@ucsb.edu}}}

\begin{document}
\addtocontents{toc}{\protect\setcounter{tocdepth}{-10}}
\maketitle






\begin{abstract}
We propose a nonparametric approach that fuses deterministic initial forecasts and a partial realization into a  probabilistic  forecast using the framework of functional \emph{curve} completion. Our motivation is to construct  intra-day forecasts for renewable energy generation which is a key operational task in modern power grids. The aim is to combine the day-ahead (DA) forecast and the morning realized generation profile to create trajectories for the rest of the day, quantifying operational uncertainty. Starting with a large database of complete forecast-actuals pairs across many assets, we adopt a functional point of view and construct a weighted neighborhood based on similarity metrics for the observed partial curve, the DA forecast, the day and the geographical  location. The selected complete nearest-neighbor curves provide a full predictive distribution, including functional focal envelopes and predictive bands based on functional depth. The hyperparameters governing neighborhood size, distances, and fusion rates are tuned across forecast update times and directly target probabilistic calibration in terms of predictive probability integral transform. Three case studies using synthetic wind and solar assets, as well as real-life regional wind generation showcase that our algorithm is able to achieve excellent calibration.
\end{abstract}

\textbf{Keywords}---\textit{Functional Completion; Functional Nearest Neighbors;  Intra-day  Renewable Forecasting; Probabilistic Forecast Revision.}

\section{Introduction}





We consider the task of  dynamically revising Variable Renewable Energy (VRE) forecasts in power grids \cite{messner2019online}. Our motivation is to convert prior Day-Ahead (DA) forecasts into \emph{intra-day probabilistic} versions as realized data arrives \cite{uniejewski2019understanding}. The forecasting task we tackle is distinguished by the fusion of the original forecast \emph{curve} and the partial realization of the daily generation profile, embedding the problem into the framework of functional completion \cite{liebl2019partially}. 

A DA forecast is typically issued in the evening (e.g.,~6pm) of the day preceding the operating day and covers the full 24-hour horizon of that day, from midnight to midnight \cite{weron2014electricity}. The initial DA forecast forms the backbone of power grid operations, driving the security-constrained unit commitment that provides the baseline dispatch schedules for the committed units on the following day. As the operational day unfolds, the system operator then undertakes Security-Constrained Economic Dispatch (SCED), managing any discrepancies between forecasts and actuals. While traditional SCED is done on a limited look-ahead basis (typically 2-4 hours horizon), the increased uncertainty due to high level of renewable generation (both grid-scale and behind-the-meter) and the multi-hour ramp constraints of peaking units is raising the importance of ongoing intra-day planning, which in turn requires intra-day probabilistic forecasts \emph{for the remainder of the day}.

This setting is different from the time-series setup where predictions are generally sought just over the next few time-steps (or a narrow predictive window) and history is viewed through the Markovian lens. In contrast, we operate in the functional domain, mapping partial \emph{curves} to future complete \emph{curves}. Moreover, given the numerous nuances of VRE generation profiles, we eschew generative auto-regressive approaches and instead advocate for a non-parametric framework where a data-driven fully probabilistic functional forecast is constructed by leveraging a historical database of full curves. 



We construct well-calibrated probabilistic  intra-day asset-level revisions solely from the original DA forecasts and partial observations. The resulting functional envelopes achieve excellent coverage as well as PIT metrics, indicating improved sharpness. We demonstrate that assimilating partially observed curves significantly improves accuracy relative to either bootstrapped DA forecasts or ARIMA models. Our method
generalizes effectively across different VRE resources and aggregation levels and offers a new tool for system operator needing to do reserves allocation and real-time energy dispatch.


Zooming out, our method is applicable in generic spatio-temporal problems where a partial curve requires a probabilistic completion. With that broader lens, we use a contextual database of curves to find a functional neighborhood of a target partial curve, where nearest neighbors are based on \emph{labels} associated with each curve. In our case the ``curve'' is a bivariate (forecast, actual) profile. Other domains where our method could be applicable are financial settings (such as predicting trading volume for the trading day given the volume in the morning \cite{shang2022dynamic, wang2026functional}, based on longer-term covariates such as recent market trends), epidemiology (predicting future case counts based on partially complete case counts of a repeated outbreak \cite{bracher2022endemic}, such as seasonal flu, given public health variables), environmental data (counting species across their annual migration window \cite{vandoren2018birdcast}), etc.

The rest of the paper is organized as follows. Section \ref{sec:preview} provides a high level overview of our method and the resulting intra-day probabilistic forecasts, illustrating the outputs of our algorithms and their practical use cases. Section \ref{sec:method} defines our methodology and  
the scoring metrics we use assess the performance of our forecasts and fine-tune them. Section \ref{sec:results} presents the results of our experiments on (i) synthetic wind assets; (ii) synthetic solar assets and (iii) real-life ERCOT zonal wind data. Section \ref{sec:conclude} concludes.

\section{Intra-day Probabilistic Forecasts}\label{sec:preview}

As an example of a concrete setting from present-day markets, the Electric Reliability Council of Texas (ERCOT) operates the Day-Ahead Market (DAM) and the Real-Time Market (RTM) for energy and Ancillary Services. 
The DAM closes at 10:00am on the day before the operating day, and ERCOT relies on DA forecasts of VRE generation and demand based on numerical weather prediction to clear the market \cite{ERCOT_Protocols_Section4}. During the operating day, ERCOT settles the RTM in 15-minute intervals \cite{ERCOT_Protocols_Section6} and executes a SCED every 5 minutes \cite{ERCOT_Protocols_Section11}, using telemetry and short-term load forecasts. In addition, ERCOT performs the Hourly Reliability Unit Commitment (HRUC) to ensure sufficient online capacity, which uses updated load and VRE forecasts to commit additional resources when needed. Resources allocated through HRUC receive compensation through out-of-market uplift payments. For the real-time market operation, a new 24-hour forecast is issued with a 2~hour lead time. 

\begin{remark}
    \label{rem:generalization}
    One may generalize our setup to consider an $\ell$-lead hours forecast for a forecast window $\{\ell, \ldots, \ell+T\}$ of $T$ hours. The most common DA forecast has $\ell=6$ and $T=24$. Presently, ERCOT publishes hourly-updated forecasts for the next $\ell+T=168$ hours, so for example one could consider $\ell=12$ and $T=48$ i.e., noon-stamped forecasts for the following 2 days that are utilized for multi-day reserve and capacity planning. 
\end{remark}

\begin{figure*}[htb!]
    \centering
    \begin{overpic}[width=.75\textwidth, trim = {0cm, 0cm, 0cm, 0cm}]{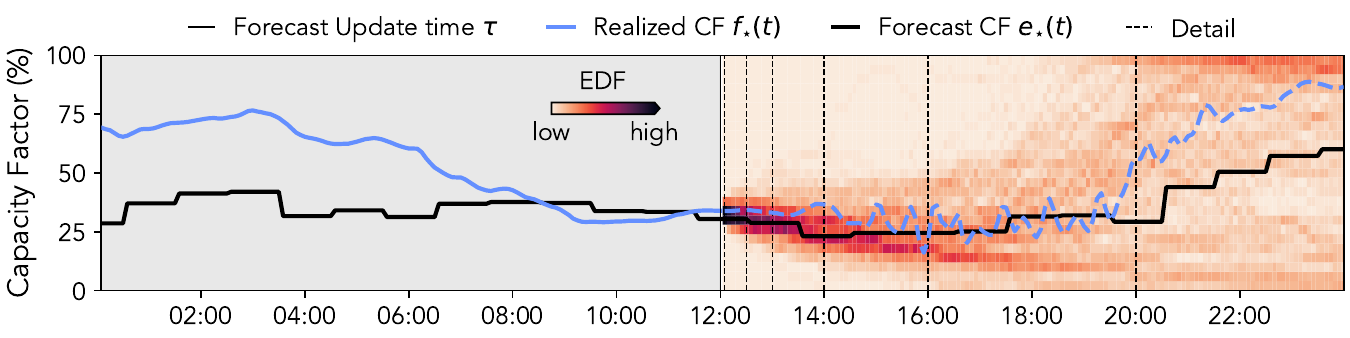}
    \end{overpic}
    \begin{overpic}[width=.7725\textwidth, trim = {-.25cm, 0cm, 0cm, 0cm}]{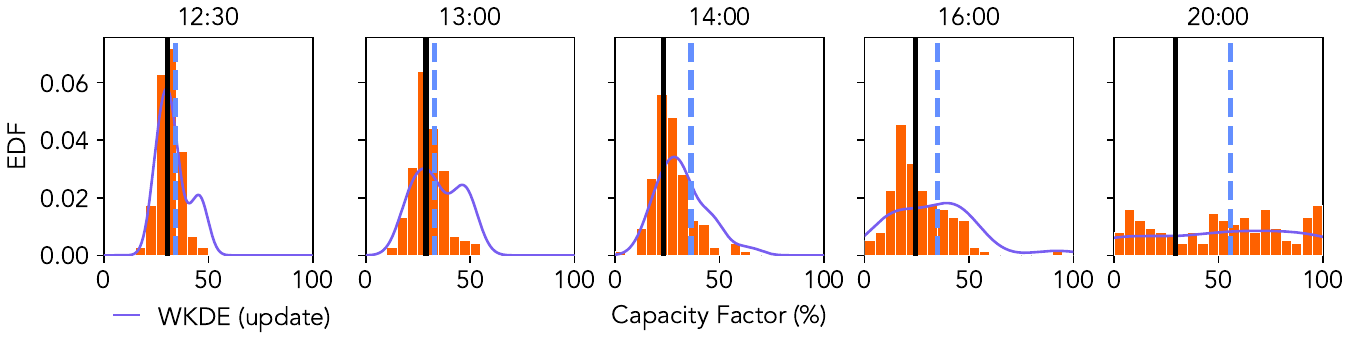}
    \end{overpic}
    \caption{\textbf{Intraday functional probabilistic forecasting.} Wind generation at the Blanco Canyon Wind asset on June 26. The plots show probabilistic forecasts for the afternoon given the realized capacity factors (CF) from 12am until the forecast update time $\tau=12$pm. (a) Heatmap for the predictive empirical density function, darker colors correspond to higher probabilities. The DA forecast is in black, the morning observations are in blue, and the afternoon (unseen from forecast update time $\tau$) generation is in dashed blue. (b) histograms and KDE estimates of the marginal predictive distributions at 12:30pm, 1pm, 2pm, 4pm, and 8pm (these times are indicated by the dashed lines in the top panel). The vertical black and blue lines indicate the DA forecast and the realized generation at that time.}
    \label{fig:fcst_illustration}
\end{figure*}

To preview our approach for integrating the actual observations from the partially completed day with the DA forecast,
Figure~\ref{fig:fcst_illustration} illustrates the proposed intra-day probabilistic forecasting for the remaining hours. We consider an ERCOT wind asset with a realized generation trajectory over the first 12 hours of the day shown by the blue curve. Here and for the remainder of the paper, all generation is expressed in terms of capacity factors (CF), so all quantities are in $[0\%,100\%]$. The respective hour-scale DA forecast (released at 6pm the day before) is shown as the black stepped line. For this particular day, the DA forecast projected medium CFs throughout the day. As it happened, realized generation $f_\star(t)$ was higher during the night and morning and then largely matched the forecast during 8am-12pm, getting to $f_\star (\tau) \approx 30\%$ at 12pm. The proposed algorithm fuses this partially realized curve, the DA forecast, as well as fully realized curves from similar days/assets to output the likely generation profiles over the remaining 12 hours of the day. The orange heatmap visualizes the forecasted densities across time, which come from a weighted ensemble of underlying potential trajectories. These trajectories are our building blocks and present a non-parametric, fully probabilistic, functional, dynamically constructed forecast for the tail portion of the day. 

The resulting mean forecast is shown as the dashed blue line while the histograms in the second row of Figure~\ref{fig:fcst_illustration} present the forecast densities at select future timestamps. Additional curve-based analysis such as functional prediction bands is presented in sections below.

The above forecasting method blends the information from the morning (the blue curve) and from the DA forecast in a non-parametric way. 
As expected, there is higher forecasting precision for the early afternoon (the respective histograms are peaked), and much more uncertainty about the evening, with the marginal predictive distribution at 8pm almost uniform. 


The ensembled curves $f_\star(t)$ are obtained through a Nearest Neighbor filtering procedure that we developed. They all start at the \emph{forecast update time} $\tau$ and cover the remaining 12 hours $s \in [\tau, T]$. The neighborhood is based on defining a similarity metric across forecasts (and similarly but distinctly across realized partial curves), and then ranking curves in the database.

To achieve a forecast revision that is plausible to complete the partial $f_\star$, the similarity metric strongly concentrates on keeping $|f_\star - f|$ very small around $\tau$, witness the ``garden hose squeeze'' in Figure~\ref{fig:fcst_illustration} which creates the desired uncertainty cone. The similarity metric also takes into account the weighted $L^2$ distance of $e_\star$ to $e$ across the entire day. In order to use ``apples'' to forecast ``apples'', we  furthermore takes into account the season and the geographical location of potentially ``neighboring'' asset trajectories. Lastly, to achieve the mean-reversion effect vis-a-vis $e_\star$, we gradually blend each afternoon $f$ with its respective DA forecast $e$.

A major benefit of our method is that it intrinsically respects the $[0,1]$ constraints of CFs and the tendency of wind generation to be persistent at extreme values, such as frequent occurrences of zero generation across many hours, as seen in Figure~\ref{fig:fcst_illustration}~and~\ref{fig:functional_neighbors}. Statistical capture of such multivariate dependence (which is across the joint $(f,e)$ pairs) is known to be extremely difficult. Indeed, while some of the marginal histograms in the second row of Figure~\ref{fig:fcst_illustration} look Gaussian (e.g., at $s=1$pm), many of them are not, and directly obtaining the skewed, zero-inflated distribution at $s=4$pm is intractable in a machine learning setup.  Because we directly employ actual (forecast, realized generation) curves, our predictions are fully data-driven and embed all such nuances.

\paragraph{Applications.}

An intra-day probabilistic forecast updating to quantify asset-level generation uncertainty has multiple applications in power systems, benefiting grid operators, asset owners, energy traders, and transmission operators. Grid operators balance supply and demand by adjusting the energy dispatch and reserve requirements in real-time with intra-day asset-level forecasts \cite{spyrou2022flexibility, spyrou2020value}. Asset owners can minimize imbalance penalties by reducing deviations from their committed DA schedules and optimizing their revenue strategy in both balancing energy and ancillary services markets \cite{munoz2023online,aung2024optimal}. For hybrid asset owners that combine solar or wind generation with co-located storage systems, an intra-day trajectory forecast can support them in planning an optimal energy storage strategy (charge/discharge) to smooth out output variability and exploit price arbitrage opportunities in volatile markets \cite{hodge2018combined,yuan2024hybrid,aung2026battery}, whose intra-day price dynamics are themselves an active modeling target \cite{hirsch2024simulation}. Energy traders can use these forecasts to hedge exposure to uncertainty \cite{hu2021effects} and rebid dynamically to adjust their positions in the RTM \cite{xiao2023integrated}, where probabilistic  intra-day forecasts have been shown to improve trading strategies \cite{hirsch2026probabilistic}. Transmission operators utilize intra-day forecasts to avoid congestion \cite{frysztacki2020modeling} and reduce curtailment by implementing grid flexibility measures \cite{li2020stochastic}. These applications can improve operational efficiency to lower overall production costs by enabling market participants and system operators to make better decisions,  reducing balancing costs, and ultimately improving reliability. In this context, Functional Data Analysis (FDA) provides a theoretical framework to accomplish this purpose \cite{gonzalez2017forecasting}. Eyeing a general methodology, we consider the statistical task of forecast fusion, which is agnostic to the underlying uncertainty source and hence can be equally well applied to solar generation, wind generation, load, etc.

\paragraph{Improving Reliability Unit Commitment.}

Operators are increasingly pivoting to out-of-market interventions to maintain system reliability. In ERCOT, the costs associated with HRUC uplift payments have risen significantly, reaching a peak in 2022 \cite{PotomacEconomics2024}. The growing dependence on uplift payments underscores the need for market mechanisms to adopt a dynamic and adaptive probabilistic framework to manage uncertainty, particularly under extreme tail risks \cite{NERC2024TailRisk}. Much of this uncertainty arises from forecasting errors in load and variable renewable energy (VRE) generation. Therefore, moving away from deterministic intra-day forecasts and using probability distributions for asset-level outcomes would enable stochastic HRUC commitment decisions to better reflect reliability risks at the extremes.


\paragraph{Related Literature.}

Probabilistic forecasting is a well-established topic in power markets \cite{terren2026probabilistic}, with a particularly extensive literature on short-term and intra-day electricity \emph{price} forecasts (EPF) \cite{oconnor2026evolution} and load forecasts \cite{hong2016probabilistic}. Probabilistic forecasting for renewable generation is less studied, and primarily focuses on spatio-temporal methods \cite{li2025probabilistic, arrieta2022spatio}.  While some of the techniques, such as ensemble methods \cite{narajewski2020ensemble}, DA forecast fusion \cite{serafin2027data} and trajectory simulations \cite{puc2026scenario} are similar to our setup, prices form a low-dimensional object and lack the heterogeneous labeling by underlying asset that is a fundamental feature of our model (in tandem, our setting lacks the external dynamic regressors available for EPF, all our labels like asset location are static). The concept of similarity-based selection, introduced for power markets as Pattern Sequence Similarity \cite{alvarez2010energy} is therefore typically restricted to similar day methods \cite{borunda2026advances} or nearest-neighbor marginal forecasting \cite{zhang2016k}. We moreover contrast our model-free paradigm with time-series techniques or generative neural network models \cite{chen2025probabilistic}. 


Closest to us are functional frameworks \cite{shang2011nonparametric, shang2019intraday, elias2022} that approach sequence-to-sequence conditional curve completion using nonparametric techniques \cite{burba2009k, sun2011functional}. In particular, El\'ias et al.~\cite{elias2022} have shown the benefit of using functional depth concepts for electricity load forecasting, while  \cite{terren2025extreme} employed functional depth to identify the day-ahead scenarios that drive operational risk on a realistic grid model. In contrast, the functional Hilbertian ARMAX approach of \cite{gonzalez2017forecasting} cannot complete partial curves, and the parametric functional time series method of \cite{shang2013functional} updates principal component scores as the day progresses but requires re-estimation for each asset. 


\section{Methods}\label{sec:method}

\paragraph{Functional Data.}

We consider a set of historical functional curves of $\cV$-generation, such as $\cV=\cS$ (solar) or $\cV=\cW$ wind generation. In our dataset, these curves represent daily observations made at a 5-minute temporal resolution (i.e.,~of length $T = 24 \times 12 = 288$). We treat these curves as functional elements over the domain $\cI \in \{1,\ldots,T\}$ corresponding to an operational day (0:05 to 0:00) and which are indexed by respective assets $\ell$ and days $d$. The overall collection is denoted $\cF^{(\cV)} = \left\{ f^{(\cV)}_{\ell,d} \in \cC (\cI)~\middle|~ d = 1, \dots, D, \ell=1,\ldots, L^{(\cV)} \right\}$ and consists of $D$ days and $L^{(\cV)}$ assets. 


In parallel, we also have the collection  $\cE^{(\cV)} = \left\{ e^{(\cV)}_{d, \ell} \in \cC (\cI)\middle|~ d = 1, \dots, D, \ell=1,\ldots, L^{(\cV)} \right\}$ of matched historical CF DA forecasts for the same $D$ operational days and assets $\ell$, recorded at 1-hour frequency in an operational day (1am to 12am). We upsample (as piecewise-constant interpolation) these DA forecasts to the same intervals as the realized observations (0:05 to 0:00).

We flatten $\cF = \left\{ f_n \in \cC (\cI) : n = 1, \ldots, N \right\}$ and $\cE = \left\{ e_n \in \cC (\cI) : n = 1, \ldots, N \right\}$ into collections of $N = D \times L$ curves, and omit the superscript $\cV$. Additionally, we track the day of the year $d$ via the set $\cD =\left\{ d_n \in \{1,2,\ldots, 365\}~\middle|~ n = 1, \ldots, N \right\}$ and the real-valued geographic coordinates $\cX = \left\{ \left(x_1, x_2\right)_n~\middle|~ n = 1, \ldots, N \right\}$ with $ x_1 \in [-180, 180]$ (longitude) and $x_2 \in [-90, 90]$ (latitude). Together, $(d_n, x_{1,n}, x_{2,n})$ form the \emph{label} of the respective (forecast, actuals) pair. 

\paragraph{Problem.} 

Our main interest is in completing a CF curve $\{f_\star (t) : t \le \tau\}$ that is observed up to the forecast update time $\tau$, to probabilistically extend it to the remainder of the day $\cI_{>\tau}= [\tau,T]$. The main idea is to construct a nonparametric neighborhood based on similar $f_\ell \in \cC \left( \cI \right)$, blended with the respective DA forecasts $e_{d,\ell}$.

\paragraph{Data.} 

We utilize the 2030 Texas-7k grid model which is an open-source synthetic 6,717-bus grid model that covers the geographic footprint of the ERCOT \cite{birchfield2016grid}. Texas-7k contains 187 wind and 211 solar asset locations, corresponding to a mix of existing physical and proposed plant sites from the 2021 National Laboratory of the Rockies (NLR) PERFORM dataset~\cite{bracher2021}. 

The generation time series for solar and wind assets cover calendar years 2017 and 2018. Actuals have 5-min resolution and were simulated from the Weather Research and Forecasting (WRF) model for wind, and the National Solar Radiation Database for solar, using re-analysis of historical weather. DA forecasts are issued with an 11-hour lead time and have hourly resolution. They are derived from the European Center for Medium-Range Weather Forecasts (ECMWF) 51-member ensemble. We normalize each asset's forecasts and actuals by its nameplate capacity, so that all generation capacity factors (CF) take values in $[0,1]$.


\subsection{Spatiotemporal Asset Similarity}

The forecasting problem takes in the partial CF curve $f_\star (s), s \le \tau$ observed up to forecast update time $\tau$, and the DA forecast $e_\star \in \cC \left( \cI \right)$ for the full operational day $d_\star$ at a given asset with coordinates $\mathbf{x}_\star$. We aim to probabilistically complete  $f_\star \in \cC \left( \cI_{<\tau} \right)$ from $\tau$ until $T$.

\begin{figure*}[htb!]
    \centering
    \begin{minipage}{.48\textwidth} 
        \centering
        \begin{overpic}[width=1.1\textwidth, trim = {0cm, 0cm, 0cm, 0cm}, clip]{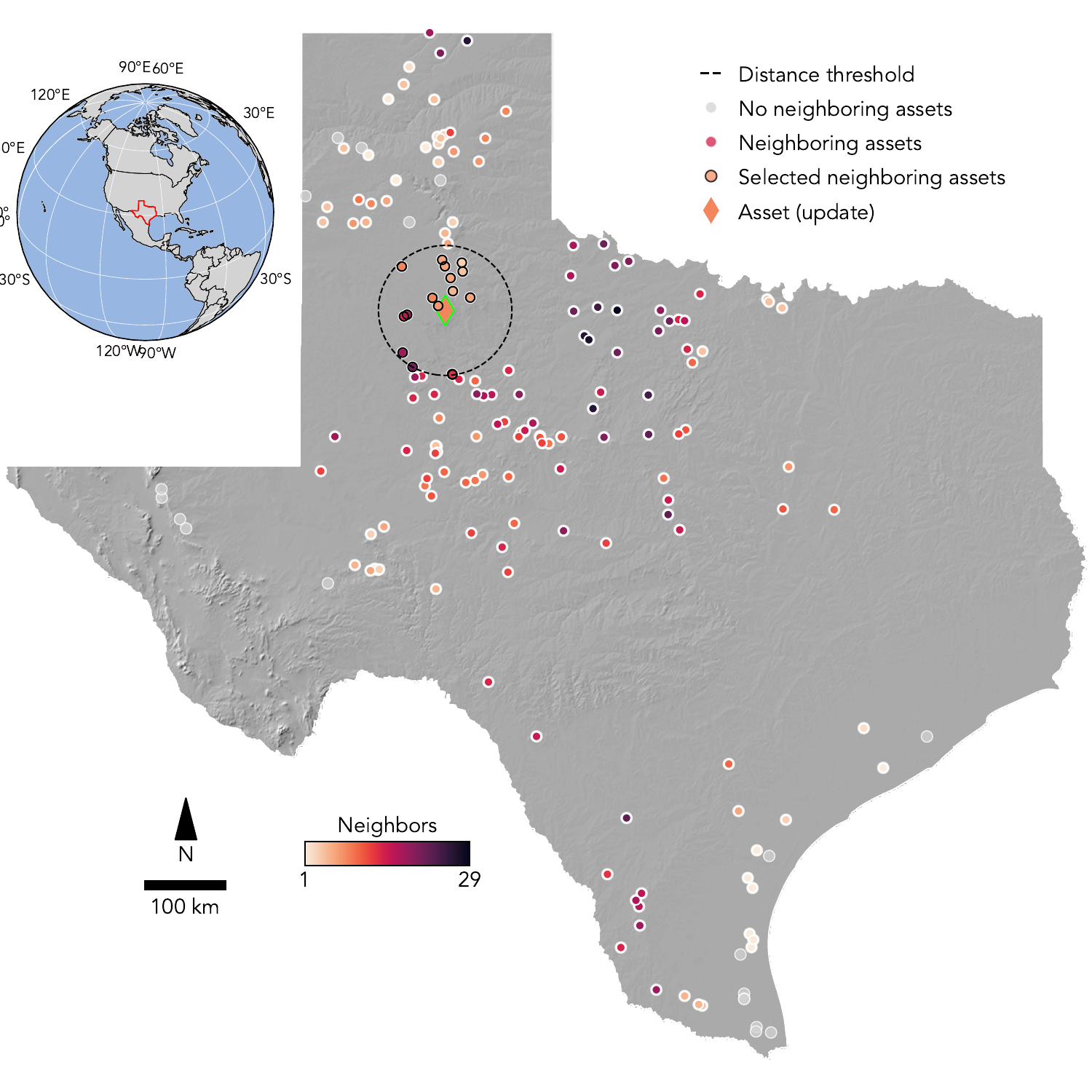}
        \end{overpic}
    \end{minipage}
    \begin{minipage}{.5\textwidth} 
        \centering
        \begin{overpic}[width=.75\textwidth, trim = {0cm, 0cm, 0cm, 0cm}, clip]{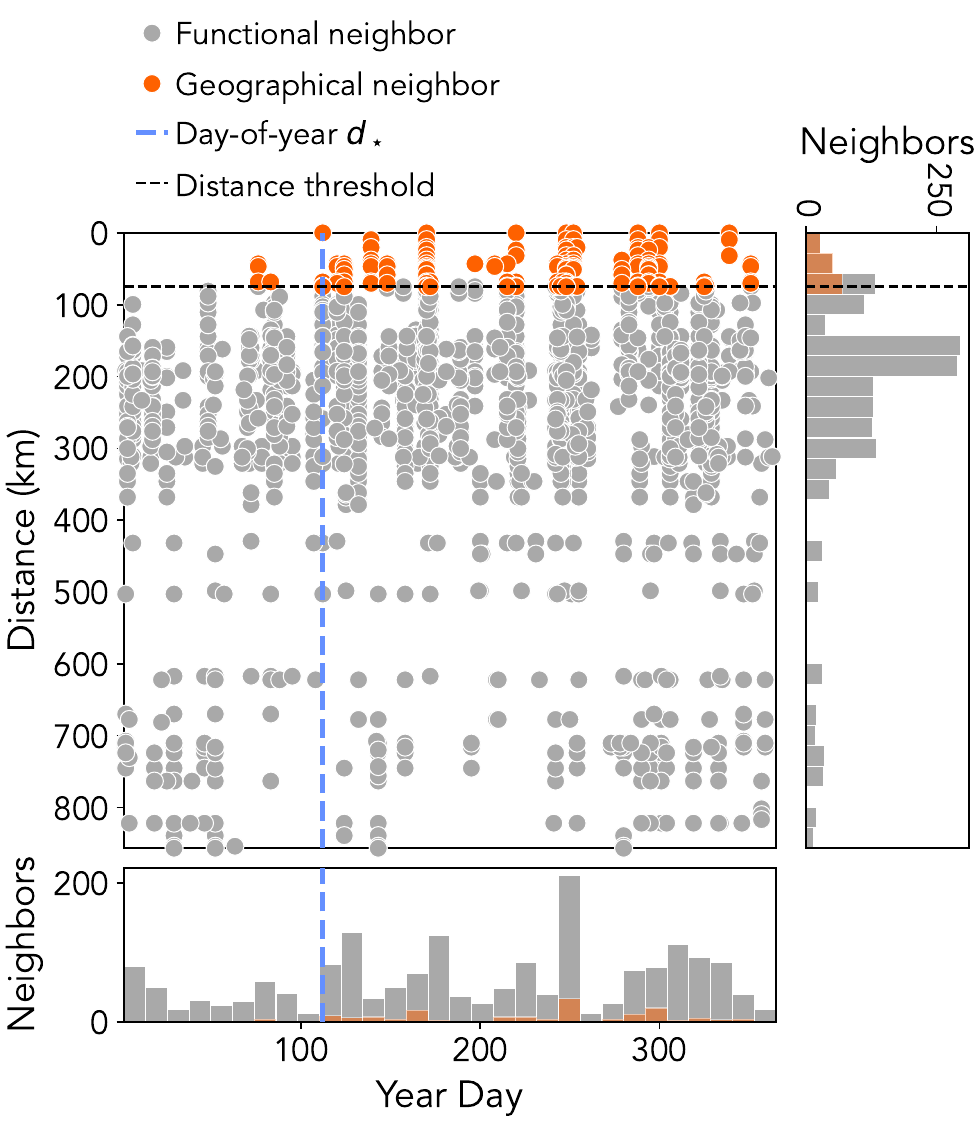}
        \end{overpic}
    \end{minipage}
    \caption{\textbf{Mapping functional neighbors.} Neighboring curves selected for Blanco Canyon Wind (marked with a bright green marker) on Apr 23 at $\tau=2$pm (left panel). Gray markers represent assets without any similar curves to $(f_\star, e_\star)$. Shaded markers with a white/black edge represent assets with $(f_\star, e_\star)$-neighbors before/after applying the distance filter. The Haversine distance threshold, defined by the most distant functional neighbor, is the dashed black circle. \emph{Right panel:} Distribution of $(f_\star, e_\star)$-neighbors in the larger $\kappa_0$-neighborhood by month (x-axis) and by distance to Blanco Canyon Wind (y-axis, black dashed line shows the distance cutoff). In both panels, light/dark colors represent low/high frequency.}
    \label{fig:spatio-temporal-wind-ercot}
\end{figure*}



\paragraph{Similarity Metric.} 


Consider a candidate $(f_n, e_n) \in \cF \times \cE$. To construct a similarity measure between $f_\star$ and $f_n$, and between $e_\star$ and $e_n$ we first use distances, namely the $L_2$ distance between the respective curves. We implement exponential weighting whereby time-points closer to $\tau$ are given higher weight to ensure that similar curves are close to each other at the present. The $f$-weights are exponential 
with decay parameter $\varepsilon_f \geq 0$ dubbed the \emph{forgetting rate}, 
\begin{equation}
    \label{eq:partial_curve_dist}
    \left\| f_\star - f_n \right\|_{\varepsilon_f} := \left[ \frac{1}{\sum_{t \in \mathcal{I}_{\le\tau}} \exp \left(-\varepsilon_f (\tau - t)\right)} \sum_{t \in \mathcal{I}_{\le\tau}} \exp \left( -\varepsilon_f (\tau - t) \right) \left( f_\star (t) - f_n(t) \right)^2 \right]^{1/2}.
\end{equation}
When $\varepsilon_f = 0$, $\|\cdot\|_0$ reduces to the Euclidean distance, while as $\varepsilon_f \to \infty$ the weight concentrates on the most recent observation and $\| f_\star - f_n \|_{\varepsilon_f} \to \left| f_\star(\tau) - f_n(\tau) \right|$.

Similarly, we use an exponentially weighted $L^2$-distance between the DA forecast $e_\star$ and $e_n$ (which both live on the full time domain $\cI$),
\begin{equation}
    \label{eq:DA_forecast_dist}
    \left \| e_\star - e_n \right\|_{\varepsilon_{e^-}, \varepsilon_{e^+}} :=   \left[ \frac{\sum_{t \in \cI} \phi_{\varepsilon_{e^-},\varepsilon_{e^+}}(t, \tau)  \left( e_\star(t) - e_n(t) \right)^2}{\sum_{t \in \cI} \phi_{\varepsilon_{e^-},\varepsilon_{e^+}}(t, \tau)}
\right]^{1/2}.
\end{equation}
where the weight function $\phi_{\varepsilon_{e^-}, \varepsilon_{e^+}} \left(t, \tau \right)$ is piecewise defined as a double exponential centered at $\tau$:
\begin{equation}
    \label{eq:DA_forecast_decay}
    \phi_{\varepsilon_{e^-}, \varepsilon_{e^+}} \left(t, \tau \right) = \begin{cases}
        \exp \left( \varepsilon_{e^-}(t-\tau) \right) 
        & t < \tau \\
        \exp \left( \varepsilon_{e^+} (\tau - t) \right) & t \geq \tau, 
    \end{cases}
\end{equation}
with \emph{past decay rate} $\varepsilon_{e^-}$ and \emph{future decay rate} $\varepsilon_{e^+} \geq 0$. This reduces to $\|\cdot\|_0$ when $\varepsilon_{e^-} = \varepsilon_{e^+} = 0$, and concentrates on the DA forecast at $\tau$ as $\varepsilon_{e^-}, \varepsilon_{e^+} \rightarrow \infty$. The two functional distances are mapped into similarity measures through an exponential link function,
\begin{equation}
    \label{eq:exp_link}
    \varphi_\lambda \left( r \right) = \exp \left( -\lambda r \right),
\end{equation}
with a channel-specific \emph{similarity weight} $\lambda > 0$. Applying Eq.~\eqref{eq:exp_link} to the distances of Eqs.~\eqref{eq:partial_curve_dist} and \eqref{eq:DA_forecast_dist} yields the functional (i.e., channel) similarities
\begin{equation}
    \label{eq:f_e_similarity}
    \omega^f_n \left( f_\star; \;\mathcal{I}_{\le \tau} \right) := \varphi_{\lambda_f} \left( \left\| f_\star - f_{n} \right\|_{\varepsilon_f} \right), \qquad
    \omega^e_n \left( e_\star; \mathcal{I} \right) := \varphi_{\lambda_e} \left( \left\| e_\star - e_{n} \right\|_{\varepsilon_{e^-}, \varepsilon_{e^+}} \right),
\end{equation}
which take values in $(0, 1]$ and attain $1$ when the target and the historical curve coincide over the respective domain. Thus $\lambda_f$ and $\lambda_e$ control the trade-off between giving more importance to similarity between observed curves $f_n$ and $f_\star$ versus similarity between unrevised DA forecasts $e_n$ and $e_\star$ when selecting neighbors. The more accurate DA forecasts are, the large $\lambda_e$ should be; if DA forecasts are non-informative then $\lambda_e \ll \lambda_f$ is preferable, relying instead on similarity to the observed partial curves.



Seasonal variability implies that we want to take into account $d_\star, d_n$: it is unlikely that a realized profile from January would be informative for a June intra-day forecast \cite{mayes2015}. This temporal effect is very strong for solar generation that is affected by length-of-sunlight but is also substantial  for wind generation. 
%
The seasonal weather variability is fundamentally driven by the annual cycle of solar radiation associated with the Earth's orbital motions and axial tilt \cite{hartmann2015global}. Around equinoxes, day length and solar incidence angles are similar between spring and fall, while summer and winter solstices are maximally far from each other. To preserve this spring-fall symmetry and winter-summer asymmetry, we use the solar declination angle, 
\begin{equation}\label{eq:declination}
    S(d) = \sin\left(\frac{2\pi}{365}(d - 81)\right),
\end{equation}
a periodic function of $d$ that vanishes at $d_{spring} = 81, d_{fall} = 265$ equinoxes, and attains its maximum/minimum at the summer/winter solstice $S(d_{summer} = 172) \approx 1, S(d_{winter} = 355) \approx -1$, respectively. The day-of-year distance between the day $d_\star$ corresponding to $f_\star$ and day $d_n$ is
\begin{equation}
    \label{eq:elevation_dist}
    \left \| d_\star - d_n \right\|_{S} := \left[ \left( S(d_\star) - S(d_n) \right)^2 \right]^{1/2},
\end{equation}
which is small whenever the two days share a similar solar elevation profile (Figure~\ref{fig:SED}). We define the similarity metric
\begin{equation}
    \label{eq:d_similarity}
    \omega_n^d \left(d_\star \right) := \varphi_{\lambda_d} \left( \left\| d_\star - d_n \right\|_{S} \right),
\end{equation}
utilizing seasonal \emph{similarity weight} $\lambda_d > 0$.



\paragraph{Geographical filtering.} Renewable generation is driven by meteorological patterns in the vicinity of the target asset. Consequently, it generally does not make sense to count completed curves from faraway assets as neighbors. For instance, the wind generators in the Texas-7k grid model are primarily concentrated in the northwest corner of Texas (corresponding to Panhandle, West, North and Central zones of ERCOT), which has very different wind patterns compared to the coastal plains along the Gulf, see the left panel of Figure~\ref{fig:spatio-temporal-wind-ercot}. 

To measure geographic distance,  we employ the Haversine metric on a sphere  (in usual lon-lat coordinates), indicated by the circle in the left panel of Figure~\ref{fig:spatio-temporal-wind-ercot}. The Haversine distance between the location of target asset $\mathbf{x}_\star$ and an asset at location $\mathbf{x}_n$ is
\begin{equation}
     \left \|  \mathbf{x}_\star - \mathbf{x}_n\right\|_{H} := 2 R \arcsin{\left( \frac{1 - \cos \left( x_{2, n} - x_{2,\star} \right) + \cos x_{2, n} \cos x_{2,\star} \left( 1 - \cos \left(x_{1, n} - x_{1, \star} \right) \right)}{2} \right)^{1/2}},
\end{equation}
where $R = 6,371$~km is the Earth radius. 



\paragraph{Aggregate Similarity and Nearest Neighborhood Construction.}
We combine these three distances into a single similarity measure based on the weighted Chebyshev ($\text{L}_\infty$) combination
\begin{equation}
    \label{eq:omega-weights}
    \omega^{fed}_n \left(f_\star, e_\star, d_\star \right) := \min \left[ \omega^f_n \left( f_\star; \mathcal{I}_{\le \tau} \right), \omega^e_n \left( e_\star; \mathcal{I} \right), \omega_n^d \left(d_\star \right) \right].
\end{equation}
The ranking based on $\omega^{fed}$ similarities is used to define nearest neighbors: sort 
    \begin{equation}
       \omega^{fed}_{(1)} \left( f_\star, e_\star, d_\star \right) \ge \omega^{fed}_{(2)} \left( f_\star, e_\star, d_\star\right) \ge \ldots \ge \omega^{fed}_{(\kappa_0)} \left( f_\star, e_\star, d_\star  \right)
    \end{equation}
and identify the set $\mathcal{N}_{\kappa_0}$ of $\kappa_0$ nearest neighbors of $(f_\star, e_\star, d_\star)$. We next show that this is equivalent to a nearest-neighbor rule in a single composite dissimilarity. Define
\begin{equation}
    \label{eq:composite}
    r_n = \max\bigl\{\lambda_f \left\| f_\star - f_n \right\|_{\varepsilon_f}, \lambda_e \left\| e_\star - e_n \right\|_{\varepsilon_{e^-}, \varepsilon_{e^+}}, \lambda_d \left\| d_\star - d_n \right\|_{S}\bigr\}.
\end{equation}

Then $\omega^{fed}_n = \exp(-r_n)$, so ranking by decreasing $\omega^{fed}_n$ is equivalent (due to strict monotonicity of the exponential function) to ranking by increasing $r_n$. Hence, under the min-combination in Eq.~\eqref{eq:omega-weights}, $\mathcal{N}_{\kappa_0}$ is the set of $\kappa_0$ nearest neighbors of $(f_\star, e_\star, d_\star)$ under the weighted Chebyshev combination in Eq.~\eqref{eq:composite}: a curve is admitted only if it is close to the target in \emph{every} channel simultaneously, with the $\lambda$'s setting the exchange rates between channels. 
We evaluate the rankings
and normalized weights in the log-domain directly via $r_n$, which is immune to the floating-point underflow of $\exp(-\lambda r)$ at large similarity weights.

As a final step, we sub-select $\kappa$ nearest neighbors out of the $\kappa_0$ above, re-ranking $\mathcal{N}_{\kappa_0}$ according to  the Haversine distance $\|\cdot\|_H$:
    \begin{equation*}
        \| \mathbf{x}_\star - \mathbf{x}_{(1)}\|_H \le \| \mathbf{x}_\star - \mathbf{x}_{(2)} \|_H \le \ldots \le \| \mathbf{x}_\star - \mathbf{x}_{(\kappa)} \|_H.
    \end{equation*}
Ties are broken based on highest similarity $\omega^{fed}_{(1)} \left( f_\star, e_\star, d_\star \right)$. The motivation is that hard geographic partitioning of the dataset creates many edge cases (e.g., the few assets that are in Central Texas, complex topography along the edge of the Llano Estacado, etc.) and therefore we opt for a data-driven nearest-neighbors similarity that operates with the full dataset. 

To sum up, we start with four distances: $\|\cdot\|_{\varepsilon_f}$ in actuals, $\|\cdot\|_{\varepsilon_{e^-}, \varepsilon_{e^+}}$ in forecasts, $\|\cdot\|_{S}$ in day-of-year, and $\|\cdot\|_{H}$ in geographical distance. The first three are fused together using $\omega^{fed}$ similarity and are used to select the nearest $\kappa_0$ neighbors, in the $(f,e,d)$-space. We then sub-select according to the spatial Haversine distance. The similarity weights $\lambda_f, \lambda_e, \lambda_d$, the past and future decay rates, and the neighborhood sizes $\kappa_0, \kappa$ are all algorithm hyperparameters to be user-selected and/or fine tuned during instantiation. 






Supplementary Figure~\ref{fig:neighbors} visualizes the similarity rankings: similarities for the DA forecasts (which are always 24-hour long) are in panel (a) and similarities for the realized generation trajectories over the morning $t \in [1,\tau]$ are in panel (b). These rankings are merged to obtain the final ensemble (weighted by similarity scores) shown in panel (c).  Note how this panel is a smooth \emph{temporal} interpolation of the top two panels (beyond also merging the two similarity metrics), illustrating the fusion of $f_n$ and $e_n$'s.  Panels (d, e) show the spatiotemporal distribution of the $\kappa$ nearest neighbor curves.

\begin{remark}
The above strategy of defining similarity across different labels associated with the curve to be completed (DA forecast, day-of-year, location) can be straightforwardly extended to incorporate further covariates or regressors, for instance more meteorological information like the temperature on day $d$. We do not have access to any other context labels besides the above, hence the reason for our specific construction \eqref{eq:omega-weights}.
\end{remark}

\subsection{Fusion with Day-Ahead Forecast} 

Having identified the nearest neighbor pairs $(f_i, e_i), i=1,\ldots, \kappa$, we wish to complete our partial curve $f_\star$. An important feature of our setup is that the intra-day functional revision is driven both by the complete neighbor realizations $f_i$ and their DA forecasts $e_i$, 
still relevant for the tail interval $\Iplus$. To do this, we create a weighted average, gradually moving from an ensemble based on the $\kappa$ selected $f_{k(i)}$'s for the near-term, to the original DA prediction $e_{k(i)}$ for the end of the day: 
   $\left( 1  - \delta_{\nu,\vartheta,\eta} (s) \right) \cdot f_{k(i)} (s) + \delta_{\nu,\vartheta,\eta} (s) \cdot e_{k(i)} \left( s \right), \ \forall i = 1, \dots, \kappa$.

For the weights $\delta_{\nu,\vartheta,\eta} (s)$, we use a Linear Inverse-Exponential function (LIE) parameterized by a \emph{trust level} $\nu \in [0,1]$ and a \emph{trust horizon} $\vartheta > 0$. The LIE function is built from a rescaled variable $u_{\vartheta,\eta}(s) = \eta \left(\frac{\tau - s}{\vartheta} + 1\right)$, which is passed through a negative Exponential Linear Unit (ELU) \cite{clevert2015fast},
\begin{equation}
    \text{ELU} \left( u_{\vartheta,\eta}(s) \right) = \begin{cases}
        u_{\vartheta,\eta} (s) & s \le \tau + \vartheta \\
        \exp \left( u_{\vartheta,\eta}(s) \right) - 1 & s > \tau + \vartheta
    \end{cases},
\end{equation}
and finally rescaled into the range $\delta_{\nu,\vartheta}(s) \in [0, \nu]$ by the {trust level} $\nu$,
\begin{equation}\label{eq:trust-level}
    \delta_{\nu,\vartheta,\eta}(s) := \nu \cdot \frac{\eta - \text{ELU} \left(u_{\vartheta,\eta}(s)\right)}{\eta + 1}.
\end{equation}
The fixed parameter $\eta = 2.5$ controls the sharpness of the transition between the two regimes. The function $\delta_{\nu,\vartheta,\eta}(s)$ linearly increases for $s \le \tau + \vartheta$ and transitions smoothly to a decaying exponential for $s > \tau + \vartheta$, with a horizontal asymptote of $\nu$, see Figure~\ref{fig:LIE}. The parameter $\vartheta$, measured in hours, represents how steep $\delta$ is originally. Hence, a quick transition to a high-confidence in DA forecasts is obtained by small $\vartheta$, high $\nu$ (e.g., $\vartheta=2,\nu=0.8)$, while a gradual transition that never fully trusts the DA forecast corresponds to a large $\vartheta$ and moderate $\nu$ (e.g., $\vartheta=6,\nu=0.4)$. The following Lemma, proved in the Appendix, summarizes properties of the LIE function $\delta_{\nu,\vartheta,\eta}(s)$.

\begin{lemma}[Monotonicity and asymptotics of the LIE trust weight]
    \label{lem:LIE}
    For fixed $\nu \in [0,1]$, $\vartheta > 0$, $\eta > 0$, the function $\delta_{\nu,\vartheta,\eta}(s)$ is nondecreasing in $s$ for $s \ge \tau$, with $\delta_{\nu,\vartheta,\eta}(\tau) = 0$ and $\delta_{\nu,\vartheta,\eta}(s) \to \nu$ as $s \to \infty$; in particular $\delta_{\nu,\vartheta,\eta}(s) \in [0, \nu]$ for all $s \ge \tau$, and $\delta_{\nu,\vartheta,\eta}(\tau + \vartheta) = \nu\,\eta/(\eta+1)$ is the value at the linear/exponential transition point.
\end{lemma}

\paragraph{DA Forecast Bias Correction.} 

Our goal is to obtain the most accurate forecast on $\Iplus$, and thus we expect that the neighbor $f_i$'s are statistically unbiased by construction. However, once we blend with the DA forecast $e_i$, and especially with $e_\star$, we need the latter to be unbiased as well. Since in practice (and certainly for the energy forecasting application) the forecasts may be biased, we first de-bias them before fusing. 




Let $f_{\ell,p,t}(d)$ denote the observed generation and $e_{\ell,p,t}(d)$ the DA forecast curves for asset $\ell \in \{1, \ldots, L\}$, where $p \in \{2017,\ 2018\}$ denotes the period (year), $d \in \{1, \ldots, 365\}$ the day-of-year, and $t \in \{1, \ldots, 288\}$ the 5-minute interval within a day. We fit the following linear regression model independently for each asset $\ell$, period $p$, and interval $t$ to correct systematic forecast biases,
\begin{equation}\label{eq:bias-regression}
    f_{\ell,p,t}(d) = \alpha_{\ell,p,t} + \beta_{\ell,p,t} \cdot e_{\ell,p,t}(d) + \epsilon_{\ell,p,t}(d), 
\end{equation}
where, for each fixed $(\ell,p,t)$, the residuals
$\epsilon_{\ell,p,t}(d)$, taken over $d=1,\ldots,365$, are assumed i.i.d.\ Gaussian. The fitted coefficients $\hat\alpha_{\ell,p,t}$ and
$\hat\beta_{\ell,p,t}$ are used to obtain bias-corrected
forecasts
\begin{equation}\label{eq:bias}
    \hat{e}_{p,\ell.t}(d) = \hat\alpha_{\ell,p,t} + \hat\beta_{\ell,p,t} \cdot e_{\ell,p,t}(d).
\end{equation}

The bias coefficients for each period $p$ are estimated in-sample, using that period's own forecast--generation pairs. This is a deliberate idealization of operational practice: system operators continuously monitor and recalibrate the systematic bias of their forecast vendors, so by the time any given day is forecast, the prevailing bias level is approximately known. Estimating $(\hat\alpha, \hat\beta)$ over the full period approximates this running recalibration with a single static fit. 

\paragraph{Upscaling.} 

Because the DA forecasts are only hourly, we need to upscale the fused curves back to the 5-min timecale. To do so, we use cubic $\mathcal{B}$-splines, namely a linear combination of $\bar{b}_\tau = m (\tau) + q - 1$ basis functions
\begin{equation}
    \bar \mu \left( s \right) = \sum_{i = 1}^{\bar{b}_\tau} c_i \mathcal{B}_{i, d} \left( s \right),
\end{equation}
where $c_i \in \mathbb{R}$ are the basis coefficients, $\mathcal{B}_{i, d} \left(s \right)$ are piecewise cubic polynomials (order $q = 4$) estimated with the Cox–de Boor recursion formula, 
and the number of terms is $m (\tau) = \lceil  \frac{\left|\mathcal{I}_{>\tau}\right|}{9} \rceil$.


\paragraph{Fusion Acceptance Probability.}

Renewable energy CFs are frequently either exactly 0 or 1, corresponding to zero generation or maximum generation. Directly fusing DA forecasts $e_n$ (which rarely take values in $\{0,1\}$) with observed curves $f_n$ shrinks the predictive distribution toward the interior of the unit interval, artificially reducing the probability masses at 0 and 1. To preserve the latter, we introduce a randomized fusion acceptance mechanism that proxies edge-inflation. Specifically, for every neighbor of $f_{\star}$  either fusion is done as above, or it is skipped and $f_{k(i)}$ is retained as-is. The fusion probability is taken as
\begin{equation}
    \label{eq:fusion-rejection}
    p_i(p_{\mathrm{fusion}}) := \min\left( \frac{p_{\mathrm{fusion}} \cdot \left( 1 - \nu \right)}{1 - \frac{1}{\left| \mathcal{I}_{>\tau} \right|}
    \sum_{s \in \mathcal{I}_{>\tau}}
    \mathbb{I} \bigl( f_{k(i)}(s) \in \{0, 1\} \bigr)}, 1 \right),
\end{equation}
where $p_{\mathrm{fusion}}$ is a hyperparameter controlling the fusion probability, and $\nu$ is the {trust level} from Eq.~\eqref{eq:trust-level}. The probability of retaining $f_{k(i)}$ increases with the fraction of time that a curve $f_{k(i)}$ is at the edge, and decreases with $\nu$. Consequently, zero- or one-inflated curves are more likely to be kept unfused, preserving the tails of the predictive distributions.

Let $B_i \sim \Bern \left(p_i \right)$, with $B_i = 1$ indicating that the original forecast is retained (i.e., fusion is skipped), and $B_i = 0$ indicating forecast–observation fusion.  Then the fused forecast is 
\begin{equation}
    \label{eq:blended-fcst}
    \hat{\mu}_i (s) := B_i \cdot f_{k(i)}(s) + \left( 1 - B_i \right) \cdot \left[ \left( 1 -\delta_{\nu,\vartheta}(s)\right) \cdot f_{k(i)}(s) + \delta_{\nu, \vartheta}(s) \cdot e_{k(i)}(s) \right], \quad \forall i = 1, \dots, \kappa.
\end{equation}
In sum, the functional neighborhood we construct is the set of $\kappa$ curves $\cJ = \left\{\hat{\mu}_i (s), \ i=1,\ldots, \kappa \right\}$. 

\subsection{Adaptive and Dynamic Uncertainty Quantification}

\paragraph{Marginal Predictive Distribution at Time $s$.}

Reindexing the $\kappa$ fused forecasts $\hat{\mu}_i(s), i=1,\ldots,\kappa$ in Eq.~\eqref{eq:blended-fcst} in ascending order, $\hat\mu_1(s) \le \hat\mu_2(s) \le \dots \le \hat\mu_\kappa(s)$, with the weights $\omega_i$ reordered correspondingly, the weighted marginal quantiles function 
is defined as $Q_{\omega}\left(q; \hat\mu_\cdot(s)\right) := \hat\mu_{k(q)}(s)$, where $k(q)$ is the smallest index $k'$ such that $\sum_{i=1}^{k'} \omega_i \ge q \sum_{i=1}^{\kappa} \omega_i$. The median of the marginal predictive density is given by
\begin{equation}
    \label{eq:wquantile}
    \hat \mu_{\text{\text{med}}}(s) := Q_{\omega}\left(0.5; \hat \mu_\cdot(s) \right).
\end{equation}

The full marginal predictive density $p_s(\cdot)$ for the horizon $s \in \mathcal{I}_{>\tau}$ is constructed via the weighted Gaussian kernel smoothing \cite{cai2001weighted}, 
\begin{equation}
    p_{s} \left(\hat{\mu}; h(s) \right) := \frac{1}{\sum^{\kappa}_{i = 1} \omega_i} \sum^{\kappa}_{i = 1} \omega_i \cdot \frac{1}{\sqrt{2\pi}\, h(s)} \exp \left(-\frac{\left(\hat{\mu} - \hat{\mu}_i (s)\right)^2}{2h^2(s)}\right),
\end{equation}
where $\hat{\mu}$ denotes the CF value at which the density is evaluated, and $\omega_i := \omega^{fed}_{k(i)} \left( f_\star, e_\star, d_\star \right)$ abbreviates the similarity metric in Eq.~\eqref{eq:omega-weights} of the $i$-th selected neighbor $k(i)$.
The KDE bandwidth $h(s)$ depends on $s$ and is set according to the weighted extension of Silverman's Rule \cite{silverman2018} with the effective sample size term 
\begin{equation}
    \label{eq:wbandwith}
    h(s) = 0.9 \cdot \min\left(\hat{s}_{\omega} \left(\hat\mu_\cdot(s)\right),\; \frac{Q_{\omega} \left(0.75; \hat\mu_\cdot(s)\right) - Q_{\omega} \left(0.25; \hat\mu_\cdot(s)\right)}{1.34}\right) \cdot \left(\frac{\left(\sum_{i=1}^{\kappa} \omega_i\right)^2}{\sum_{i=1}^{\kappa} \omega_i^2}\right)^{-1/5},
\end{equation}
where $\hat{s}_{\omega} \left(\hat\mu_\cdot(s)\right)$ is the weighted marginal standard deviation, and $Q_{\omega} \left(0.25; \hat\mu_\cdot(s)\right)$ and $Q_{\omega} \left(0.75; \hat\mu_\cdot(s)\right)$ are the first and third predictive quartiles.

\paragraph{Marginal Quantiles.} The basic forecasting interval is based on marginal quantiles coming from the predictive distribution. Given a level $\alpha$, let $l \left(s; \alpha\right) = Q_\omega(1-\alpha; \hat{\mu}_\cdot(s))$ and $u \left(s;\alpha\right) = Q_\omega(\alpha; \hat{\mu}_\cdot(s))$ be the lower and upper quantiles. Then the corresponding central $(1 - \alpha) \times 100\%$ predictive band, computed pointwise for each $s$ is $\cR^{\mathrm{mar}}_\alpha(s) := [l(s; \alpha), u(s; \alpha)]$.

\paragraph{Nonparametric Functional Quantiles.}

Beyond the standard marginal predictive distribution at a fixed future hour $s$, the power of our approach lies in being able to construct functional boxplots and functional quantiles for the completed curve across the entire \emph{forecast horizon}. The goal is to create summary statistics of the set of curves $\cJ = \left\{\hat{\mu}_i (s), i=1,\ldots, \kappa \right\}$.  Specifically, we are interested in the median curve, the central region and the maximum envelope. The median curve is the most-central curve and interpreted as the ``most likely'' functional forecast. The central region is the band containing the 50\% of most-central curves. The maximum envelope is the band defined by most extreme non-outlier functions. 

\paragraph{Functional Envelope.} To define the above objects, we use the notion of \emph{centrality} or outlyingness  based on ranking by functional depth. %
Modified Band Depth (MBD) \cite{dai2020} averages the proportion of time that a functional observation $\hat{\mu}_i$ lies between pairs of curves in $\cJ_{/i}$ 
\begin{equation}
    \label{eq:mbd}
    \mathcal{MBD} \left(\hat{\mu}_i, \cJ_{/i} \right):= \binom{\kappa}{2}^{-1} \sum_{1 \leq i_1<i_2 \leq \kappa} \Leb \left( \{s: \min \left(\hat{\mu}_{i_1}(s), \hat{\mu}_{i_2}(s) \right) \leq \hat{\mu}_i(s) \leq \max \left(\hat{\mu}_{i_1}(s), \hat{\mu}_{i_2}(s) \}\right) \right).
\end{equation}

Using the metric $\mathcal{MBD}$ in Eq.~\eqref{eq:mbd}, we compute the functional depth $\text{MBD}_i$ of each curve and rank them according to $\text{MBD}_{(1)} \ge \text{MBD}_{(2)} \ge \ldots $. The \emph{deepest curve} corresponding to $\text{MBD}_{(1)}$ is called the central one:
\begin{align}
    \label{eq:central-curve}
    \hat \mu_{\text{\text{deep}}} (s) := \hat \mu_{i}(s): \qquad \text{MBD}_i = \text{MBD}_{(1)}.
\end{align}

Based on the idea of the functional boxplot \cite{sun2011functional}, the \emph{functional envelope} $\mathcal{R}^{\mathrm{fun}}_{\alpha}$ is defined as the region that contains the collection $\mathcal{H}^{\mathrm{fun}}_{\alpha}$ of the $\lceil (1 - \alpha) \cdot \kappa \rceil$ deepest curves $\hat{\mu}_i(\cdot)$ from the set in Eq.~\eqref{eq:blended-fcst}, where $\alpha$ is the significance level and $\lceil \cdot \rceil$ is the ceiling function,
\begin{equation}
    \label{eq:functional-envelope}
    \mathcal{R}^{\mathrm{fun}}_{\alpha}(s)
    := \Bigl[ \min_{\hat{\mu}_i \in \mathcal{H}^{\mathrm{fun}}_{\alpha}} \hat{\mu}_i(s),
       \ \max_{\hat{\mu}_i \in \mathcal{H}^{\mathrm{fun}}_{\alpha}} \hat{\mu}_i(s) \Bigr],
    \qquad s \in \mathcal{I}_{>\tau}.
\end{equation}
Note that the collection $\mathcal{H}^{\mathrm{fun}}_{\alpha}$ is fixed across horizons $s$, which gives a functional-driven, rather than a marginal, construction. 

\begin{remark}
    \label{rem:nesting}
    For $0 < \alpha_{j'} < \alpha_j < 1$ we have $\lceil(1-\alpha_j)\kappa\rceil \le \lceil(1-\alpha_{j'})\kappa\rceil$ and, since both collections are prefixes of the same depth ordering, $\mathcal{H}^{\mathrm{fun}}_{\alpha_j} \subseteq \mathcal{H}^{\mathrm{fun}}_{\alpha_{j'}}$. Taking the min/max in Eq.~\eqref{eq:functional-envelope} over a larger collection can only widen the band, so $\mathcal{R}^{\mathrm{fun}}_{\alpha_j}(s) \subseteq \mathcal{R}^{\mathrm{fun}}_{\alpha_{j'}}(s)$ for every $s \in \mathcal{I}_{>\tau}$: the higher-coverage envelope (smaller $\alpha_{j'}$, built from more of the deepest curves) contains the lower-coverage one, so the bands are nested and can never cross.
\end{remark}

The functional envelope in Eq.~\eqref{eq:functional-envelope} is not calibrated, so the $\lceil (1 - \alpha) \cdot \kappa \rceil$ deepest curves do not exactly attain the nominal coverage of $1 - \alpha$, because the depth ranking is a measure of centrality not of probability. We therefore introduce a calibration factor $k^{\mathrm{adj}}(\alpha)$ which adjusts the number of curves that represent a given nominal coverage, and define the \emph{adjusted functional envelope} as the region formed by the collection $\mathcal{H}^{\mathrm{adj}}_{\alpha}$ of the $\lceil k^{\mathrm{adj}}(\alpha) \cdot \kappa \rceil$ of deepest curves,
\begin{equation}
    \label{eq:adjusted-envelope}
    \mathcal{R}^{\mathrm{adj}}_{\alpha}(s)
    := \Bigl[ \min_{\hat{\mu}_i \in \mathcal{H}^{\mathrm{adj}}_{\alpha}}
              \hat{\mu}_i(s),
       \ \max_{\hat{\mu}_i \in \mathcal{H}^{\mathrm{adj}}_{\alpha}}
              \hat{\mu}_i(s) \Bigr],
    \qquad s \in \mathcal{I}_{>\tau}.
\end{equation}
The factor $k^{\mathrm{adj}}(\alpha) \in [0, 1]$ is estimated for each significance level $\alpha$, so that the functional coverage of $\mathcal{R}^{\mathrm{adj}}_{\alpha}$ matches its nominal level $1 - \alpha$. The proposed adjustment acts only on the number of curves retained, the construction $\mathcal{H}^{\mathrm{adj}}_{\alpha}$ is still a prefix of the depth ordering and is still fixed across horizons $s$.

\paragraph{Projection Envelope.}
As alternative, and similar to a functional Nadaraya--Watson estimator, 
we complete the partial curve $f_\star$ with the conditional expectation over the neighborhood. The resulting focal curve \cite{elias2022} is given by
\begin{equation}
    \label{eq:focal}
    \hat \mu_{\text{focal}} (s) = \frac{1}{\sum_{i=1}^{\kappa}\omega_i } \sum_{i=1}^{\kappa}\omega_i \cdot \hat\mu_i(s),
\end{equation}
where $\omega_i \equiv \omega^{fed}_i \left( f_\star, e_\star, d_\star \right)$ are the similarity weights in Eq.~\eqref{eq:omega-weights} and $\hat\mu_i(s)$ are the fused forecasts in Eq.~\eqref{eq:blended-fcst}.

To characterize the uncertainty of this conditional expectation, we construct a \emph{projection envelope} that has $\hat \mu_{\text{focal}}(s)$ as its central curve and is formed by the fused forecasts that bracket it \cite{elias2022}. For any candidate set $\cJ^\prime$, the degree of envelopment of $\hat \mu_{\text{focal}} (s)$ is
\begin{equation} 
    \label{eq:envelopment}
    \mathcal{E} \left(\hat \mu_{\text{focal}} (s),\cJ^\prime \right) := \Leb\!\left(\left\{s\in [\tau, T] : \min_{\hat\mu\in\cJ^\prime}\hat\mu(s)  \le \hat \mu_{\text{focal}}(s)  \le \max_{\hat\mu\in\cJ^\prime}\hat\mu(s)  \right\} \right),
\end{equation}
and the envelope is the union of successive passes that each attain full envelopment,
\begin{equation} 
    \label{eq:funtional_envelope}
    \cJ^{\text{prj}} = \left\{ \hat \mu_{\text{focal}}(s) \right\} \cup \bigcup_{j \ge 1} \cJ_j,\quad
    \cJ_j \in \argmin_{\cJ^\prime \subseteq \cP_j} \sum_{\hat \mu_i \in \cJ^\prime} \left\| \hat \mu_i (s) - \hat \mu_{\text{focal}} (s) \right\|_2,
\end{equation}
where $\cP_1 = \cJ$ and $\cP_{j+1} = \cP_j \setminus \cJ_j$ is the pool remaining after the $j$-th pass. 

We solve this envelopment problem greedily, transferring a fused forecast $\hat \mu_i (s)$ from $\cP_j$ to the current pass $\cJ^\prime$ in order of increasing $L_2$ distance to $\hat \mu_{\text{focal}} (s)$. We keep only those $\hat \mu_i (s)$ that increase the envelopment. The pass stops once $\hat \mu_{\text{focal}} (s)$ is bracketed at every point of the horizon $\mathcal{E} \left( \hat \mu_{\text{focal}} (s), \cP_j \right) < T - \tau$, and a new pass begins on the remaining pool $\cP_{j+1}$. The curves of $\cJ^{\mathrm{prj}}$ are then ranked by their $L_2$ distance to $\hat \mu_{\text{focal}} (s)$, and, as with the functional envelope, the number retained is calibrated rather than nominal. The projection band $\cJ^{\mathrm{prj}}_{\alpha}$ at significance level $\alpha$ is the collection of the $\lceil k^{\mathrm{prj}}(\alpha) \cdot |\cJ^{\mathrm{prj}}| \rceil$ curves closest to $\mu_{\text{focal}} (s)$,
\begin{align}
    \label{eq:projection-band}
    \mathcal{R}^{\text{prj}}_{\alpha} \left( s \right) := \left[
    \min_{\hat{\mu}_i \in \cJ^{\mathrm{prj}}_{\alpha}} \hat{\mu}_i \left( s \right), \
    \max_{\hat{\mu}_i \in \cJ^{\mathrm{prj}}_{\alpha}} \hat{\mu}_i \left( s \right)
    \right].
\end{align}
This band is structurally different from $\mathcal{R}^{\mathrm{adj}}_{\alpha}$, since the initial collection is $|\cJ^{\mathrm{prj}}| \le \kappa$ rather than $\kappa$, and the ordering is by $L_2$ distance to $\hat \mu_{\mathrm{focal}} (s)$ rather than by depth. In addition, the calibration factors of the two constructions are not interchangeable $k^{\mathrm{prj}}(\alpha) \ne k^{\mathrm{adj}}(\alpha)$. Both $k^{\mathrm{prj}}(\alpha)$ and $k^{\mathrm{adj}}(\alpha)$ are fractions of the neighborhood sorted by a ranking, rather than nominal probabilities. For each significance level $\alpha$ they are selected empirically by grid search on the validation dataset, minimizing a score in Section~\ref{sec:assess}, and independently for each update time $\tau$ and case study.

We obtain nine distinct forecast objects: three central curve estimates ($\hat{\mu}_{\text{med}}$, $\hat{\mu}_{\text{focal}}$, and $\hat{\mu}_{\text{deep}}$), four regions ($\mathcal{R}^{\text{mar}}$, $\mathcal{R}^{\text{fun}}$, $\mathcal{R}^{\text{adj}}$, and $\mathcal{R}^{\text{prj}}$), two predictive distributions $p(s)$, all inferred from the collection of neighboring curves $\mathcal{J}$. Figure~\ref{fig:functional-intervals} illustrates these elements for the Blanco Canyon wind asset on two dates, contrasting a dense fused neighborhood (left panles) with a sparse one (right panels). The regions differ in how they impose temporal structure (Figure~\ref{fig:functional-intervals}c--l). Marginal quantiles $\mathcal{R}^{\text{mar}}$ are constructed pointwise, so their boundaries reflect sampling noise in the marginal densities and offer no guarantee for the overall trajectory (Figure~\ref{fig:functional-intervals}c--l). By contrast, the functional envelope $\mathcal{R}^{\text{fun}}$ uses the $\lceil (1 - \alpha) \cdot \kappa \rceil$ deepest curves under MBD, yielding boundaries that are realizable trajectories (Figure~\ref{fig:functional-intervals}e, f). However, because MBD centrality is a global property, the deepest curve may still contain local extrema. As a result, these bands are the widest of all constructions and tend to over-cover at both significance levels. The adjusted functional envelope $\mathcal{R}^{\text{adj}}$ corrects for this, restoring nominal coverage in sparse cases (Figure~\ref{fig:functional-intervals}g, h), where depth rankings are less informative. Finally, the focal envelope $\mathcal{R}^{\text{prj}}$ selects boundaries by their proximity to the focal curve, keeping the band aligned with key features such as the evening ramp, unlike the functional envelope (Figure~\ref{fig:functional-intervals}i, j).

\begin{figure*}[htb!]
    \centering
    \begin{minipage}{.49\textwidth} 
        \begin{overpic}[width=1\textwidth, trim = {0cm, 0cm, 0cm, 0cm}, clip]{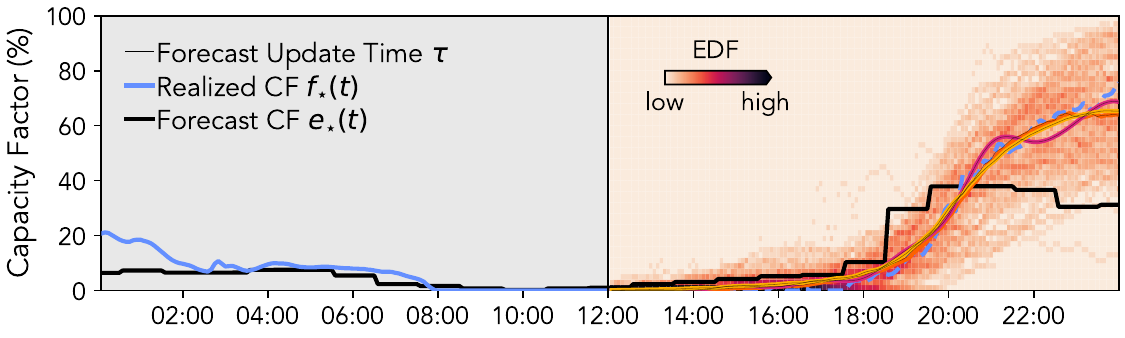}
            \put(0, 30){\textbf{a}}
        \end{overpic}
        \centering
        \begin{overpic}[width=1\textwidth, trim = {0cm, 0cm, 0cm, 0cm}, clip]{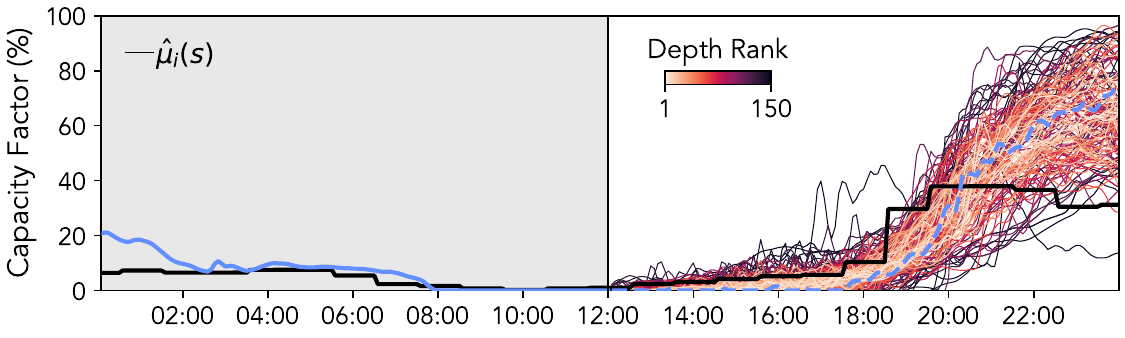}
            \put(0, 30){\textbf{c}}
        \end{overpic}
        \begin{overpic}[width=1\textwidth, trim = {0cm, 0cm, 0cm, 0cm}, clip]{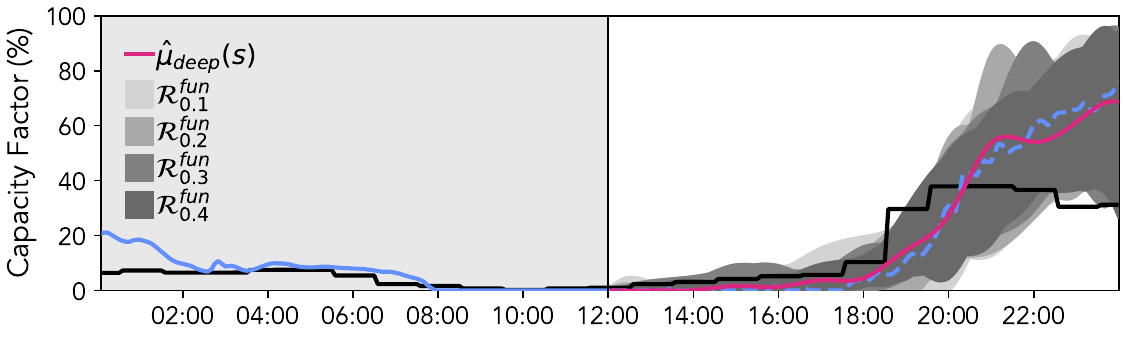}
            \put(0, 30){\textbf{e}}
        \end{overpic}
        \begin{overpic}[width=1\textwidth, trim = {0cm, 0cm, 0cm, 0cm}, clip]{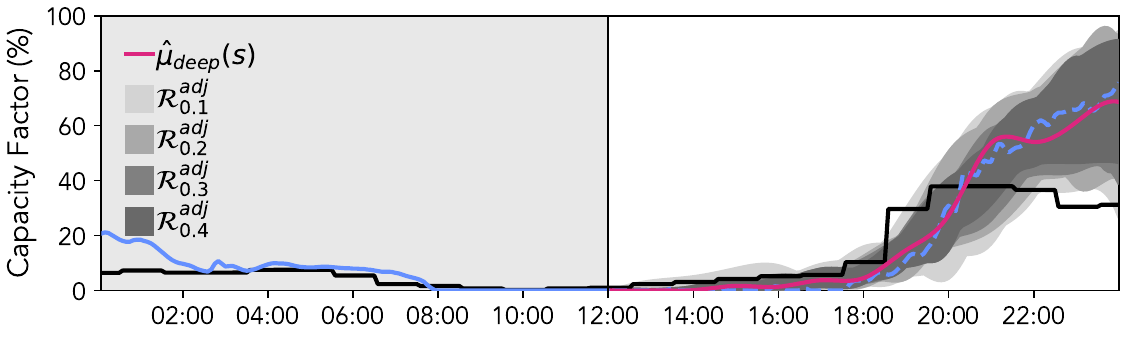}
            \put(0, 30){\textbf{g}}
        \end{overpic}
        \begin{overpic}[width=1\textwidth, trim = {0cm, 0cm, 0cm, 0cm}, clip]{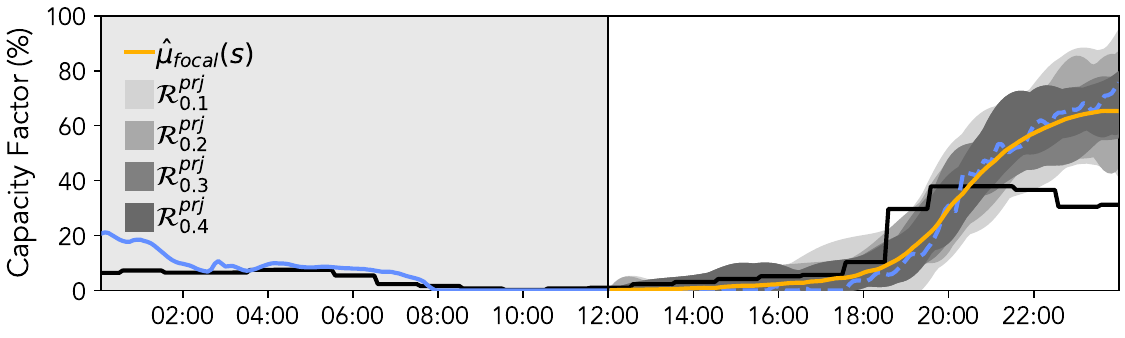}
            \put(0, 30){\textbf{i}}
        \end{overpic}
        \begin{overpic}[width=1\textwidth, trim = {0cm, 0cm, 0cm, 0cm}, clip]{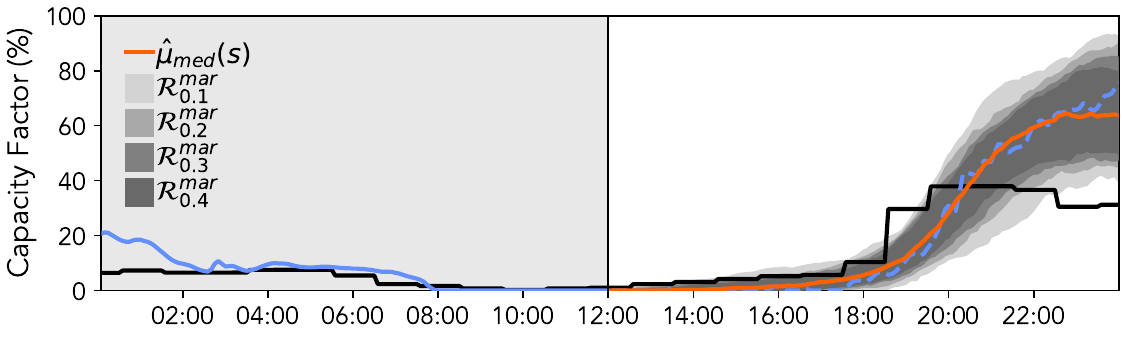}
            \put(0, 30){\textbf{k}}
        \end{overpic}
    \end{minipage}
    \begin{minipage}{.49\textwidth} 
        \centering
        \begin{overpic}[width=1\textwidth, trim = {0cm, 0cm, 0cm, 0cm}, clip]{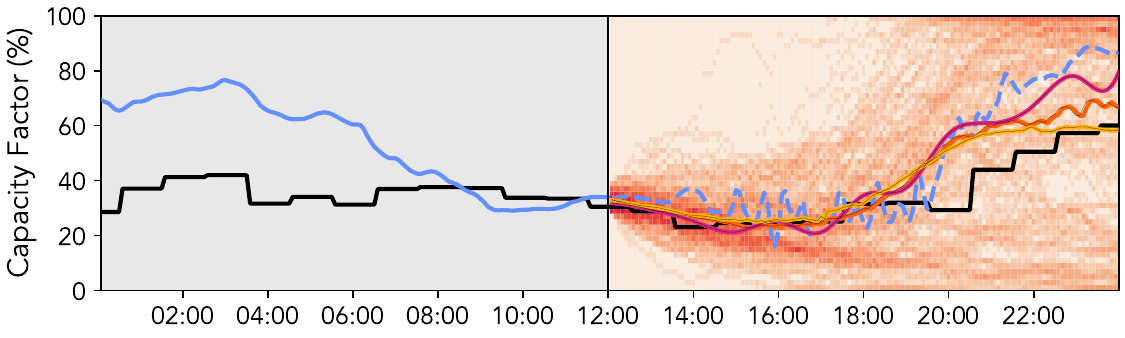}
            \put(0, 30){\textbf{b}}
        \end{overpic}
        \begin{overpic}[width=1\textwidth, trim = {0cm, 0cm, 0cm, 0cm}, clip]{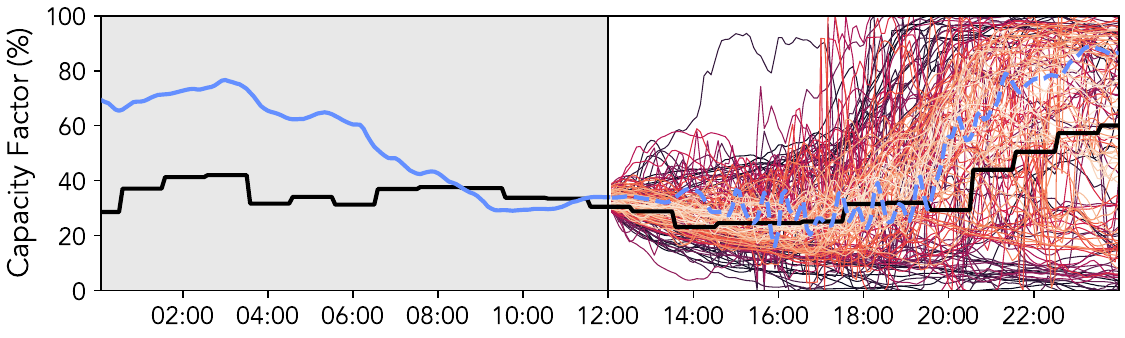}
            \put(0, 30){\textbf{d}}
        \end{overpic}
        \begin{overpic}[width=1\textwidth, trim = {0cm, 0cm, 0cm, 0cm}, clip]{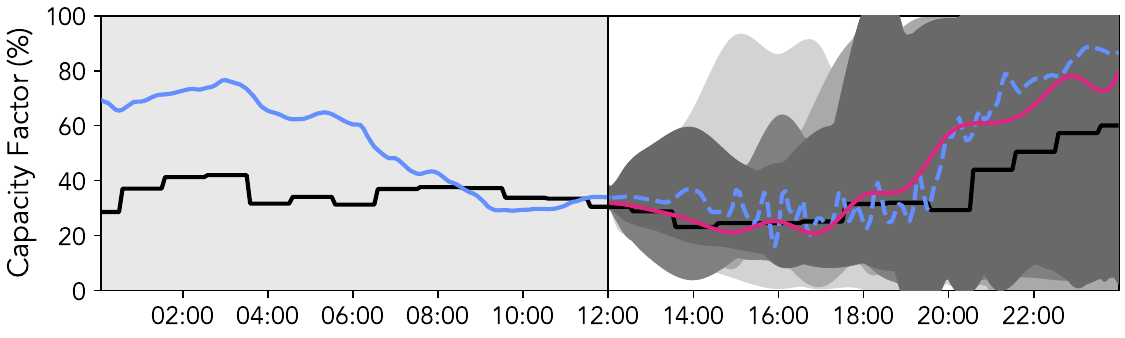}
            \put(0, 30){\textbf{f}}
        \end{overpic}
        \begin{overpic}[width=1\textwidth, trim = {0cm, 0cm, 0cm, 0cm}, clip]{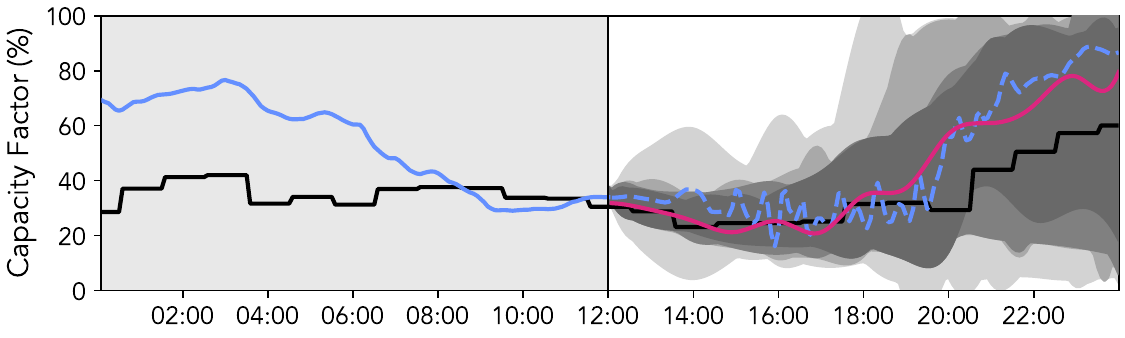}
            \put(0, 30){\textbf{h}}
        \end{overpic}
        \begin{overpic}[width=1\textwidth, trim = {0cm, 0cm, 0cm, 0cm}, clip]{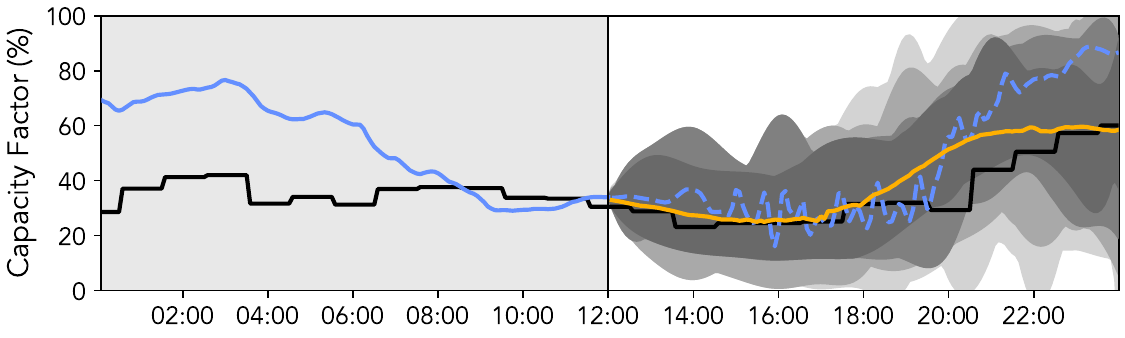}
            \put(0, 30){\textbf{j}}
        \end{overpic}
        \begin{overpic}[width=1\textwidth, trim = {0cm, 0cm, 0cm, 0cm}, clip]{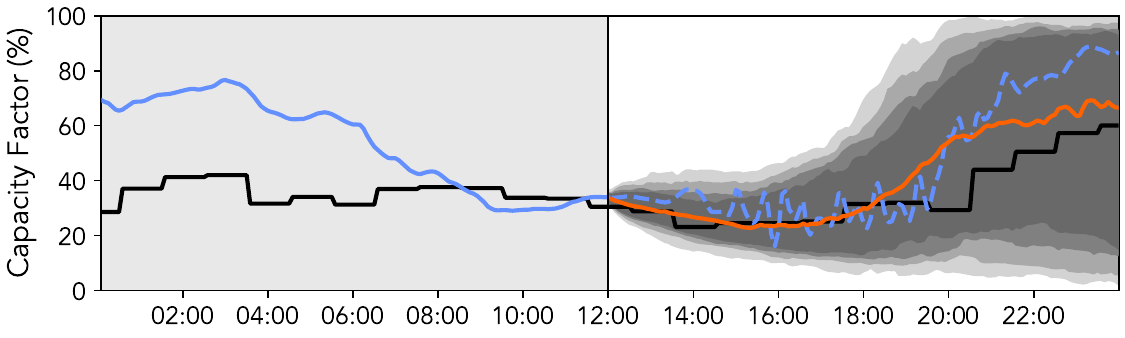}
            \put(0, 30){\textbf{l}}
        \end{overpic}
    \end{minipage}
    \caption{\textbf{Wind generation day-ahead (DA) forecast update with adaptive uncertainty quantification.} Blanco Canyon asset at forecast update time $\tau = 12$pm on Apr 23, when the fused neighborhood is dense (a, c, e, g, i, k), and on Jun 26, when it is sparse (b, d, f, h, j, l). Grey shading marks the observed period preceding $\tau$. (a, b) Predictive marginal density heatmap with the median curve $\hat{\mu}_{\mathrm{med}}(s)$ (orange), the deepest curve $\hat{\mu}_{\mathrm{deep}}(s)$ (violet), and the focal curve $\hat{\mu}_{\mathrm{focal}}(s)$ (yellow). (c, d) Fused forecasts $\hat{\mu}_i(s)$ coloured by Modified Band Depth (MBD) rank. (e--l) Nested predictive regions at $\alpha \in \{0.1, 0.2, 0.3, 0.4\}$ (light to dark): MBD-functional envelope $\cR^{\mathrm{fun}}$ (e, f), adjusted MBD-functional envelope $\cR^{\mathrm{adj}}$ (g, h), focal envelope $\cR^{\mathrm{prj}}$ (i, j), and marginal quantiles $\cR^{\mathrm{mar}}$ (k, l). The realized capacity factor $f_\cdot(t)$ is shown in blue and the DA forecast $e_\cdot(t)$ in black throughout.}
    \label{fig:functional-intervals}
\end{figure*}


\subsection{Calibration}\label{sec:assess}

In this section we discuss how to fine-tune our model for a given dataset and how to assess its performance. We consider the point forecasts, the predictive bands indexed by a confidence level $\alpha$ and the functional envelopes. In line with our focus on curve completion, our assessment metrics are not marginal but averaged across lead times $s>\tau$.

To assess the point forecasts $\hat{\mu}_\cdot$ we use lead-time-averaged root mean squared error (RMSE), mean absolute error (MAE) and mean bias error (MBE): 
\begin{align}
\label{eq:RMSE}
    \mathrm{RMSE}^2_n \left(\tau; \Theta\right) =
    \frac{1}{\left|\mathcal{I}_{>\tau}\right|}
    \sum_{s \in \mathcal{I}_{>\tau}}
    \left( \hat{\mu}_{n}\left(s; \Theta\right) - f_n(s) \right)^2;\\
    \label{eq:MAE}
    \mathrm{MAE}_n \left(\tau; \Theta\right) =
    \frac{1}{\left|\mathcal{I}_{>\tau}\right|}
    \sum_{s \in \mathcal{I}_{>\tau}}
    \left| \hat{\mu}_{n}\left(s; \Theta\right) - f_n(s) \right|; \\
\label{eq:MBE}
    \mathrm{MBE}_n \left(\tau; \Theta\right) =
    \frac{1}{\left|\mathcal{I}_{>\tau}\right|}
    \sum_{s \in \mathcal{I}_{>\tau}}
    \left( \hat{\mu}_{n}\left(s; \Theta\right) - f_n(s) \right);
\end{align}
a forecast can be unbiased ($\mathrm{MBE} \approx 0$) while still inaccurate.

For the interval forecasts, we employ the functional implementation of interval score (FIS) that compares the realization $f_n(s)$ with the point summary $m_n(s)$ and the predictive band $[l_n(s; \alpha), u_n(s; \alpha)]$ at significance level $\alpha$. The FIS is the  time-average of interval scores across lead-times $s$:
\begin{equation}
    \label{eq:FIS}
    \begin{split}
        \mathrm{FIS}_n \left(\alpha\right) & := \frac{1}{\left|\mathcal{I}_{>\tau}\right|}
    \sum_{s \in \mathcal{I}_{>\tau}} \Big\{ | u_n \left(s;\alpha\right) - l_n \left(s;\alpha\right) | + \frac{2}{\alpha} \left( l_n \left(s;\alpha\right) - f_n (s) \right) \mathbb{I} \left( f_n (s) < l_n \left(s;\alpha\right) \right) \\
        &\qquad + \frac{2}{\alpha} \left( f_n (s) - u_n \left(s;\alpha\right)  \right) \mathbb{I} \left( f_n (s) > u_n \left(s;\alpha\right) \right) \Big\}.
    \end{split}
\end{equation}

In complement, we compute the Functional Coverage Score (FCS)  \cite{khosravi2011comprehensive} that records the fraction of realizations covered by the $s$-marginal confidence band $[l_n(s; \alpha), u_n(s; \alpha)]$, averaged across the forecast horizons $s$,
\begin{equation}\label{eq:FCS}
    \mathrm{FCS}_n\left(\alpha\right) :=
    \frac{1}{\left|\mathcal{I}_{>\tau}\right|}
    \sum_{s \in \mathcal{I}_{>\tau}}
    \mathbb{I}\Bigl( l_n\left(s; \alpha\right) \le f_n(s)
    \le u_n\left(s; \alpha\right) \Bigr),
\end{equation}
so that $\mathrm{FCS}_n = 1$ is perfect coverage, and $\mathrm{FCS}_n = 0$ means the confidence bands never envelop the realization. Below we use $\alpha \in \{0.1, 0.2\}$.

To assess the entire predictive distribution at a fixed lead time $s$, we use the Kolmogorov-Smirnov (KS) statistic on the Probability Integral Transform (PIT) based on the predictive CDF $\hat{F}_{n,\Theta_\tau}\left(\cdot\,; s\right)$ constructed from the $\kappa$ fused forecasts $\hat{\mu}_i(s)$ in Eq.~\eqref{eq:blended-fcst}. The KS tests whether the empirical distribution of predictive quantiles
\begin{equation}
    F_{\mathrm{PIT}} \left(s, u; \Theta_\tau\right) = \frac{1}{N} \sum_{n=1}^{N} \mathbb{I} \left( \hat{F}_{n,\Theta_\tau} \left(f_n(s); s \right) \leq u \right),
\end{equation}
is uniform,
\begin{equation}\label{eq:KS}
    D_\mathrm{KS} \left( s; \Theta_\tau \right) = \sup_{u \in [0,1]} \left| F_{\mathrm{PIT}} \left( s, u; \Theta_\tau \right) - u \right|.
\end{equation}

Under local 
exchangeability of the target curve and the fused neighbors, the PIT is exactly (discretely) uniform and the marginal bands attain their nominal coverage.
The hyperparameter tuning for $\Theta_\tau$ described below is designed to drive the empirical PIT toward this ideal, and the reported KS distances $D_{\text{KS}}$ measure the remaining gap.

Last but not least, we assess the coverage of the functional envelopes across the entire $\Iplus$, using Simultaneous Coverage Probability (SCP) that  measures whether the entire realized curve is contained in the predictive region across all horizons,
\begin{equation}
    \label{eq:SCP}
    \mathrm{SCP}_n (\alpha ) := \mathbb{I} \Bigl( \{
    l_n \left( s; \alpha \right) \le f_n \left(s \right) \le u_n \left(s; \alpha \right) \ \forall s \in \mathcal{I}_{>\tau}\} \Bigr).
\end{equation}
A functional envelope is well-calibrated if empirical $\overline{\mathrm{SCP}}({\alpha})$ (averaged across $N$ forecasts) is close to its nominal coverage level $1-\alpha$.
Notice that the SCP is much stricter than the conventional marginal coverage, which checks each horizon $s$ independently. 

In the results, we report the $\overline{\mathrm{FIS}}({\alpha})$, $\overline{\mathrm{FCS}}({\alpha})$ and $\overline{\mathrm{SCP}}({\alpha})$ averaged across the forecast horizon and the $N$ forecasts in the testing set. 

\paragraph{Hyperparameter Tuning.}

Given the number of hyperparameters in $\Theta_\tau$, a standard $k$-fold grid-search cross-validation is not computationally feasible. Instead, for each forecast update time $\tau$, we estimate an optimal hyperparameter set $\Theta^\star_\tau$ that is common to all $L$ assets 
and invariant across days $d$.

The tuning objective function balances probabilistic calibration against deterministic accuracy. To assess calibration across lead times $s$, we partition the evaluated portion of the forecast horizon $\mathcal{I}_{>\tau}$ into $J = 6$ consecutive sub-intervals $\mathcal{I}_1, \ldots, \mathcal{I}_J$ of equal length, and compute the KS statistic in Eq.~\eqref{eq:KS} from the PIT values pooled within each sub-interval,
\begin{equation}\label{eq:KS-agg}
    D_\mathrm{KS} \left( \mathcal{I}_j; \Theta \right) = \sup_{u \in [0,1]} \left| \frac{1}{N \cdot \left|\mathcal{I}_j\right|} \sum_{n=1}^{N} \sum_{s \in \mathcal{I}_j} \mathbb{I} \left( U_n \left(s; \Theta \right) \leq u \right) - u \right|.
\end{equation}
The optimal hyperparameter set minimizes the aggregated KS statistic Eq.~\eqref{eq:KS-agg}, penalized by the accuracy of the focal curve,
\begin{equation}\label{eq:loss}
        \Theta^\star_\tau
        = \argmin_{\Theta} \ \mathcal{L}  \left(\tau; \Theta \right), \qquad \text{where}\quad 
        \mathcal{L}  \left(\tau; \Theta \right) := \left[ \frac{1}{J} \sum_{j=1}^{J} D_\mathrm{KS} \left(\mathcal{I}_j; \Theta \right) + \lambda_0 \cdot \mathrm{RMSE} \left(\tau; \Theta \right) \right],
\end{equation}
with $\lambda_0 \geq 0$ a regularization weight that anchors the predictive ensemble to the observations, preventing well-calibrated but inaccurate solutions, and
\begin{equation}
    \mathrm{RMSE} \left(\tau; \Theta \right) = \frac{1}{N} \sum_{n=1}^{N} \left[ \frac{1}{\left|\mathcal{I}_{>\tau}\right|} \sum_{s \in \mathcal{I}_{>\tau}} \left( f_n (s) - \hat{\mu}_{focal,n} \left(s; \Theta\right) \right)^2 \right]^{1/2},
\end{equation}
with $\hat{\mu}_{focal,n}(s;\Theta)$ the focal curve in Eq.~\eqref{eq:focal} associated with forecast $n$. We take $\lambda_0 = 0.25$ at the asset level, and $\lambda_0 = 0.5$ for the 2DA zonal forecasts, so both terms have the same order of magnitude.

To obtain $\Theta^\star_\tau$, we use the $D$ observed curves in $\mathcal{F}_\ell$ and corresponding set of DA forecasts $\cE_\ell$ from 2017 as the training set---the database of complete curves from which the functional neighbors are selected---and the partial curves $f_{\ell,d}(s)$ from 2018 as the testing set. We randomly select 10 assets $\ell^\prime$ from the testing set and use their 2018 curves as the validation (i.e., calibration) set, on which $\mathcal{L}(\tau;\Theta)$ is minimized. We then evaluate out-of-sample performance on 10 assets randomly chosen from the remaining $L - 10$ assets of the testing set.


Guided by an exploratory analysis conducted a priori, we partition the hyperparameters in $\Theta_\tau$ into a fixed subset and a tuned subset. The forgetting rate $\varepsilon_f$, the decay rates $\varepsilon_{e^-}$ and $\varepsilon_{e^+}$, the lengthscale of the partial curves $\lambda_f$, the number of spatiotemporal neighbors $\kappa$, and the fusion probability $p_{\mathrm{fusion}}$ are fixed across all forecast update times and both energy features (see Table~\ref{tab:asset_hyperparameters}). The remaining 5 hyperparameters are tuned by minimizing Eq.~\eqref{eq:loss}: the DA forecast and seasonal length-scales, parameterized relative to $\lambda_f$ through the ratios $\rho_e = \lambda_e / \lambda_f \in [10^{-5}, 10^{5}]$ and $\rho_d = \lambda_d / \lambda_f \in [10^{-5}, 10^{5}]$ (log scale), cf.~Proposition \ref{prop:invariance}; the size of the functional neighborhood $\kappa_0 \in [150, 5000]$ (log scale); the trust horizon $\vartheta \in [1, 12]$ hours (linear scale); and the trust level $\nu \in [0, 1]$ (linear scale). 


Since an exhaustive search over the joint hyperparameter space is intractable, we minimize Eq.~\eqref{eq:loss} via Bayesian optimization \cite{mockus1991bayesian}, as described next.



\paragraph{Fine-Tuning via Bayesian Optimization (BO).} 
We use a Gaussian Process (GP) \cite{williams2006gaussian} as the probabilistic surrogate for the relationship between the hyperparameters and the objective function in Eq.~\eqref{eq:loss}, $$\mathcal{L} \left(\tau; \Theta, \right) \sim \mathcal{GP}\left( \mathcal{M} \left( \Theta \right), \mathcal{K} \left(\Theta, \Theta^\prime\right)\right),
$$
where $\mathcal{M} \left(\Theta \right)$ is the mean function and $\mathcal{K} \left(\Theta, \Theta^\prime \right)$ defines the covariance between different hyperparameter configurations $\Theta, \Theta^\prime$; a separate GP surrogate is independently fitted for each forecast update time $\tau$. The GP surrogate provides an estimate of the uncertainty across the hyperparameter space that BO sequentially uses to select new hyperparameter combinations, balancing the trade-off between exploiting regions with low predicted values of $\mathcal{L} \left(\tau; \Theta \right)$ and exploring regions with high predictive uncertainty through an acquisition function. 

We use the BO implementation available in \texttt{Optuna} \cite{optuna2019}, which employs a Mat\'ern kernel $\mathcal{K}_{\nu = 5/2}$ with automatic relevance determination via a separate length-scale for each hyperparameter in $\Theta$, and the log-expected improvement acquisition criterion \cite{ament2023unexpected}. Hyperparameters whose admissible ranges span several orders of magnitude ($\rho_e$, $\rho_d$, and $\kappa_0$) are searched in $\log$ space, and integer-valued hyperparameters ($\kappa_0$, $\vartheta$) are rounded to the nearest admissible value. To initialize the GP surrogate, we draw 100 hyperparameter configurations via Latin Hypercube Sampling (LHS) \cite{mckay1979latin}; the BO then performs 100 additional guided evaluations, terminating at this fixed budget of 200 evaluations in total. The procedure is repeated across $R = 5$ random initializations.

\section{Results}\label{sec:results}

We present the results of our approach on three different case studies. First, we consider individual  wind assets (Section~\ref{sec:wind}) which is our main case study. To test the transferability and robust of our methods, we then consider individual solar assets (Section~\ref{sec:solar}) and zonally aggregated wind generation from real-life ERCOT data (Section \ref{sec:wind_zone}) which includes a  $48$-hour forecast horizon. In each case study we examine the tuned hyperparameters, sharpness, and calibration.

\subsection{Wind Power Generation}\label{sec:wind}

We consider  187 wind assets available in the 2030 Texas-7k ERCOT grid model for 2017 and 2018. The operation horizon is 24 hours and generation is recorded at 5 minute intervals, giving a total of $T=288$ periods.  The dataset captures a broad range of climatological and operational conditions, with the wind farms exhibiting substantial heterogeneity in installed capacity (Figure~\ref{fig:dataset}a), DA forecasting errors (Figure~\ref{fig:dataset}b), and average CFs (Figure~\ref{fig:dataset}c), providing a realistic testbed for our experiments. We remove systematic bias from the DA forecasts to reflect realistic operational conditions; 
later in this subsection we reintroduce the bias as a controlled robustness check.

\begin{table}[htb!]
    \centering
    \scriptsize
    \caption{\textbf{Tuning hyperparameters for dynamic updating of wind assets' unbiased DA forecast.} Estimated optimal hyperparameter for forecast update times $\tau=6$am, 12pm, and 6pm. The upper block lists the hyperparameters fine-tuned via Bayesian optimization. The lower block lists the hyperparameters held fixed across forecast update times.}
    \setlength{\tabcolsep}{2pt} 
    \renewcommand{\arraystretch}{1.25}
    \begin{tabular}{rc ccc}
    \toprule
    \textbf{Tuning Hyperparameter} & &
    \textbf{6am} & \textbf{12pm} & \textbf{6pm} \\
    \midrule
    Trust level                          & $\nu$    & 9.7\% & 20.2\% & 0\% \\
    Trust horizon                        & $\vartheta$ & 6~h   & 10~h   & 12~h    \\
    DA forecast similarity weight        & $\lambda_e$ & 21.8  & 11.4   & 13.5   \\
    Seasonal similarity weight           & $\lambda_d$ & 4.38  & 0     & 0.98   \\
    kNN step I neighborhood size         & $\kappa_0$  & 150   & 1,815  & 150   \\
    \midrule
    Partially observed curve forgetting rate & $\varepsilon_f$
        & \multicolumn{3}{c}{0.125~h} \\
    DA forecast past decay rate             & $\varepsilon_{e^-}$
        & \multicolumn{3}{c}{1~h} \\
    DA forecast future decay rate                     & $\varepsilon_{e^+}$
        & \multicolumn{3}{c}{4~h} \\
    Partially observed curve similarity weight & $\lambda_f$
        & \multicolumn{3}{c}{40} \\
    kNN step II neighborhood size              & $\kappa$
        & \multicolumn{3}{c}{150} \\
    Probability of fusion  with DA forecast                    & $p_{\mathrm{fusion}}$
        & \multicolumn{3}{c}{0.75} \\
    \bottomrule
\end{tabular}
    \label{tab:wind_unbiased_asset_hyperparameters}
\end{table}

To illustrate how fine-tuning evolves throughout the operational day, we analyze $\Theta^\star_\tau$  for $\tau \in {72, 144, 216}$, corresponding to 6am, 12pm, and 6pm. The similarity weights $\lambda_f$, whose inverse $\rho_f = 1/\lambda_f$ is the corresponding length-scale, define the neighborhood and are interpreted relative to the fixed weight on the partial curve $f_\star$ ($\lambda_f = 40$). The tuned hyperparameters indicate that the influence of the DA forecast changes over the course of the day (Table~\ref{tab:wind_unbiased_asset_hyperparameters}). The DA forecast is the primary channel in the morning. $\lambda_e/\lambda_f$ is largest at 6am ($0.55$) and roughly halves by 12pm ($0.29$), settling at $0.34$ by 6pm, while seasonal similarity remains secondary at every forecast update time ($\lambda_d/\lambda_e \leq 0.20$). In contrast, the trust level $\nu$ increases from 9.7\% at 6am to 20.2\% at midday, and the trust horizon $\vartheta$ extends from 6 to 12 hours. By 6pm, the DA forecast's influence on the trajectory is eliminated ($\nu = 0\%$), and the selected neighbors guide the forecast through the end of the day. Thus, the DA forecast informs the neighborhood selection early on, helps anchor trajectories midday, and becomes unnecessary by evening.

We now return to the 12pm case, in which we know the DA forecast has lost half its weight in defining the neighborhood, and the pull toward it at the end of the day has doubled (Table~\ref{tab:wind_unbiased_asset_hyperparameters}). To see  how the revised forecast uncertainty and functional envelopes adapt across days, we fix the asset (Blanco Canyon) and compare two days. On a typical day (Figure~\ref{fig:functional-intervals}a), the similarity weights are almost uniform over the neighborhood, so the effective sample size $\mathrm{ESS}(\tau) = 1 / \sum_{i=1}^{\kappa} \omega_i^2 \in [1, \kappa]$
\cite{kish1965survey} is close to the upper bound ($\mathrm{ESS}(\tau)\approx\kappa$). On an unusual day (Figure~\ref{fig:functional-intervals}b), the weights concentrate on a few neighbors ($\mathrm{ESS}(\tau)\ll \kappa)$. For each day, the selected curves are ranked by $\mathcal{MBD}$ in Eq.~\eqref{eq:mbd} and shown in order of depth (Figure~\ref{fig:functional-intervals}c,d). On the unusual day, the trajectories share no common trend, so the updated forecasts revert only weakly toward the DA forecast (Figure~\ref{fig:functional-intervals}d), and the {forecasting envelope} is wider, with heavier tails. The 60\% predictive region ($\alpha = 0.4$) is wider on that day for all four constructions (functional envelope $\mathcal{R}^{\mathrm{fun}}_\alpha$, its adjusted variation $\mathcal{R}^{\mathrm{adj}}_\alpha$, the projection envelope $\mathcal{R}^{\mathrm{prj}}_\alpha$, and the marginal quantiles $\mathcal{R}^{\mathrm{mar}}_\alpha$); see Figure~\ref{fig:functional-intervals}e--l. At the 90\% level ($\alpha = 0.1$), the projection envelope and the marginal quantiles give bands of comparable width on both days, widening toward the end of the operational day, whereas the adjusted functional envelope produces a narrower and less conservative band. The envelope width therefore depends on how compact the neighborhoods are, not on a deterministic, time-dependent variance. 

\subsection{Forecast Assessment}

\begin{table*}[htb!]
    \centering
    \scriptsize
    \caption{\textbf{Wind asset unbiased DA forecast probabilistic revision confidence band scores.} Functional interval score (FIS), functional coverage score (FCS), and simultaneous coverage probability (SCP) at significance levels $\alpha \in \{0.1, 0.2\}$ and forecast update times $\tau \in \{72, 144, 216\}$ (i.e.,~6am, 12pm, and 6pm). The confidence regions are constructed from the adjusted functional envelope $\mathcal{R}^{\mathrm{adj}}_\alpha$, the projection envelope $\mathcal{R}^{\mathrm{prj}}_\alpha$, the unadjusted functional envelope $\mathcal{R}^{\mathrm{fun}}_\alpha$, and the weighted marginal quantiles $\mathcal{R}^{\mathrm{mar}}_\alpha$. For each construction, we report the metric ``score'' and the fraction $k_\alpha$.}
    \begin{adjustbox}{max width=\linewidth}
        \setlength{\tabcolsep}{2.5pt} 
        \renewcommand{\arraystretch}{1.25}
        \begin{tabular}{ll cccccc cccccc cccccc}
    \toprule
    & & \multicolumn{6}{c}{$\boldsymbol{\overline{\mathrm{FSI}}({\alpha})}$}
      & \multicolumn{6}{c}{$\boldsymbol{\overline{\mathrm{FCS}}({\alpha})}$}
      & \multicolumn{6}{c}{$\boldsymbol{\overline{\mathrm{SCP}}({\alpha})}$} \\
    \cmidrule(lr){3-8} \cmidrule(lr){9-14} \cmidrule(lr){15-20}
    & & \multicolumn{2}{c}{$\boldsymbol{\mathcal{R}^{\mathrm{adj}}_{\alpha}}$}
      & \multicolumn{2}{c}{$\boldsymbol{\mathcal{R}^{\mathrm{prj}}_{\alpha}}$}
      & $\boldsymbol{\mathcal{R}^{\mathrm{fun}}_{\alpha}}$ & $\boldsymbol{\mathcal{R}^{\mathrm{mar}}_{\alpha}}$ 
      & \multicolumn{2}{c}{$\boldsymbol{\mathcal{R}^{\mathrm{adj}}_{\alpha}}$}
      & \multicolumn{2}{c}{$\boldsymbol{\mathcal{R}^{\mathrm{prj}}_{\alpha}}$}
      & $\boldsymbol{\mathcal{R}^{\mathrm{fun}}_{\alpha}}$ & $\boldsymbol{\mathcal{R}^{\mathrm{mar}}_{\alpha}}$ 
      & \multicolumn{2}{c}{$\boldsymbol{\mathcal{R}^{\mathrm{adj}}_{\alpha}}$}
      & \multicolumn{2}{c}{$\boldsymbol{\mathcal{R}^{\mathrm{prj}}_{\alpha}}$}
      & $\boldsymbol{\mathcal{R}^{\mathrm{fun}}_{\alpha}}$ & $\boldsymbol{\mathcal{R}^{\mathrm{mar}}_{\alpha}}$ \\
    \cmidrule(lr){3-4} \cmidrule(lr){5-6} \cmidrule(lr){7-7} \cmidrule(lr){8-8}
    \cmidrule(lr){9-10} \cmidrule(lr){11-12} \cmidrule(lr){13-13} \cmidrule(lr){14-14}
    \cmidrule(lr){15-16} \cmidrule(lr){17-18} \cmidrule(lr){19-19} \cmidrule(lr){20-20}
    $\boldsymbol{\tau}$ & $\boldsymbol{\alpha}$
    & score & $k_\alpha$ & score & $k_\alpha$ & score & score
    & score & $k_\alpha$ & score & $k_\alpha$ & score & score
    & score & $k_\alpha$ & score & $k_\alpha$ & score & score \\
    \midrule
    \multirow[c]{2}{*}{6am} & 0.1 & 0.566 & 0.60 & 0.560 & 0.75 & 0.621 & \textbf{0.532} & 0.915 & 0.60 & 0.913 & 0.65 & 0.971 & \textbf{0.898} & \textbf{0.743} & 1.00 & 0.682 & 1.00 & 0.621 & 0.283 \\
     & 0.2 & 0.473 & 0.32 & 0.469 & 0.45 & 0.568 & \textbf{0.447} & 0.806 & 0.32 & \textbf{0.804} & 0.32 & 0.956 & 0.805 & \textbf{0.743} & 1.00 & 0.682 & 1.00 & 0.516 & 0.103 \\
     \midrule
    \multirow[c]{2}{*}{12pm} & 0.1 & 0.624 & 0.65 & 0.622 & 0.75 & 0.681 & \textbf{0.585} & 0.912 & 0.65 & 0.912 & 0.70 & 0.969 & \textbf{0.901} & \textbf{0.770} & 1.00 & 0.712 & 1.00 & 0.661 & 0.387 \\
     & 0.2 & 0.526 & 0.40 & 0.522 & 0.50 & 0.618 & \textbf{0.494} & 0.810 & 0.40 & \textbf{0.804} & 0.40 & 0.951 & 0.808 & \textbf{0.770} & 1.00 & 0.712 & 1.00 & 0.567 & 0.190 \\
     \midrule
    \multirow[c]{2}{*}{6pm} & 0.1 & 0.646 & 0.75 & 0.646 & 0.85 & 0.685 & \textbf{0.612} & \textbf{0.904} & 0.70 & \textbf{0.904} & 0.75 & 0.967 & 0.909 & \textbf{0.879} & 1.00 & 0.822 & 1.00 & 0.762 & 0.514 \\
     & 0.2 & 0.538 & 0.50 & 0.540 & 0.60 & 0.603 & \textbf{0.518} & 0.810 & 0.50 & \textbf{0.800} & 0.50 & 0.940 & 0.813 & \textbf{0.814} & 0.95 & 0.822 & 1.00 & 0.653 & 0.309 \\
    \bottomrule
\end{tabular}
    \end{adjustbox}
    \label{tab:wind_unbiased_asset_envelope}
\end{table*}

Table~\ref{tab:wind_unbiased_asset_envelope} evaluates the four predictive band constructions  in terms of sharpness and coverage. The weighted marginal quantiles $\mathcal{R}^{\mathrm{mar}}_\alpha$ attain the lowest FIS at every $\tau$ and $\alpha$, with the adjusted functional and projection envelopes only $4$--$7\%$ behind and the functional envelope $12$--$27\%$ behind. Under FCS, the first three constructions are all close to nominal (within $0.015$ at both the $90\%$ and $80\%$ levels), while $\mathcal{R}^{\mathrm{fun}}_\alpha$ is systematically too conservative. Within that margin, the ranking depends on the significance level: the projection envelope $\mathcal{R}^{\mathrm{prj}}_\alpha$ is the closest to nominal at the $80\%$ level at every update time ($0.804$, $0.804$, and $0.800$) and never deviates by more than $0.013$ at either level, whereas the marginal quantiles are closest at the $90\%$ level. The adjusted variation $\mathcal{R}^{\mathrm{adj}}_\alpha$ corrects the over-coverage of $\mathcal{R}^{\mathrm{fun}}_\alpha$ but at a cost of $4$--$7\%$ in FIS. On these two criteria, then, the pointwise construction is the sharpest while remaining within $0.013$ of nominal.

\begin{table}[htb!]
    \centering
    \scriptsize
    \caption{\textbf{Wind asset unbiased DA forecast probabilistic revision deterministic error metrics.} Root mean squared error (RMSE), mean absolute error (MAE), and mean bias error (MBE) of the point forecasts at forecast update times $\tau \in \{72, 144, 216\}$ (i.e.,~6am, 12pm, and 6pm). The estimators are the weighted marginal median curve $\hat{\mu}_{\text{med}}(s)$, the deepest curve $\hat{\mu}_{\text{deep}}(s)$ ranked by Modified Band Depth, the focal curve $\hat{\mu}_{\text{focal}}(s)$, the ARIMA and Bootstrapped Re-anchored Day-Ahead (BRDA) forecast only. Errors are computed over the forecast horizon $\mathcal{I}_{>\tau}$ and expressed in capacity factor units.}

    \setlength{\tabcolsep}{10pt} 
    \renewcommand{\arraystretch}{1.25}

\begin{tabular}{l rrr rrr rrr}
    \toprule
     \textbf{Estimator} &
    \multicolumn{3}{c}{$\tau=6$am} & \multicolumn{3}{c}{$\tau=12$pm} & \multicolumn{3}{c}{$\tau=6$pm} \\  
    \cmidrule(lr){2-4} \cmidrule(lr){5-7} \cmidrule(lr){8-10} 
    &   \textbf{RMSE} & \textbf{MAE} & \textbf{MBE} & \textbf{RMSE} & \textbf{MAE} & \textbf{MBE} & \textbf{RMSE} & \textbf{MAE} & \textbf{MBE}\\ \midrule
      $\hat{\mu}_{\mathrm{mar}}$ & \textbf{0.132} & \textbf{0.099} & 0.003 & \textbf{0.144} & \textbf{0.111} & -0.004 
     & 0.140 & \textbf{0.117} & -0.006
     \\
      $\hat{\mu}_{\mathrm{deep}}$ & 0.142 & 0.107 & 0.002 
     & 0.155 & 0.119 & \textbf{0.0}
     & 0.143 & 0.119 & -0.005
     \\
      $\hat{\mu}_{\mathrm{focal}}$ & \textbf{0.132} & 0.103 & -0.003 
     & 0.147 & 0.116 & -0.003
     & \textbf{0.139} & 0.119 & \textbf{-0.004}
     \\
     \midrule
      ARIMA & 0.249 & 0.204 & -0.036
     & 0.217 & 0.171 & 0.046
     & 0.206 & 0.177 & 0.073
     \\
        BRDA & 0.143 & 0.109 & \textbf{0.001} 
     & 0.148 & 0.115 & 0.001
     & 0.159 & 0.135 & 0.005 \\
    \bottomrule
\end{tabular}

    \label{tab:wind_asset-error-unbiased}
\end{table}

The stricter SCP reverses the earlier conclusion. Unlike pointwise bands, SCP requires the entire realized trajectory to lie within the band, not just individual time points. At $\alpha=0.1$, the adjusted envelope achieves SCPs of $0.743$--$0.879$ and the projection envelope $0.682$--$0.822$, while marginal quantiles perform poorly ($0.283$--$0.514$), failing as simultaneous bands at every update time. The ordering $\mathcal{R}^{\mathrm{adj}} > \mathcal{R}^{\mathrm{prj}} > \mathcal{R}^{\mathrm{fun}} \gg \mathcal{R}^{\mathrm{mar}}$ holds across both significance levels and all update times, with the exception of $\alpha=0.2$ at 6pm, where the $\mathcal{R}^{\mathrm{prj}}$ overtakes $\mathcal{R}^{\mathrm{adj}}$ ($0.822$ against $0.814$). This highlights the sharpness–simultaneity trade-off: pointwise bands are calibrated for each instant but do not capture the reliability of the entire trajectory. For operators making multi-hour commitments, only $\mathcal{R}^{\mathrm{adj}}$ and $\mathcal{R}^{\mathrm{prj}}$ provide meaningful information about the whole curve. Notably, the focal curve is the most accurate point estimator at 6pm by both RMSE and MBE (Table~\ref{tab:wind_asset-error-unbiased}), making the combination of $\hat{\mu}_{\mathrm{focal}}$ and $\mathcal{R}^{\mathrm{prj}}_\alpha$ the best choice for users needing an accurate estimate, sharp intervals, and wide coverage. However, no method reaches nominal simultaneous coverage at $\alpha=0.1$, with the neighbors fraction saturating at one. The envelope spans every neighbor, so the limiting factor is not only its width but also how well-centered their neighborhood are.

Table~\ref{tab:wind_asset-error-unbiased} reports the quality of the point forecasts.  All three forecasts are essentially unbiased ($|\mathrm{MBE}| \leq 0.006$). The best RMSE is provided by the weighted marginal median curve $\hat{\mu}_{mar}$ and the focal curve $\hat{\mu}_{\mathrm{focal}}$, which differ by less than 2\% in RMSE at all update times.  $\hat{\mu}_{\mathrm{deep}}$ consistently yields the highest RMSE and MAE. Notably, the forecast errors are stable across $\tau$'s, in part because the largest errors occur during the evening ramp and the overnight period.  

The two baselines separate clearly from our probabilistic revision (see details in Supplementary Note~\ref{sec:baselines}). ARIMA is the least accurate estimator at every update time (RMSE $0.218$--$0.250$). In contrast, BRDA is competitive at 6am and 12pm (RMSE $0.143$ and $0.148$ against $0.132$ and $0.144$) but by 6pm the focal curve is 14.4\% more accurate ($0.139$ against $0.159$). 

Calibration of the \emph{probabilistic revision} is stable even when assessed interval by interval.  Each interval $[T_0, T_1)$ has a start time $T_0 = \tau + k$, with lead time $k \in \{0, 2, 4, \ldots, 22 - \tau\}$~h, and an ending time $T_1 := T_0 + \Delta$, with $\Delta = 2$h for wind assets. The Kolmogorov--Smirnov statistic $D_{KS}$ of the PIT over two-hour intervals never exceeds $0.049$ at any update time (Table~\ref{tab:wind_asset-ks-unbiased}), and shows no change with $\tau$. However, the final interval of the day is consistently the most weakly calibrated in all three columns ($0.030$, $0.040$, $0.031$). The two comparator baselines fail the same test by a wide margin. ARIMA reaches $D_{KS}=0.193$ in the first interval after the 6am update and performs worse than our algorithm in every interval. BRDA is better behaved but still peaks at $0.098$, roughly twice as worse as us. 

\subsection{Rolling  Intra-day Forecast Revisions}

\begin{figure*}[htb!]
    \centering
    \begin{minipage}{.49\textwidth} 
        \centering
        \begin{overpic}[width=1\textwidth, trim = {0cm, 0cm, 0cm, 0cm}, clip]{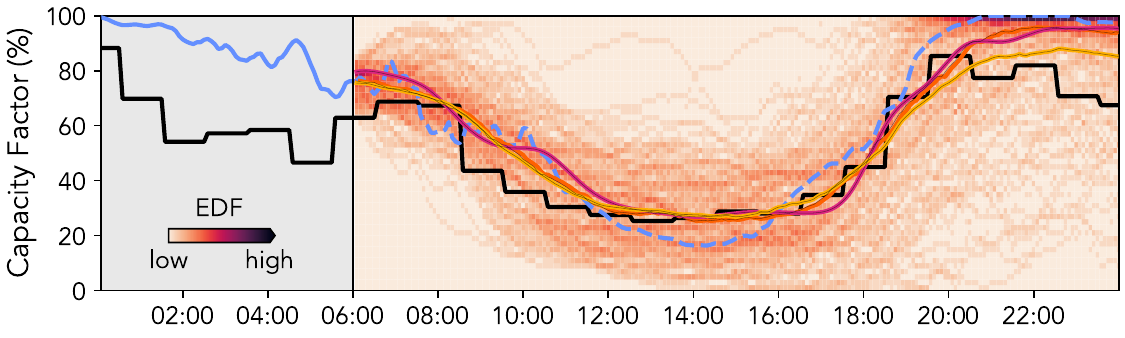}
            \put(0, 30){\textbf{a}}
        \end{overpic}
        \begin{overpic}[width=1\textwidth, trim = {0cm, 0cm, 0cm, 0cm}, clip]{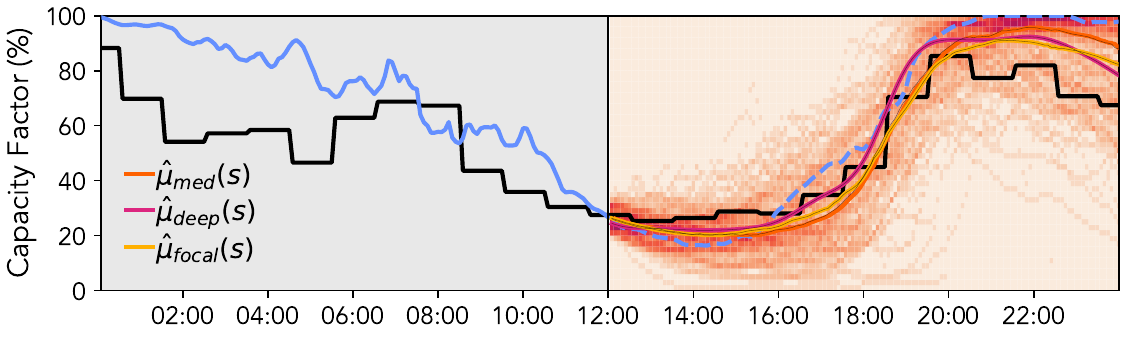}
            \put(0, 30){\textbf{c}}
        \end{overpic}
        \begin{overpic}[width=1\textwidth, trim = {0cm, 0cm, 0cm, 0cm}, clip]{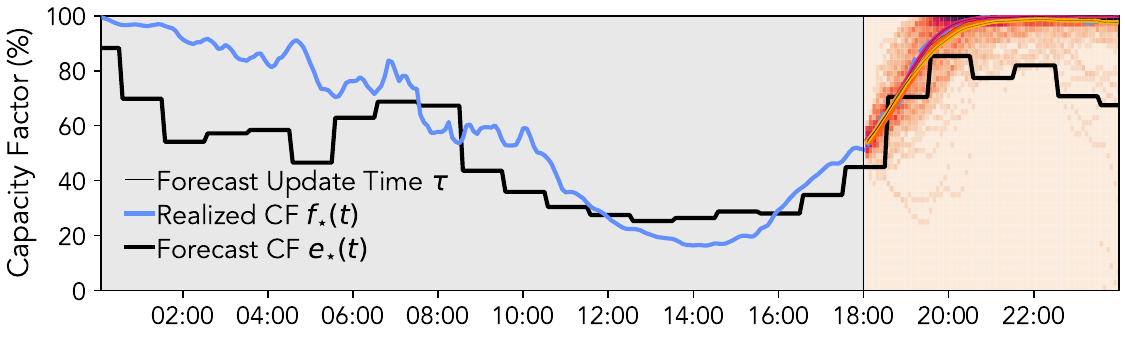}
            \put(0, 30){\textbf{e}}
        \end{overpic}
    \end{minipage}
    \begin{minipage}{.49\textwidth} 
        \centering
        \begin{overpic}[width=1\textwidth, trim = {0cm, 0cm, 0cm, 0cm}, clip]{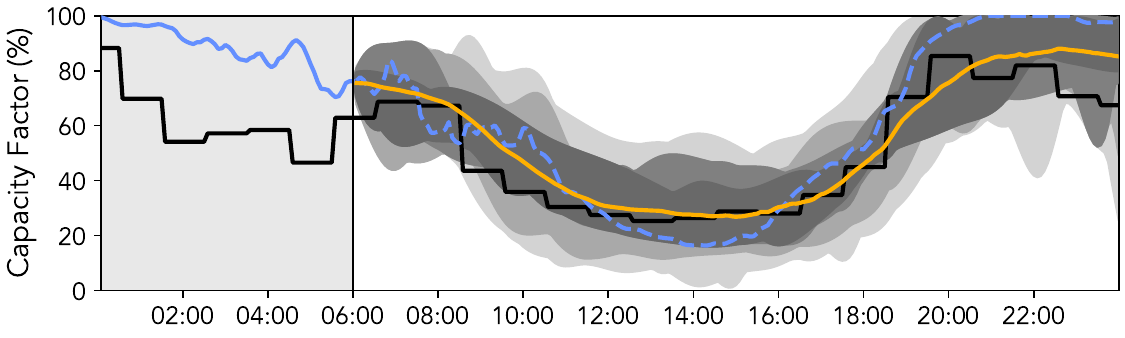}
            \put(0, 30){\textbf{b}}
        \end{overpic}
        \begin{overpic}[width=1\textwidth, trim = {0cm, 0cm, 0cm, 0cm}, clip]{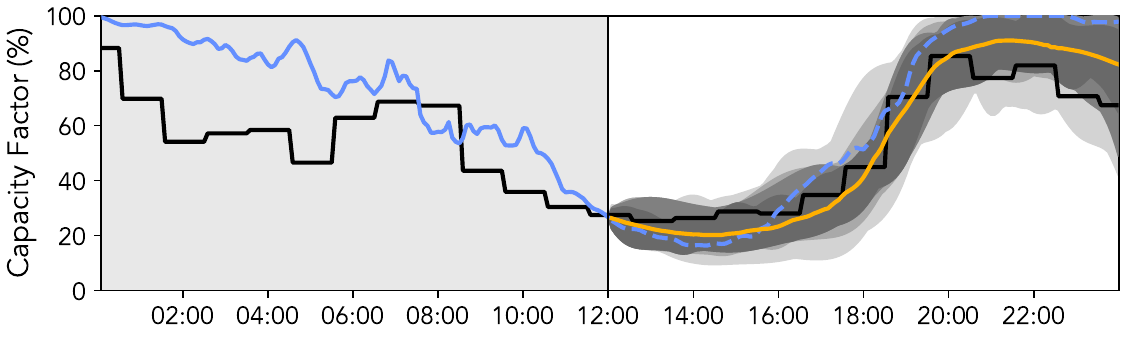}
            \put(0, 30){\textbf{d}}
        \end{overpic}
        \begin{overpic}[width=1\textwidth, trim = {0cm, 0cm, 0cm, 0cm}, clip]{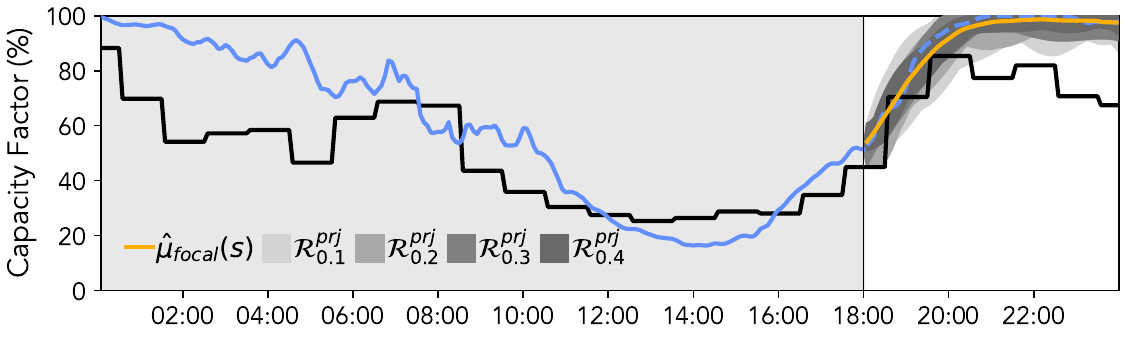}
            \put(0, 30){\textbf{f}}
        \end{overpic}
    \end{minipage}
    \centering
    \begin{overpic}[width=.165\textwidth, trim = {-.5cm, 0cm, 0cm, 0cm}, clip]{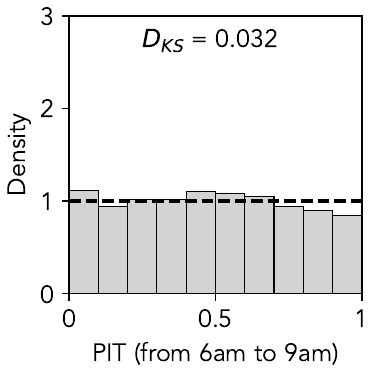}
        \put(2, 85.5){\textbf{g}}
    \end{overpic}
    \includegraphics[width=.1525\textwidth, trim = {0cm, 0cm, 0cm, 0cm}, clip]{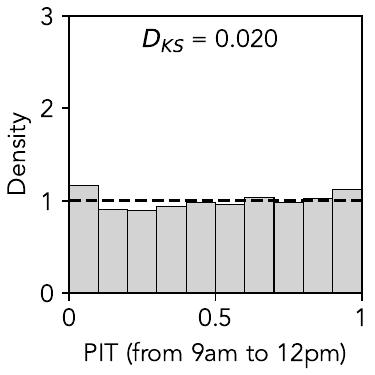}
    \begin{overpic}[width=.1525\textwidth, trim = {0cm, 0cm, 0cm, 0cm}, clip]{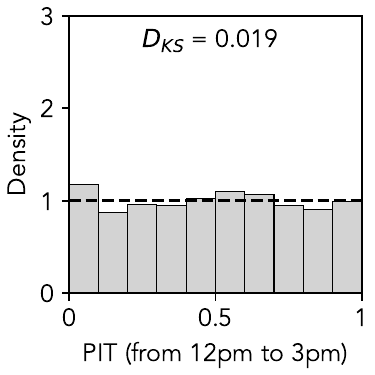}
        \put(-2, 90){\textbf{h}}
    \end{overpic}
    \includegraphics[width=.165\textwidth, trim = {-.5cm, 0cm, 0cm, 0cm}, clip]{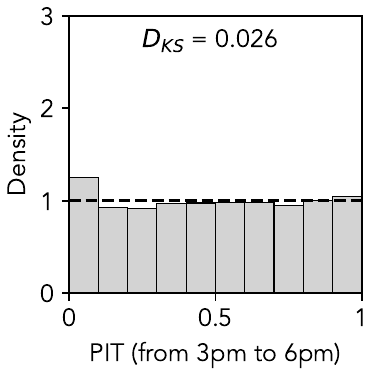}
    \begin{overpic}[width=.1525\textwidth, trim = {0cm, 0cm, 0cm, 0cm}, clip]{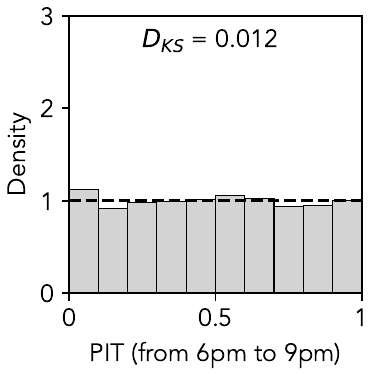}
        \put(0, 90){\textbf{i}}
    \end{overpic}
    \includegraphics[width=.1525\textwidth, trim = {0cm, 0cm, 0cm, 0cm}, clip]{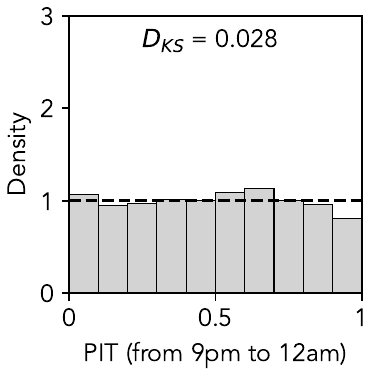}
    \caption{\textbf{Wind generation DA forecast functional probabilistic revision.}  intra-day forecasting for Barranca Wind asset on May 1. We consider three update times: $\tau=$ 6am (a, b, g), 12pm (c, d, h) and 6pm (e, f, i). Empirical density function (EDF) heatmap (a, c, e) with the median curves $\hat{\mu}_{\text{med}}$ Eq.~\eqref{eq:wquantile} (orange), focal curves $\hat{\mu}_{\text{focal}}$ Eq.~\eqref{eq:focal} (yellow), and deepest curves $\hat{\mu}_{\text{deep}}$ Eq.~\eqref{eq:central-curve} (violet). The focal curve envelope confidence regions (b, d, f). Bottom row shows the PIT evaluated in test dataset from the update events ---6am (g), 12pm (h) and 6pm (i)--- to 3h lead time (g, h, i, left), and from 3h lead time to 6h lead time (g, h, i, right). }    
    \label{fig:dynamic-pits}
\end{figure*}

We now let the update time vary, so that what changes across the revised forecasts is how much of the fixed target day has already been observed. Figure~\ref{fig:dynamic-pits} follows the Barranca Wind asset on May 1, with updates issued at $\tau=6$am, 12pm, and 6pm, each conditioned on the realized CFs up to $\tau$. As realizations accumulate, the marginal EDF concentrates, and the median $\hat{\mu}_{\text{med}}$, deepest $\hat{\mu}_{\text{deep}}$, and focal $\hat{\mu}_{\text{focal}}$ curves converge to one another (Figure~\ref{fig:dynamic-pits}a,c,e). Similarly, the focal envelopes $\mathcal{R}^{\mathrm{prj}}_\alpha$ contract at every confidence level (Figure~\ref{fig:dynamic-pits}b,d,f). The updates are not only sharper but better located.
The DA forecast $e_\star(s)$ understates the evening ramp (between $8$pm--$12$am) by $20$--$30\%$, a gap the $6$am update already closes in part and the $12$pm update closes entirely, after which $\hat{\mu}_{\text{focal}}(s)$ tracks the realized CF rather than $e_\star(s)$.

Calibration is preserved as the updates accumulate. The PITs in Figure~\ref{fig:dynamic-pits}g--i fix the lead time from each update event, evaluating the first three hours after $\tau$ separately from the following three. All six histograms are close to uniform, with $D_{KS}$ between $0.012$ and $0.032$ and no deterioration from the $0$--$3$h band to the $3$--$6$h band. These values fall below the bound obtained over the finer two-hour intervals of Table~\ref{tab:wind_asset-ks-unbiased}, as expected from the increases in observations per window. The weak calibration of the final interval reappears in the last window ($9$pm--$12$am, $D_{KS}=0.028$) but stays small, so the updated envelopes reflect an information gain rather than miscalibration.

\begin{figure*}[htb!]
    \centering
    \begin{minipage}{.49\textwidth} 
        \centering
        \begin{overpic}[width=1\textwidth, trim = {0cm, 0cm, 0cm, 0cm}, clip]{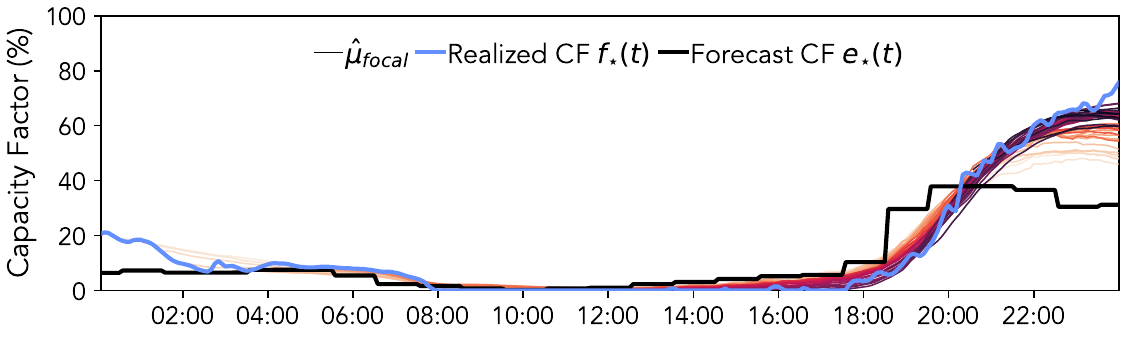}
            \put(10, 24){{\small Apr 23}}
        \end{overpic}
    \end{minipage}
    \begin{minipage}{.49\textwidth} 
        \centering
        \begin{overpic}[width=1\textwidth, trim = {0cm, 0cm, 0cm, 0cm}, clip]{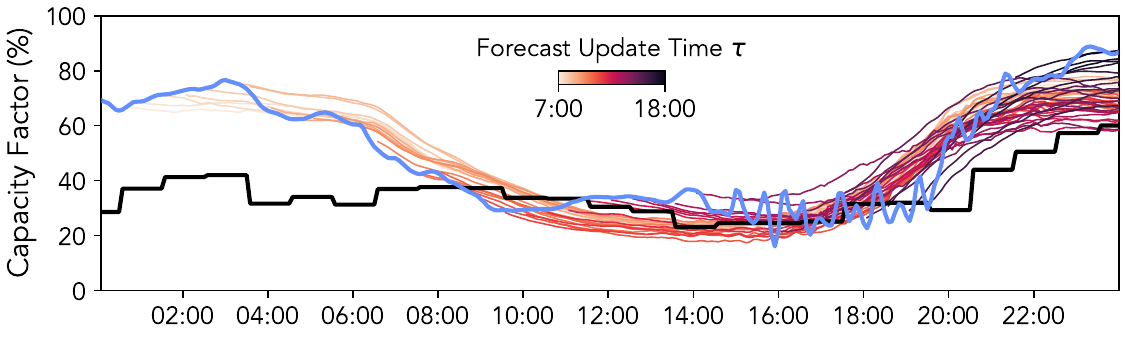}
            \put(10, 6){{\small Jun 26}}
        \end{overpic}
    \end{minipage}
    \caption{\textbf{Rolling dynamic updating.} Blanco Canyon wind asset on Apr 23 (left) and Jun 26 (right), matching Fig~\ref{fig:functional-intervals}. The focal curve $\hat{\mu}_{\text{focal}}(s), s\in \mathcal{I}_{>\tau}$ is updated every 30 minutes from $\tau=$12:30am to $\tau=12$am + $1$ day as new observations $f_\star(s)$ become available, with the colors indicating the forecast update time $\tau$. The realized trajectory $f(\cdot)$ is in blue and the original DA forecast (given only hourly) $e_\star(\cdot)$ is in black.}    
    \label{fig:focal-envelope}
\end{figure*}


Our intra-day revisions delivers forecasts at any update time, so with realizations arriving every 5 minutes one can build a matching stream of dynamic updates. To do so, we fine-tune hyperparameters for $\tau=6am, 12pm$ and 6pm and then linearly interpolate each of them for intermediate updates $\tau$; e.g.,~the trust horizon $\vartheta$ at 2pm is $(2/3)\vartheta_{12pm}+(1/3)\vartheta_{6pm}$.

Figure~\ref{fig:focal-envelope} shows the focal curve $\hat{\mu}_{\text{focal}}(s)$ updated every 30 minutes for the Blanco Canyon asset on the two contrasting days (we do not show the respective functional envelopes as the visualization is messy, but makes for a cool ``movie''). On Apr 23, the partial curve falls in a dense neighborhood and successive updates overlap, so the afternoon point forecast barely moves beyond the 11am update. On Jun 26, the neighborhood is sparse, and each new observation continues to move the focal curve through the evening ramp. 


\subsection{Case Study II: Intra-day Probabilistic Revisions for Solar Power Generation}\label{sec:solar}

We next run our probabilistic intra-day revisions on the 211 solar assets in the 2030 Texas-7k grid model for 2017 and 2018.
Solar generation is a hard test of transferability because it follows a deterministic diurnal trajectory pattern that wind lacks, and its uncertainty collapses at sunset rather than growing with lead time.  As with wind, solar generators exhibit substantial heterogeneity in installed capacity (Figure~\ref{fig:dataset}a), DA forecasting error (Figure~\ref{fig:dataset}b), and average CFs (Figure~\ref{fig:dataset}c). DA solar forecasts are similarly de-biased using Eq.~\eqref{eq:bias}. The proposed equinox-informed distance in Eq.~\eqref{eq:elevation_dist} is particularly useful for solar generation trajectories, since what makes two days comparable is the solar elevation profile rather than proximity in the calendar. The effect of this distance is visible in the neighborhood construction for Blue Bell Solar II on Aug 8 (Figure~\ref{fig:spatio-temporal-solar-ercot}). The functional neighbors are drawn almost exclusively from dates with a comparable solar elevation profile, concentrating around the target date, while the Haversine filter subsequently retains only the geographically consistent assets among them.

\begin{figure*}[htb!]
    \centering
    \begin{minipage}{.48\textwidth} 
        \centering
        \begin{overpic}[width=1.125\textwidth, trim = {0cm, 0cm, 0cm, 0cm}, clip]{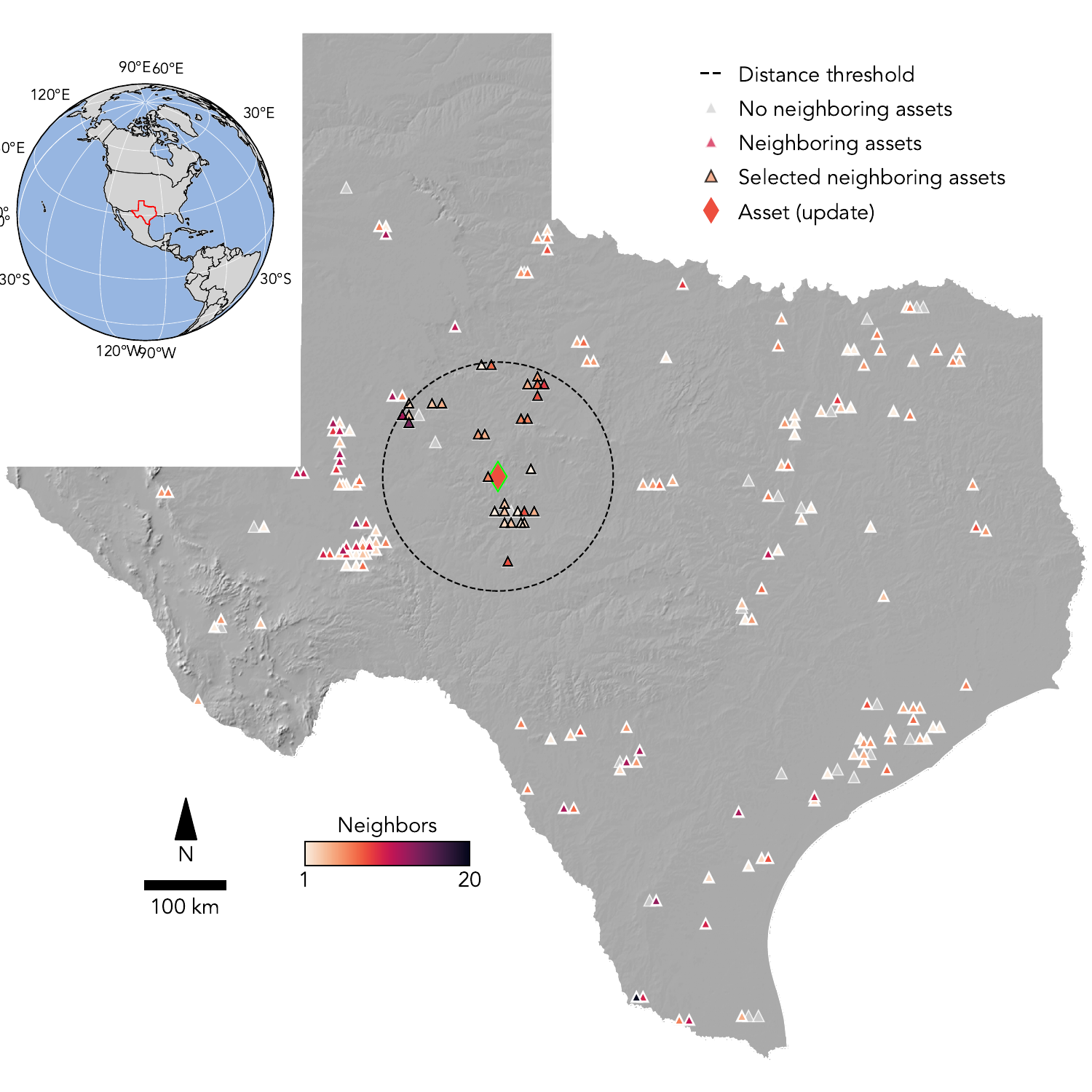}
        \end{overpic}
    \end{minipage}
    \begin{minipage}{.50\textwidth} 
        \centering
        \begin{overpic}[width=.75\textwidth, trim = {0cm, 0cm, 0cm, 0cm}, clip]{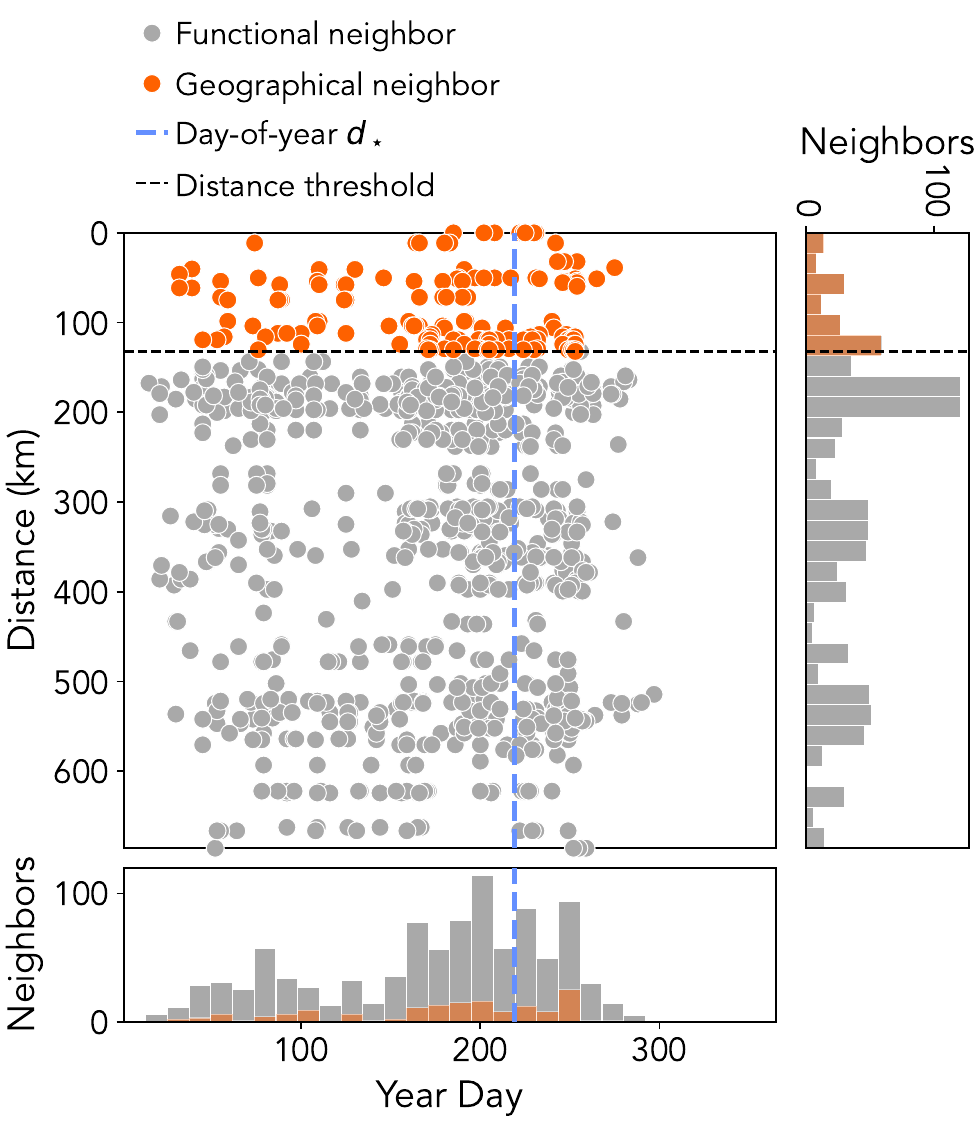}
        \end{overpic}
    \end{minipage}
    \caption{\textbf{Functional neighborhood for solar generation.} Neighboring curves selected for Blue Bell Solar II (indicated with a bright green diamond) asset on Aug 8 (left panel). Gray markers represent assets without any similar curves to Blue Bell Solar II. Shaded markers with a white edge represent $\kappa_0$ solar assets with neighboring curves before applying the distance filter, and shaded markers with a black edge those with neighbors after applying the distance filter. The Haversine distance threshold, defined by the most distant functional neighbor, is the dashed black circle. Frequency of similar curves in the $\kappa$-neighborhood for each month, ranked by the distance to Blue Bell Solar II (right panel), before further filtering out by seasonal distance. In all panels, light colors represent low frequency and darker colors higher frequency.}
    \label{fig:spatio-temporal-solar-ercot}
\end{figure*}

Solar generation is a cyclostationary stochastic process with seasonal and diurnal periodicity \cite{terren2023processing}. This complicates comparisons of generation curves across days, since some variation is astronomical rather than meteorological. The standard normalization is the clear-sky index (CSI) which is not applicable in our setting as 
our probabilistic revision operates directly on generation CFs and we do not have access to asset-level irradiance. To address this, we restrict the operational-day interval to the daylight window $\mathcal{I}(d_{\star}) = \left[t_0(d_{\star}),\, t_1 (d_{\star})\right]$, excluding the hours that registered zero generation on all of the previous seven days. We then restrict all similarity computations to $\mathcal{I}(d_{\star})$.
As a result, solar generation has minimal edge-inflation and admits $p_{\mathrm{fusion}} = 0.9$, compared to $p_{\mathrm{fusion}} = 0.75$ for wind. Solar profiles tend to be less variable, so we take a larger $\lambda_f = 100$ (rather than the $40$ used for wind), which separates near from distant neighbors and flattens the PIT histograms.

Supplementary Table~\ref{tab:asset_hyperparameters} lists the fine-tuned hyperparameters for solar generation. For solar, the DA forecast carries much less weight ($\lambda_e/\lambda_f < 0.03$) in the morning and becomes irrelevant in the afternoon when $\lambda_e$ is essentially zero. Seasonal similarity goes in the opposite direction, steadily increasing in forecast update time. Thus, DA similarity matters more in the morning and seasonal similarity dominates in the afternoon, where it contributes nearly as much as CF similarity. Compared to wind, time-of-year is much more informative for solar, forming an astronomical baseline for information fusion. Supplementary Figure~\ref{fig:solar_asset_neighbors} shows the overall neighborhood construction for a solar asset and the spatiotemporal distribution of the $\kappa=150$ neighbors selected.


\begin{figure*}[htb!]
    \centering
    \begin{minipage}{.49\textwidth} 
        \centering
        \begin{overpic}[width=.9\textwidth, trim = {0cm, 0cm, 0cm, 0cm}, clip]{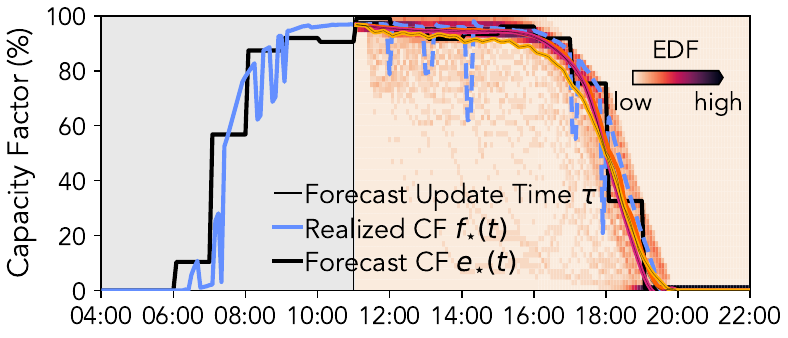}
            \put(0, 42.5){\textbf{a}}
        \end{overpic}
        \begin{overpic}[width=.9\textwidth, trim = {0cm, 0cm, 0cm, 0cm}, clip]{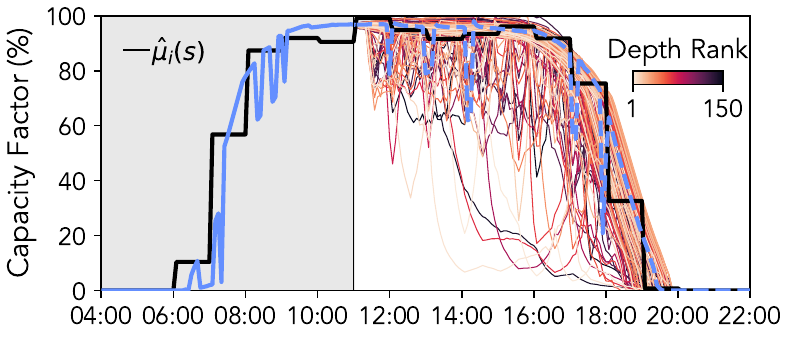}
            \put(0, 42.5){\textbf{c}}
        \end{overpic}
        \begin{overpic}[width=.9\textwidth, trim = {0cm, -.5cm, 0cm, 0cm}, clip]{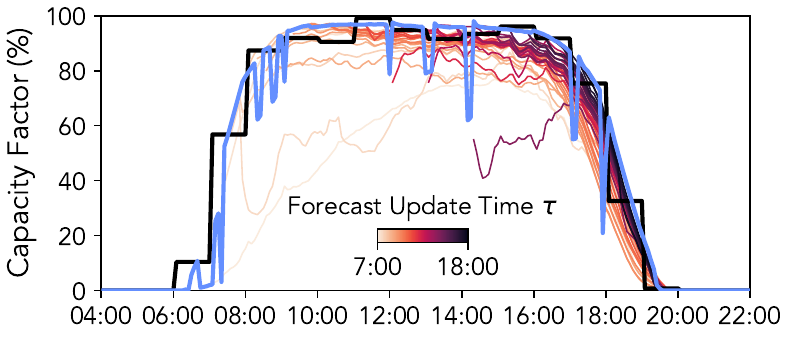}
            \put(0, 47.5){\textbf{e}}
        \end{overpic}
    \end{minipage}
    \begin{minipage}{0.49\textwidth}
        \centering
        \begin{overpic}[width=.9\textwidth, trim = {0cm, 0cm, 0cm, 0cm}, clip]{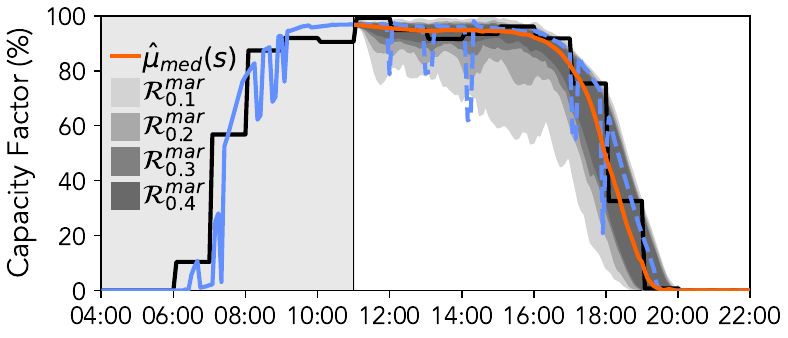}
            \put(0, 42.5){\textbf{b}}
        \end{overpic}
        \begin{overpic}[width=.9\textwidth, trim = {0cm, 0cm, 0cm, 0cm}, clip]{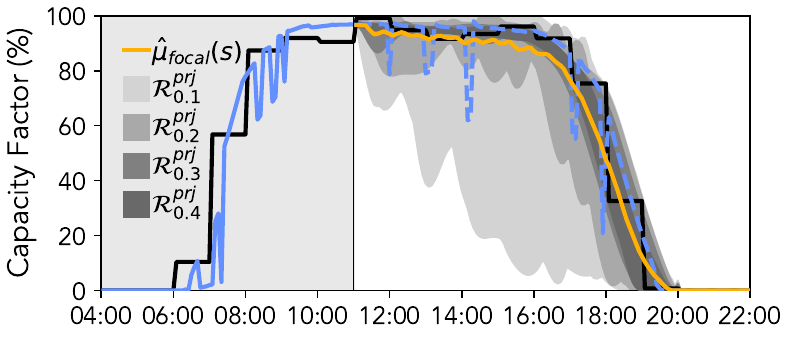}
            \put(0, 42.5){\textbf{d}}
        \end{overpic}
          \begin{overpic}[width=.4\textwidth, trim = {0cm, 0cm, 0cm, 0cm}, clip]{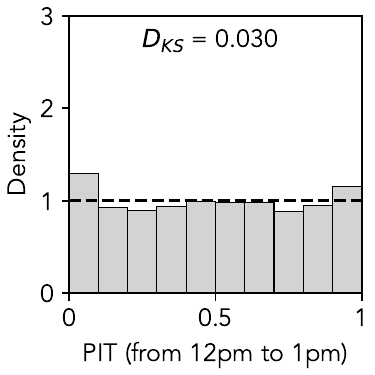}  
            \put(-5, 100){\textbf{f}}
        \end{overpic}  
          \begin{overpic}[width=.4\textwidth, trim = {0cm, 0cm, 0cm, 0cm}, clip]{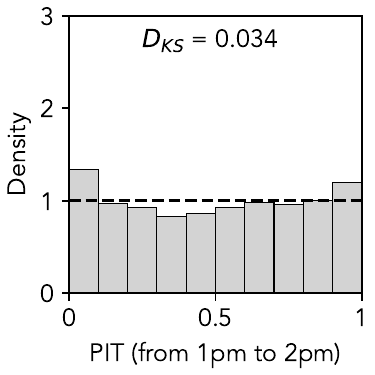}
        \end{overpic} 
    \end{minipage}
    \caption{\textbf{Solar generation DA forecast functional probabilistic revision.} Blue Bell Solar II asset on Aug 8 at forecast update time $\tau=11$am. Marginal empirical density function (EDF) heatmap with the median curve $\hat{\mu}_{\text{med}}$ in Eq.~\eqref{eq:wquantile} (orange), focal curve $\hat{\mu}_{\text{focal}}$ in Eq.~\eqref{eq:focal} (yellow), and deepest curve $\hat{\mu}_{\text{deep}}$ in Eq.~\eqref{eq:central-curve} (violet)  (a). Marginal quantile regions (b), fused forecasts $\hat{\mu}_i$ ranked by Modified Band Depth (MBD) (c), and focal envelope regions (d). Rolling probabilistic revision with the focal curve updated every 30 minutes from $\tau=7$am to $\tau=6$pm as new observations become available (e). PIT histograms at forecast update time $\tau=11$am, evaluated on lead times $s\in$ 12-1pm  (f, left) and $s\in$1-2pm (f, right).}
    \label{fig:solar_update}
\end{figure*}


Figure~\ref{fig:solar_update} illustrates the probabilistic CF completion at forecast update time $\tau=$~11am. We note the ``long tails'' of the predictive distribution during mid-day when solar generation can experience downward spikes due to passing clouds. These solar down-ramps (which are absent from the DA forecast and did end up realizing again in the afternoon) are successfully captured by a few of the ensemble members. While the marginal quantile regions $\cR^{\mathrm{mar}}$ are quite tight for all $s \in \cI_{>\tau}$ (Figure~\ref{fig:solar_update}b), the focal envelope $\cR^{\mathrm{prj}}$ retains each neighbor as a whole trajectory and therefore stays open to the cloud-driven ramps that some past-similar days have exhibited (Figure~\ref{fig:solar_update}d), highlighting the difference in the respective predictive band constructions. 


Figure~\ref{fig:solar_update}e shows a rolling update, visualizing the focal curve $\hat{\mu}_{\mathrm{focal}}$ updated every 30 minutes. Notably, the early-morning updates are very volatile due to the repeated down-ramps. Strong mean-reversion to the typical CF trajectory can be observed when forecast updates coincide with further cloudiness (e.g., for $\tau=1$pm and $\tau=$2pm).

The point and interval metrics reported in
Tables~\ref{tab:solar_asset-error-unbiased}, \ref{tab:solar_unbiased_asset_envelope} and \ref{tab:solar_asset-ks-unbiased}, demonstrate high sharpness and robust calibration. The intervals $[T_0, T_1)$ run from 10am to 4pm with lead times $k \in \{0, 2, 4, \ldots\}$ and $\Delta=1$~h for solar assets. Forecast intervals become progressively sharper throughout the day, for example FIS for $\cR^{\mathrm{adj}}$ decreases from $0.49$ at update time $\tau=10$am to $0.313$ at $2$pm at $\alpha = 0.1$ (Table~\ref{tab:solar_unbiased_asset_envelope}). The top two envelope constructions are $\cR^{\mathrm{mar}}$ and $\cR^{\mathrm{adj}}$, which also have FCS values that are very close to nominal levels. 
The unadjusted functional envelope $\cR^{\mathrm{fun}}$ is again too wide and tends to over-cover (FCS is above 95\% even with 80\% nominal). The adjusted functional envelope  $\cR^{\mathrm{adj}}$ provides the best (by far) SCP values. We note that simultaneous functional coverage is challenging in the presence of cloud-driven downramps. There are no significant differences among the point forecasts which all have comparable RMSE and near-zero MBE values (Table~\ref{tab:solar_asset-error-unbiased}).


The PIT remains close to uniform over the whole horizon ($D_{KS}\leq0.042$, Table~\ref{tab:solar_asset-ks-unbiased}), but degrades over the final daylight hours due to the hard sunset boundary that collapses the entire predictive distribution. Figure~\ref{fig:solar_update}f illustrates a couple of the PIT histograms with respective $D_{KS} = 0.030$ and $0.034$ over the two lead-time ranges, comparable to the wind experiments.

\subsection{Case Study III: ERCOT Zonal Wind Medium-Term Forecasts}\label{sec:wind_zone}

As a final case study, we extend the probabilistic revision from the  intra-day to the medium-term setting and move from synthetic assets to real market data. We start with hourly-issued, $48$~h-ahead (2DA) wind generation forecasts and their realizations for the five aggregated regions of ERCOT (Panhandle, North, Central, South and West) from Jan 2023 through Dec 2025, published by GridStatus \cite{gridstatus_ercot_wind_forecast}. We use 2023--2024 as our collection of curves, and test over 2025. To make generation comparable across zones, we convert MW to CFs by normalizing by the commissioned wind capacity in that region in the corresponding month. Monthly online capacity is constructed from generator-level records in the U.S. Energy Information Administration's Preliminary Monthly Electric Generator Inventory (Form EIA-860M), which reports capacity, location, resource, operating status and operating date \cite{eia860m}; see Figure~\ref{fig:ercot-monthly-capacity}. Each generator is then assigned to an ERCOT region by mapping its reported location to a county and applying ERCOT's published wind and solar region-to-county mapping \cite{ercot_region_county_mapping}.

\begin{figure*}[htb!]
    \centering
    \begin{minipage}{.42\textwidth} 
        \centering
        \begin{overpic}[width=.925\textwidth, trim = {0cm, .25cm, 0cm, 0cm}, clip]{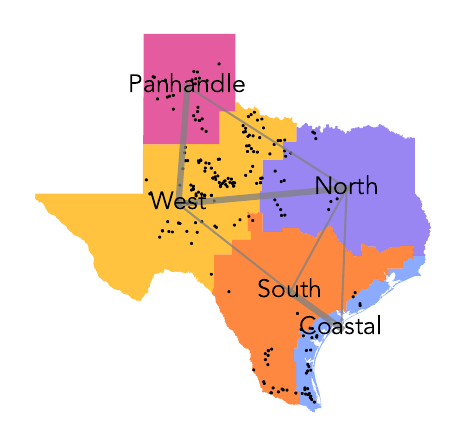}
            \put(0, 88){\textbf{a}}
        \end{overpic}
        \begin{overpic}[width=.45\textwidth, trim = {0cm, 0cm, 0cm, 0cm}, clip]{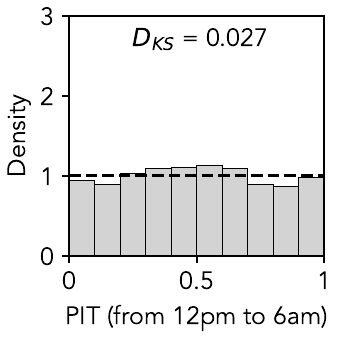}
            \put(-2.5, 101){\textbf{e}}
        \end{overpic}
        \begin{overpic}[width=.45\textwidth, trim = {0cm, 0cm, 0cm, 0cm}, clip]{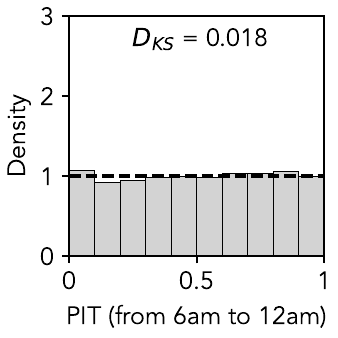}
        \end{overpic}
    \end{minipage}
    \begin{minipage}{.57\textwidth} 
        \centering
        \begin{overpic}[width=1\textwidth, trim = {0cm, 0cm, 0cm, 0cm}, clip]{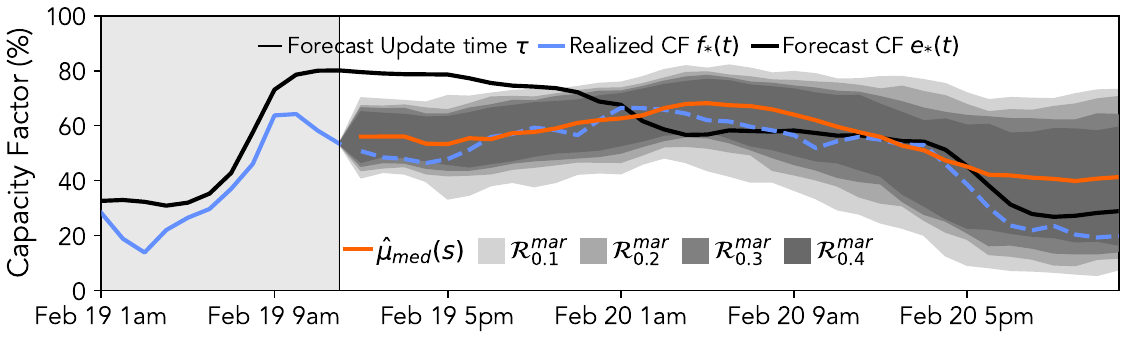}
            \put(0, 30){\textbf{b}}
        \end{overpic}
        \begin{overpic}[width=1\textwidth, trim = {0cm, 0cm, 0cm, 0cm}, clip]{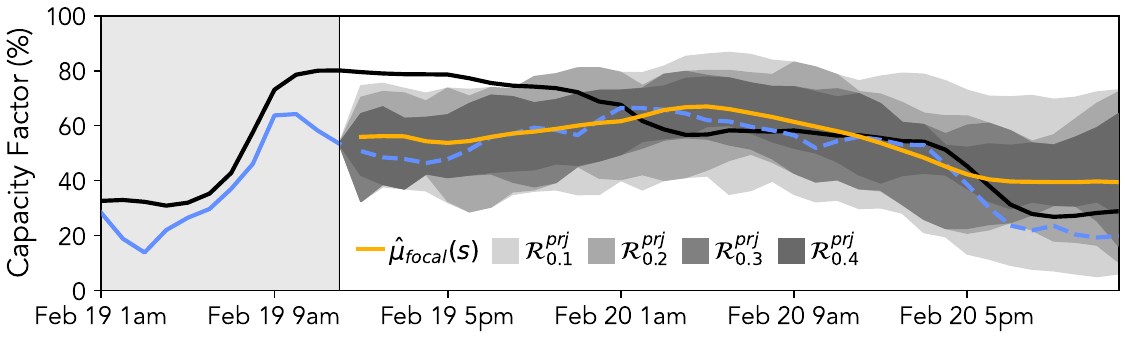}
            \put(0, 30){\textbf{c}}
        \end{overpic}
        \begin{overpic}[width=1\textwidth, trim = {0cm, 0cm, 0cm, 0cm}, clip]{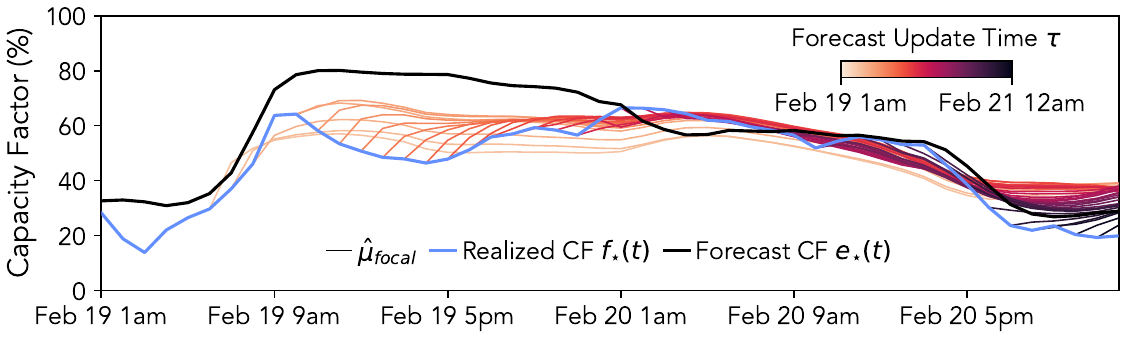}
            \put(0, 30){\textbf{d}}
        \end{overpic}
    \end{minipage}
    \caption{\textbf{Wind 2DA regional forecast probabilistic revision.} ERCOT's North region on $d_\star=$~Feb 19 at forecast update time $\tau = 12$pm with a $\cN^{(1)}$ neighborhood. (a) Wind regions in ERCOT (Panhandle, North, West, South, and Coastal), wind asset locations (black dots) and the graph $G$ defining the neighboring zones (thick edges link first-order neighbors and thin edges represent second-order neighbors). (b) Marginal quantile regions $\cR^{\mathrm{mar}}$ and $\hat{\mu}_{\text{med}}$ in Eq.~\eqref{eq:wquantile}. (c) Focal envelopes $\cR^{\mathrm{prj}}$ and $\hat{\mu}_{\text{focal}}$ in Eq.~\eqref{eq:focal}. In both panels, we show the envelopes for levels $\alpha=0.1, 0.2, 0.3, 0.4$. (d) Rolling probabilistic revision with the focal curve updated every hour as new observations become available, from $\tau = 12$pm to $\tau = 11$pm $+\,2$ days. Panel (e) PIT histograms evaluated over 0-18 lead hours (left), and 18-36 lead hours (right) starting at $\tau = 12$pm.}
    \label{fig:zone_dynamic_wind}
\end{figure*}
 
The previous case studies defined spatial distance from geographical (lat, lon) coordinates. The ERCOT zones admit no such natural distance, and leave far fewer curves to draw neighbors from. Accordingly, we introduce a discrete graph distance, representing the ERCOT zones as a graph $G = \left(V, E\right)$, where nodes represent zones and edges represent spatial proximity between zones. Figure~\ref{fig:zone_dynamic_wind}a visualizes $V$ with direct neighbors indicated by thick edges.  For a given node $i \in V$, the set of its $k$th-order neighbors is defined as $\mathcal{N}^{(k)}(i) = \{j \in V : d(i,j)=k\}$, where $d(i,j)$ denotes the graph distance between nodes $i$ and $j$. The zeroth-order neighborhood $\cN^{(0)}$ is the node $i$ itself, the first-order neighborhood $\cN^{(1)}$ is the set of direct neighbors, and the second-order neighborhood $\cN^{(2)}$ adds more distant nodes (thin edges in Figure~\ref{fig:zone_dynamic_wind}a). Raising $k$ increases the pool size for the nearest neighborhood and reduces homogeneity. As a further distinction to the first two case studies, 
%
the zonal setup has $48$h horizon.


The probabilistic revision shown in Figure~\ref{fig:zone_dynamic_wind} considers the North region on Feb~12 at forecast update time $\tau=$~12pm with a first-order neighborhood $\cN^{(1)}$ which includes 2 (North and West) out of 5 zones. The $\kappa=150$ neighborhood ensemble curves are constructed exactly as in the asset-level experiments, see Figure~\ref{fig:wind_region_neighbors}a--c. Spatially, about 85 neighbors come from the North region itself $\cN^{(0)}$, with the remainder from the West region, see Figure~\ref{fig:wind_region_neighbors}d. Temporally, the neighbors concentrate in the late winter and early spring months around the forecast day, although ensemble members from all 12 months are present (Figure~\ref{fig:wind_region_neighbors}e). Compared to asset-level wind forecasting, zonal data requires a smaller CF similarity weight ($\lambda_f = 10$ against $40$), and same neighborhood ($\kappa = 150$), without which $D_{\mathrm{KS}}$ does not reach the values reported above (Table~\ref{tab:wind_zone-envelope_c1-unbiased}). 
In the zonal experiments, the loss in Eq.~\eqref{eq:loss} is relatively flat around $\lambda_e$, but higher values progressively degrade sharpness. Therefore, we cap $\lambda_e$ at 10. This bound is small enough to identify interior values where the loss is informative ($\lambda_e = 1.87$ and $6.78$). The selected seasonal weights $\lambda_d$ are small, indicating that the time-of-year information is already in the shape of the neighboring 2DA forecasts (Table~\ref{tab:wind_zone-hyperparameters-unbiased}). Since zonal generation almost never hits either zero or maximum CF of 1, there is no need for edge-inflation and we always fuse with the original forecast, taking $p_{\mathrm{fusion}} = 1$. 

\begin{remark}
    \label{rem:reissue}
    Let $e_\star$ denote the 2DA forecast of the operational day $d_\star$ with similarity weights $\omega_i$, and let $e^{(\tau)}_\star$ denote the most recently issued 2DA forecast available at the update time $\tau$. The neighborhood $\cA_\kappa$, the weights $\omega_i$ and $\delta_{\nu,\vartheta}(s)$ computed from $e_\star$, remain fitted as at $\tau$. The fusion with Eq.~\eqref{eq:blended-fcst} is instead constructed by replacing the forecast of each neighbor $e_{k(i)}(s)$ with the reissued forecast $e^{(\tau)}_\star(s)$, so that the neighborhood is unchanged, but the fusing forecast is the current one.
\end{remark}

At $\tau=$~12pm shown in Figure~\ref{fig:zone_dynamic_wind}b-c, the marginal predictive bands $\cR^{\mathrm{mar}}$ and the projected functional envelopes $\cR^{\mathrm{prj}}$ look largely similar, with the latter being a bit tighter at the end of the day. Unlike other case studies, here we can distinguish between the quality of the point forecasts: $\hat{\mu}_{\mathrm{mar}}$ exhibits a persistent positive bias, while $\hat{\mu}_{\mathrm{deep}}$ has very high RMSE. Thus, the focal curve $\hat{\mu}_{\mathrm{focal}}$ is the clear winner, with low RMSE and MAE and with a bias never exceeding $0.004$ in magnitude 
(Table~\ref{tab:wind_zone-error_c1-unbiased}). The projection envelope $\cR^{\mathrm{prj}}_{\alpha}$, which is built around  $\hat{\mu}_{\mathrm{focal}}(s)$, has near-perfect FCS and achieves excellent FIS (Table~\ref{tab:wind_zone-envelope_c1-unbiased}). 
%
In contrast, the functional envelope $\cR^{\mathrm{fun}}_{\alpha}$ again over-covers ($0.93$--$0.97$), and the marginal $\cR^{\mathrm{mar}}$ has very poor simultaneous coverage. The adjusted $\cR^{\mathrm{adj}}$ slightly outperforms  $\cR^{\mathrm{prj}}$ on SCP and on FIS in the morning. 



As the operational day progresses, the proposed method dynamically updates the focal curve at every hour (Figure~\ref{fig:zone_dynamic_wind}d), progressively revising the trajectory over the $48$~h forecasting horizon. The PIT histograms are flat over both the $0$--$18$~h and $18$--$36$~h lead windows, with $D_{KS} = 0.027$ and $0.018$, respectively (Figure~\ref{fig:zone_dynamic_wind}e), indicating that the probabilistic forecasts remain well calibrated at this aggregation level and with an extended forecasting horizon (Table \ref{tab:wind_zone-envelope_c1-unbiased}). 

We also check the impact of picking the zonal neighborhood. Raising the zonal neighborhood order $k$, more importance is placed on DA similarity and especially seasonal similarity (Table \ref{tab:wind_zone-hyperparameters-unbiased}). Because the larger pool is climatically more heterogeneous, the algorithm switches to emphasizing time-of-year. The trust level and trust horizon hyperparameters slightly increase in $k$.


\section{Conclusion}\label{sec:conclude}

In this study we design a data-driven non-parametric approach to generate and update probabilistic forecasts intra-day. To this end, we use nearest-neighbor logic across four quantities: day-ahead (DA) hourly forecast curve, partially realized sub-hourly generation curve, geographic distance across assets and equinox distance between days of the year. Critically, we utilize a functional perspective both to define the similarity metric and to subsequently construct functional envelopes. 

The proposed framework effectively updates DA forecasts into calibrated probabilistic ensemble forecasts of the remainder of the operational day. The calibration procedure, which tunes the fusion (trust level and horizon) and neighborhood selection (lengthscale) hyperparameters via Bayesian optimization, yielded Probability Integral Transform diagnostics close to uniformity, functional coverage scores close to nominal levels, and low interval scores, and was validated across both wind and solar assets, as well as real-life ERCOT data.

The principal benefits of our method arise from its nonparametric and distribution-free nature. Our method makes no parametric assumptions about the generation process, requires no numerical weather model, and leaves the underlying DA model untouched. As a consequence, the method is portable across assets, technologies (i.e.,~wind and solar), and regions, and its low computational cost makes intra-day operational deployment practical. Theoretical support is provided for the proposed neighborhood construction, the nesting of the functional envelopes, and the calibration. Because the ensemble members are built from historical curves, they inherit realistic ramps, temporal persistence, and capacity factor boundary by construction.  Together, these properties make the framework an operationally deployable complement to existing DA forecasting pipelines, at both the asset and aggregate levels.

Our results offer two practical insights into the functional nearest-neighbor construction. First, the conservative three-channel neighborhood selection, which requires candidate neighbors to be simultaneously close in the partial realization, the DA forecast, and the day-of-year space, and is restricted to geographically nearby assets, produces ensembles whose dispersion adapts naturally to the day's difficulty. It is tight when many close neighbors exist and appropriately wide on atypical days. Second, the fusion with the DA forecast ought to be fine-tuned to appropriately capture reversion to the mean, with underlying hyperparameters that strongly depend on update time $\tau$. 

Several tasks are left for future research. First, the specifics of how to define and fuse similarity metrics between two pairs of data objects should be investigated further. For instance, we relied on a simple geodesic lat-lon distance for spatial similarity, which may not capture the differences in climatic conditions, terrain, vegetation, topography, built environments, and moisture that may strongly impact the relevance of using a different asset for forecasting. Alternatives, such as using spatial warping techniques \cite{li2025probabilistic} or explictly clustering by  climate zones  might overcome this limitation.


Next, our implicit premise was to seek an unbiased probabilistic forecast that is assessed symmetrically in terms of predictive quantiles.  However, in power system operations, one may seek a forecast that  minimizes expected operational costs (e.g., curtailment) or reduces operational risks (e.g., loss-of-load) in quasi-real-time decision-making. The forecast and functional envelope that are optimal for decision support may therefore exhibit intentional bias if the underlying cost or risk structure is asymmetric. Hence, taking into account operational risk criteria could be embedded into the construction of similarity metrics and for fine-tuning algorithm parameters. 

\section*{Code Availability}

The Python software, scripts to reproduce the experiments, and visualization Jupyter notebooks developed to generate the figures in this research are publicly available in a GitHub repository \url{https://github.com/gterren/dynamic_update}. The experiment results are available in a Zenodo repository. The assets locations are from the 2030 Texas-7k grid publicly available in a GitHub repository \url{https://github.com/PrincetonUniversity/Vatic}. 

\section*{Acknowledgements}

G.T.S. was partly supported by the California NanoSystems Institute (CNSI) Climate Innovation Postdoctoral Fellowship at the University of California Santa Barbara (UCSB). M.L. was partially supported by the National Science Foundation (NSF) grants DMS-2407550 and DMS-2420988. G.T.S. and M.L. thank Prof.~Deshmukh for providing valuable insight. G.T.S. and M.L. also acknowledge the ARPA-E PERFORM dataset generated by NREL \cite{OEDI_Dataset_5772}. Use was made of the computational facilities purchased with funds from NSF CNS-1725797 and administered by the Center for Scientific Computing (CSC). The CSC is supported by the CNSI and the Materials Research Science and Engineering Center (MRSEC; NSF DMR 2308708) at UCSB. 

\section*{Declaration of generative AI use}

The authors acknowledge the use of the AI tool Claude to assist with coding and data processing tasks, as well as editing and improving the clarity of the manuscript during the preparation of this work. After using this tool, the authors reviewed and edited the content as needed and take full responsibility for the content of the publication.

\section*{Author contributions}

G.T.S. and M.L. conceptualized the study, acquired funding, developed the methodology, curated the data, conducted the formal analyses and investigations, drafted the manuscript, and  reviewed and edited the manuscript. G.T.S. developed the software and worked on the visualization. M.L. supervised the project and handled project administration. 

\section*{Competing interests}

The authors declare no competing interests.

\printbibliography

\appendix

\clearpage

\section{Supporting Figures} 

\begin{figure*}[htb!]
    \centering
    \begin{minipage}{.48\textwidth} 
        \centering
        \begin{overpic}[width=1\textwidth, trim = {0cm, 0cm, 0cm, 0cm}, clip]{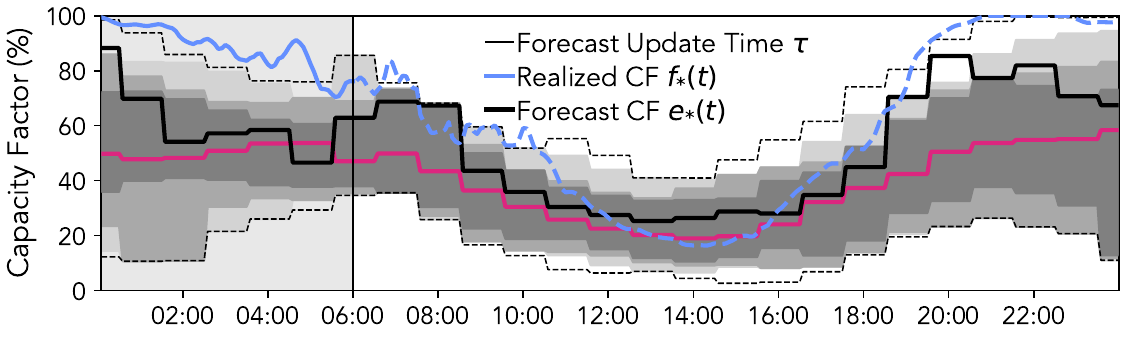}
            \put(0, 28.5){\textbf{a}}
        \end{overpic}
        \begin{overpic}[width=1\textwidth, trim = {0cm, 0cm, 0cm, 0cm}, clip]{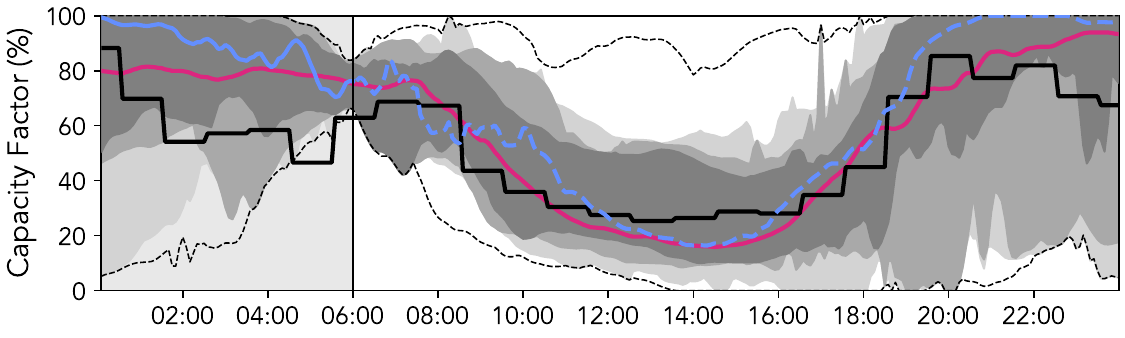}
            \put(0, 28.5){\textbf{b}}
        \end{overpic}
        \begin{overpic}[width=1\textwidth, trim = {0cm, 0cm, 0cm, 0cm}, clip]{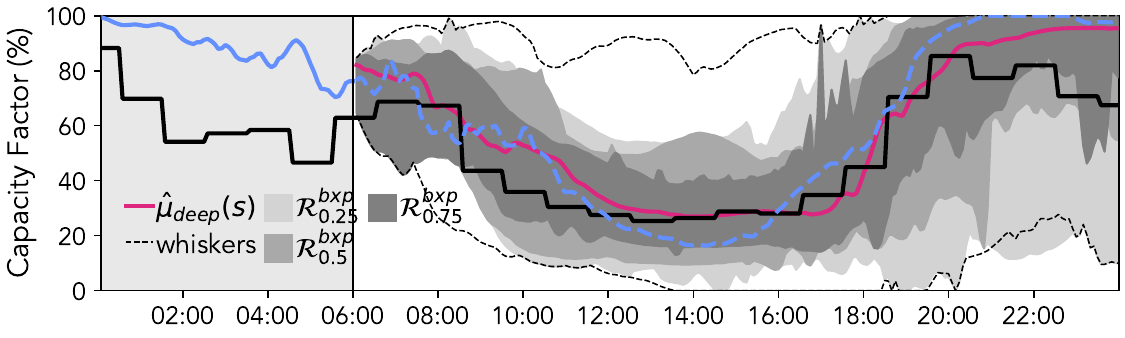}
            \put(0, 28.5){\textbf{c}}
        \end{overpic}
    \end{minipage}
    \begin{minipage}{.48\textwidth} 
        \centering
        \begin{overpic}[width=.725\textwidth, trim = {1cm, 0cm, 3cm, 0cm}, clip]{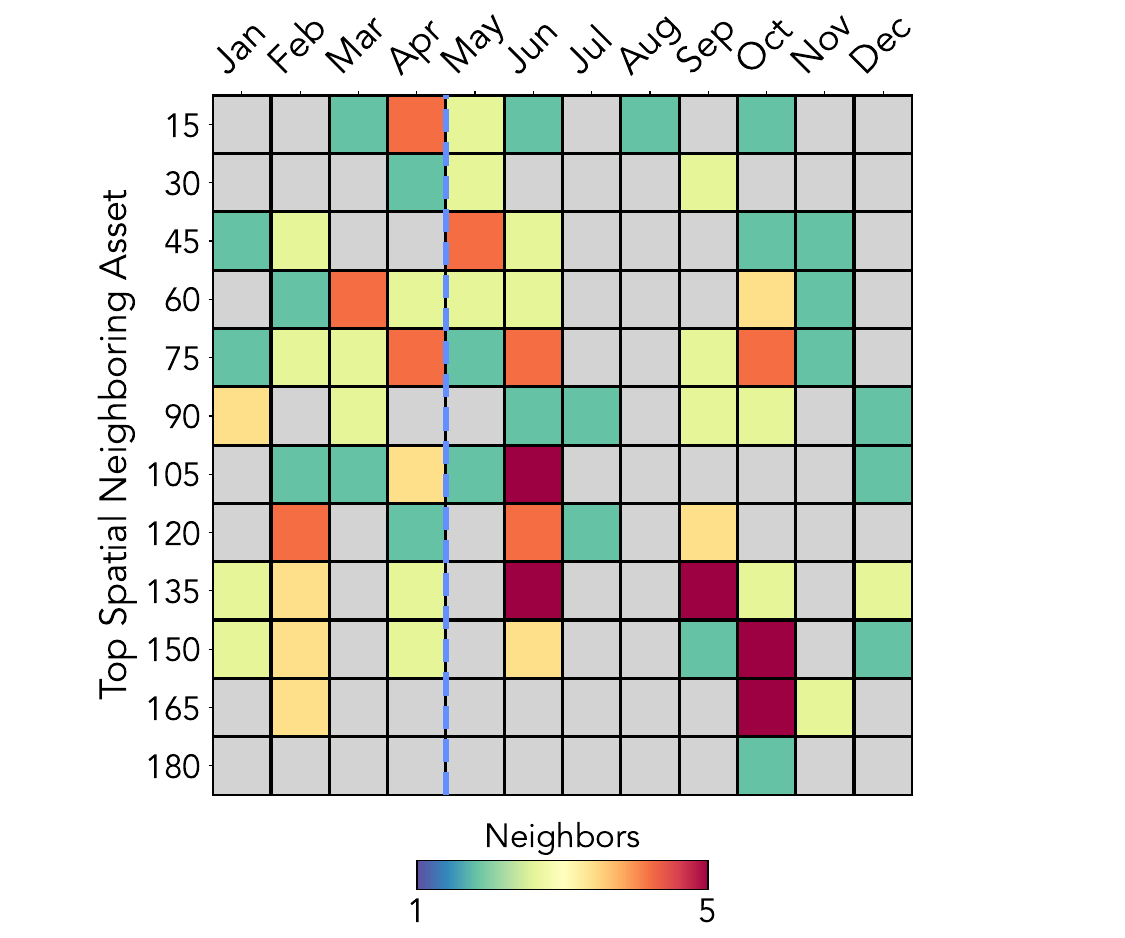}
            \put(2.5, 90){\textbf{d}}
        \end{overpic}
        \begin{overpic}[width=.255\textwidth, trim = {0cm, -3cm, 0cm, 0cm}, clip]{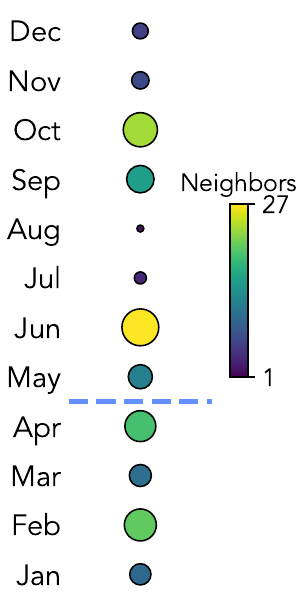}
            \put(0, 105){\textbf{e}}
        \end{overpic}
    \end{minipage}
    \caption{\textbf{Functional nearest neighbors.} Barranca Wind asset on $d_\star$=~May 1 at forecast update time $\tau = 6$am. \emph{Left column, panels (a,b,c):} Day-ahead forecast neighbors $e_1, \ldots, e_\kappa$, realized capacity factor neighbors $f_1 (t), \ldots, f_\kappa (t)$, and fused curves $\hat{\mu}_1(s), \ldots, \hat{\mu}_\kappa (s)$, all ranked by Modified Band Depth; \emph{Panel d:} Spatiotemporal distribution of neighbors; $d_\star$ marked as the blue vertical line; \emph{Panel e:} temporal distribution of neighbors by month; $d_\star$ is the blue horizontal line.}
    \label{fig:neighbors}
\end{figure*}

\begin{figure*}[htb!]
    \centering
    \begin{minipage}{.48\textwidth} 
        \centering
        \begin{overpic}[width=1\textwidth, trim = {0cm, 0cm, 0cm, 0cm}, clip]{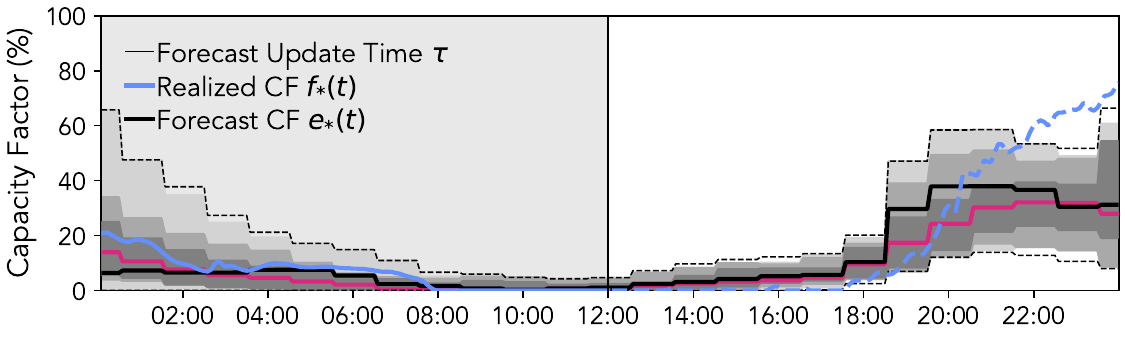}
            \put(0, 28.5){\textbf{a}}
            \put(55,25){{\scriptsize DA forecast neighbors}}
        \end{overpic}
        \begin{overpic}[width=1\textwidth, trim = {0cm, 0cm, 0cm, 0cm}, clip]{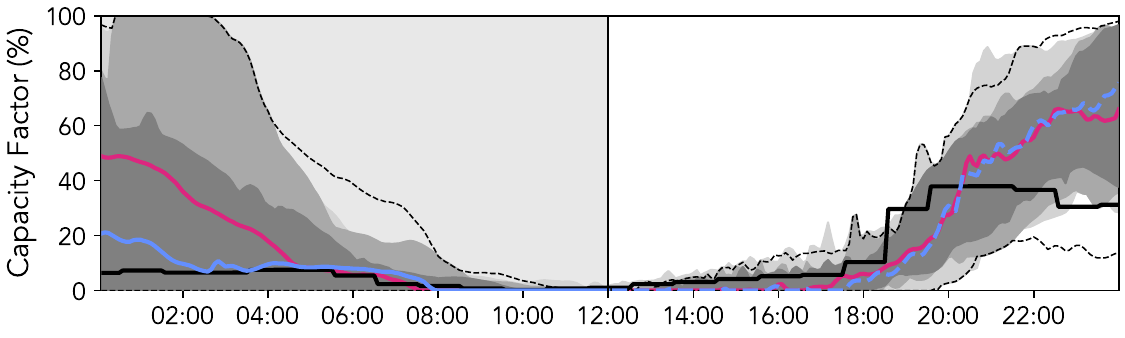}
            \put(0, 28.5){\textbf{c}}
            \put(55,25){{\scriptsize Partial CF neighbors}}
        \end{overpic}
        \begin{overpic}[width=1\textwidth, trim = {0cm, 0cm, 0cm, 0cm}, clip]{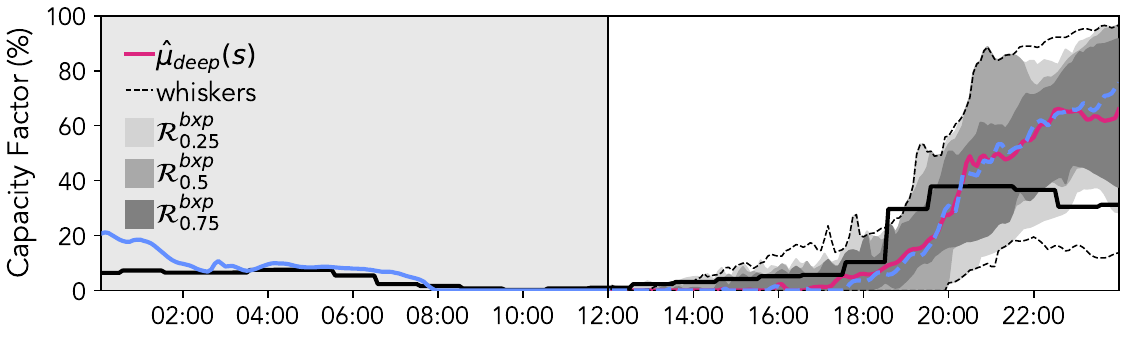}
            \put(0, 28.5){\textbf{e}}
            \put(55,25){{\scriptsize Fused neighbors}}
        \end{overpic}
    \end{minipage}
    \begin{minipage}{.48\textwidth} 
        \centering
        \begin{overpic}[width=1\textwidth, trim = {0cm, 0cm, 0cm, 0cm}, clip]{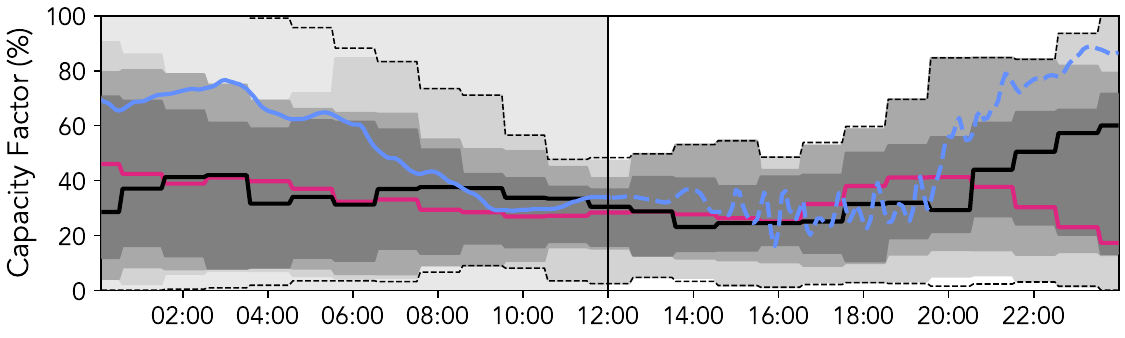}
            \put(0, 28.5){\textbf{b}}
        \end{overpic}
        \begin{overpic}[width=1\textwidth, trim = {0cm, 0cm, 0cm, 0cm}, clip]{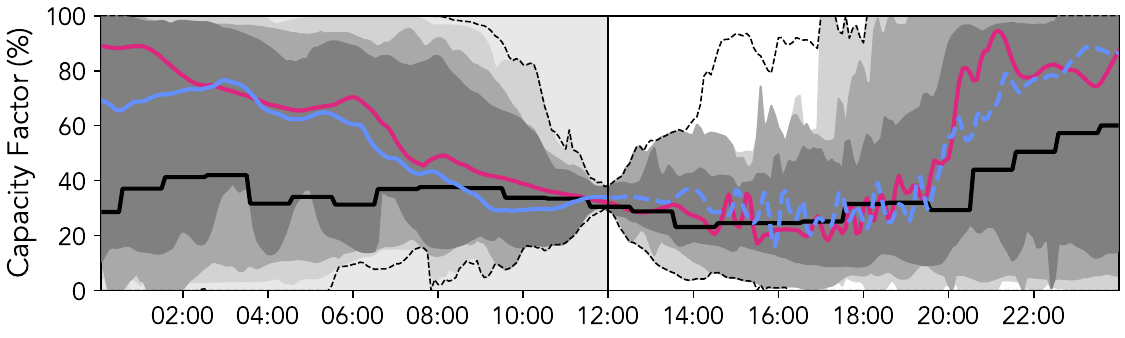}
            \put(0, 28.5){\textbf{d}}
        \end{overpic}
        \begin{overpic}[width=1\textwidth, trim = {0cm, 0cm, 0cm, 0cm}, clip]{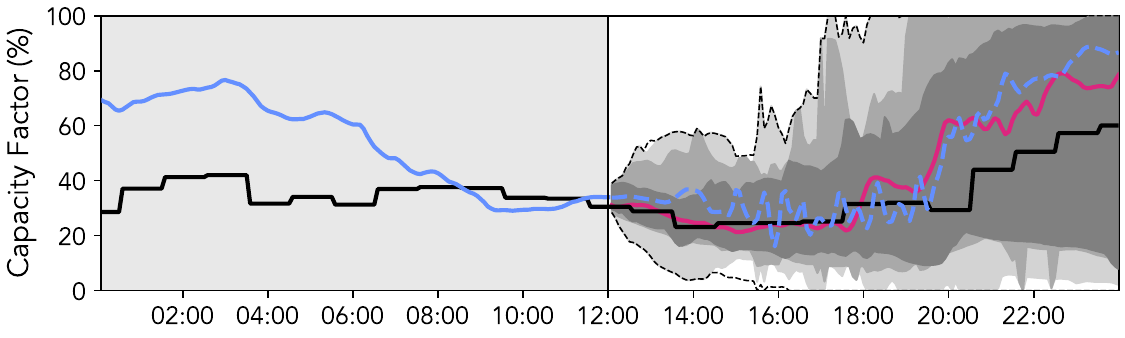}
            \put(0, 28.5){\textbf{f}}
        \end{overpic}
    \end{minipage}
    \caption{\textbf{Functional nearest neighbors.} Blanco Canyon asset at forecast update time $\tau = 12$pm. \emph{Left column (a, c, e)}: $d_\star$=~Apr 23; \emph{Right column (b, d, f):} $d_\star$=~Jun 26. \emph{Top Row:} DA forecast neighbors $e_1, \ldots, e_\kappa$, visualized as enhanced functional boxplot based on MBD (panels a, b); \emph{Middle row:}  realized partial CF curve neighbors $f_1, \ldots, f_\kappa$ (panels c, d); \emph{Bottom row:} fused curves $\hat{\mu}_1, \ldots, \hat{\mu}_\kappa$ (panels e, f).}
    \label{fig:functional_neighbors}
\end{figure*}

\begin{figure}[htb!]
    \begin{minipage}{.45\textwidth} 
        \centering
        \begin{overpic}[width=1\textwidth, trim = {0cm, 0cm, 0cm, 0cm}, clip]{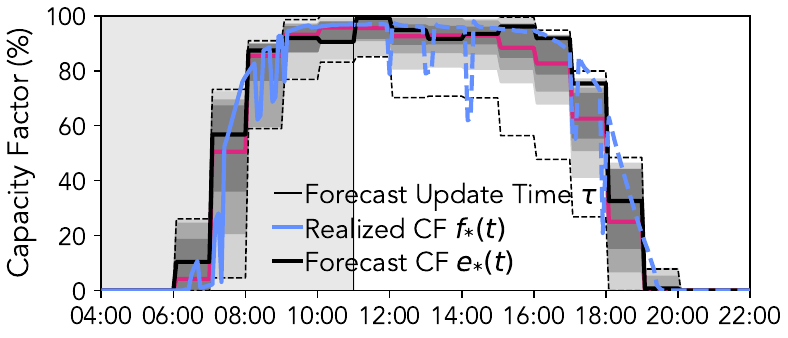}
            \put(0, 42){\textbf{a}}
        \end{overpic}
        \begin{overpic}[width=1\textwidth, trim = {0cm, 0cm, 0cm, 0cm}, clip]{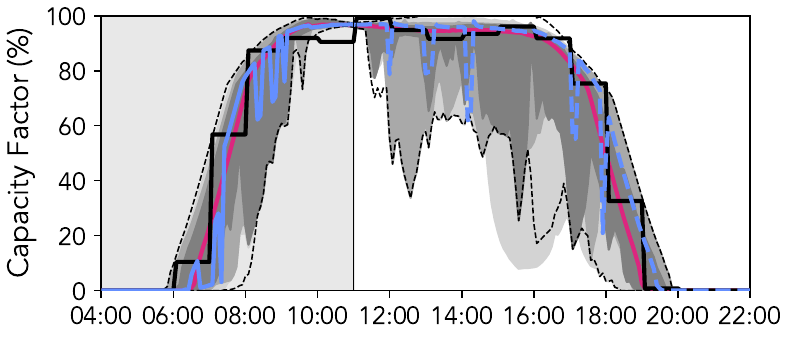}
            \put(0, 42){\textbf{b}}
        \end{overpic}
        \begin{overpic}[width=1\textwidth, trim = {0cm, 0cm, 0cm, 0cm}, clip]{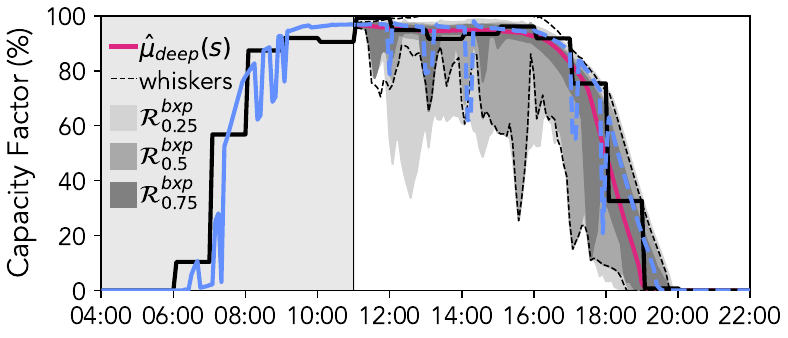}
            \put(0, 42){\textbf{c}}
        \end{overpic}
    \end{minipage}
    \begin{minipage}{.54\textwidth} 
        \centering
        \begin{overpic}[width=.71\textwidth, trim = {1cm, 0cm, 3cm, 0cm}, clip]{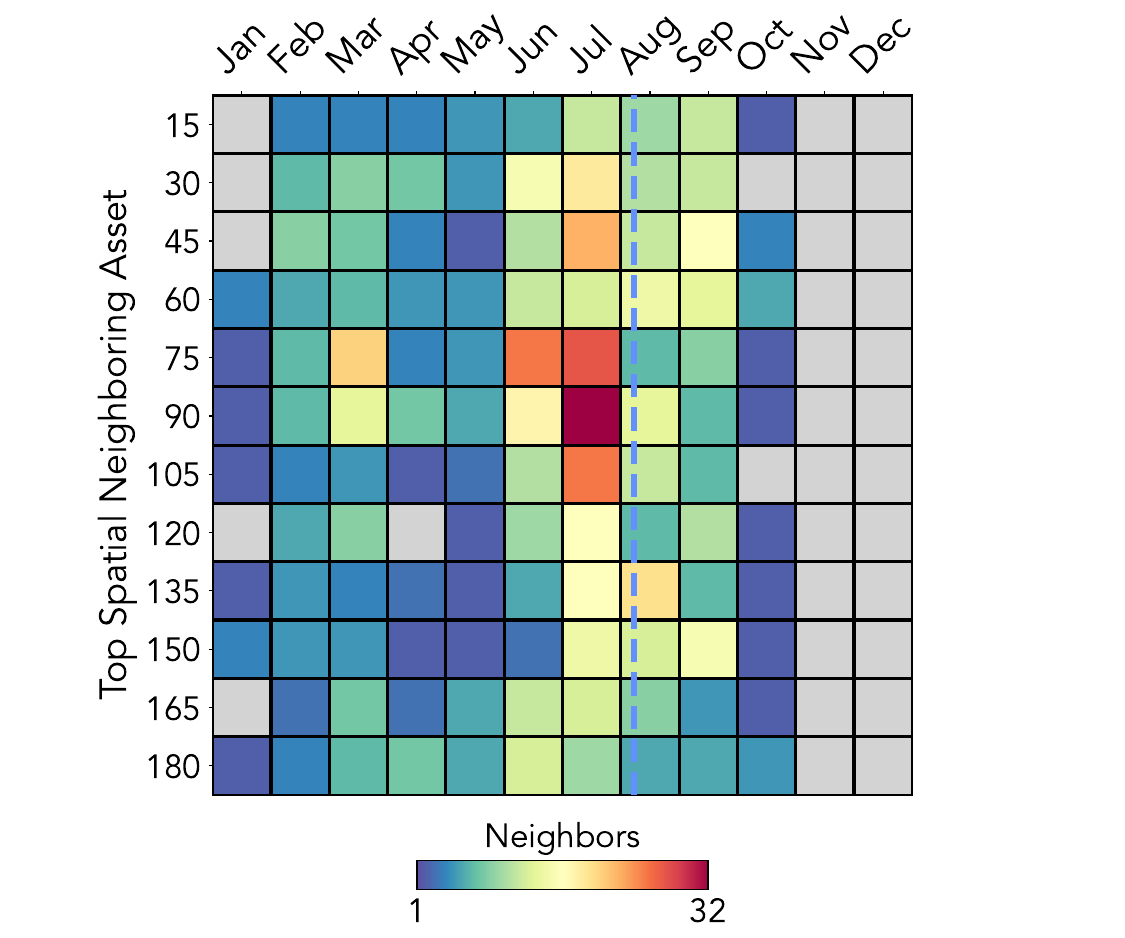}
            \put(2.5, 90){\textbf{d}}
        \end{overpic}
        \begin{overpic}[width=.275\textwidth, trim = {0cm, -2.25cm, 0cm, 0cm}, clip]{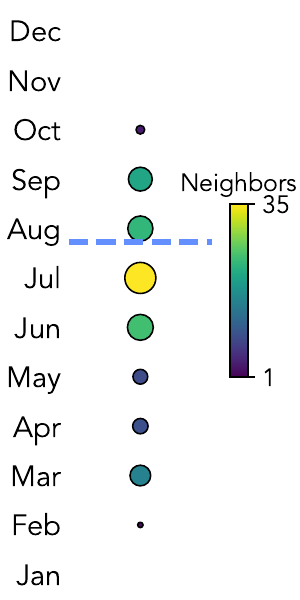}
            \put(0, 101){\textbf{e}}
        \end{overpic}
    \end{minipage}
    \caption{\textbf{Functional nearest neighbors for a solar asset.} Blue Bell Solar on Aug 8 at forecast update time $\tau = 11$am. \emph{Left column:} Blue curve is the partial realized $f_\star$ and solid black is the original DA forecast $e_\star$. \emph{Panel a:} Enhanced functional boxplot of the DA-forecast neighbors $e_1, \ldots, e_\kappa$. \emph{Panel b:} Enhanced functional boxplot of the CF neighbors $f_1, \ldots, f_\kappa$. \emph{Panel c:} Enhanced functional boxplot of the fused neighbors $\hat{\mu}_1, \ldots, \hat{\mu}_\kappa$. \emph{Panel d:} Spatiotemporal distribution of the $\kappa_0$ neighbors across months of the years and Haversine distances; $d_\star$ is marked as the blue vertical line. \emph{Panel e:} Distribution of neighbors by month $d_\star$ is the blue horizontal line.}
    \label{fig:solar_asset_neighbors}
\end{figure}


\begin{figure}[htb!]
    \begin{minipage}{.53\textwidth} 
        \centering
        \begin{overpic}[width=1\textwidth, trim = {0cm, 0cm, 0cm, 0cm}, clip]{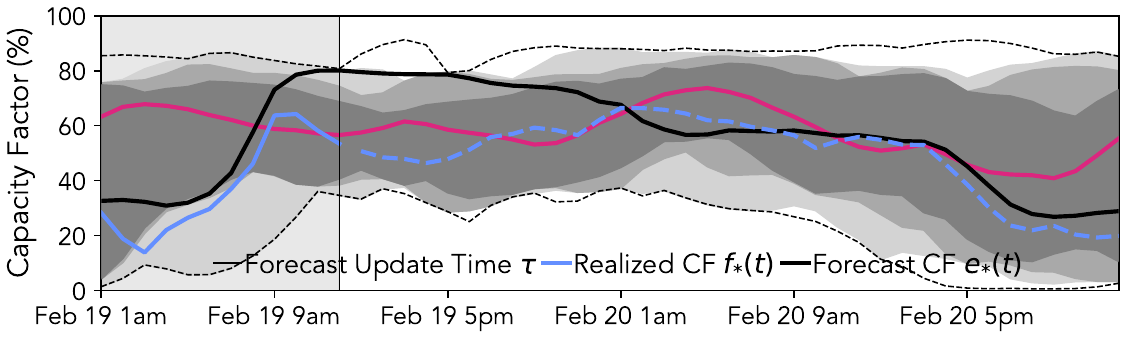}
            \put(0, 28.5){\textbf{a}}
        \end{overpic}
        \begin{overpic}[width=1\textwidth, trim = {0cm, 0cm, 0cm, 0cm}, clip]{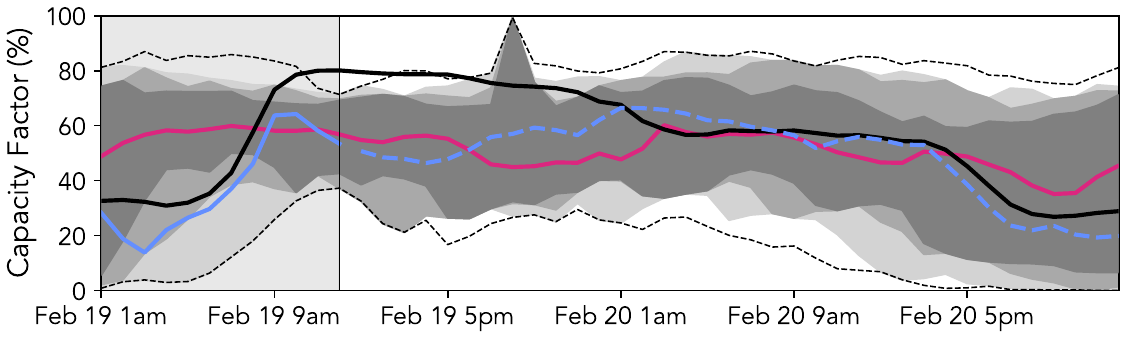}
            \put(0, 28.5){\textbf{b}}
        \end{overpic}
        \begin{overpic}[width=1\textwidth, trim = {0cm, 0cm, 0cm, 0cm}, clip]{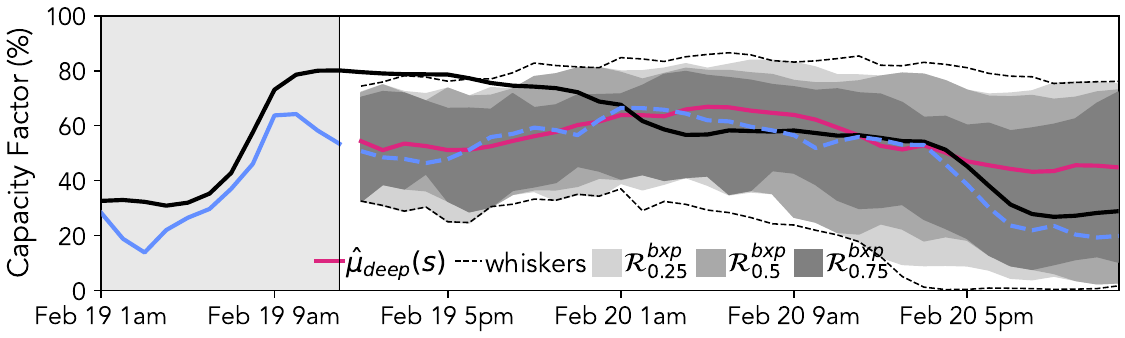}
            \put(0, 28.5){\textbf{c}}
        \end{overpic}
    \end{minipage}
    \begin{minipage}{.45\textwidth} 
        \centering
        \begin{overpic}[width=.75\textwidth, trim = {0cm, -.25cm, 0cm, 0cm}, clip]{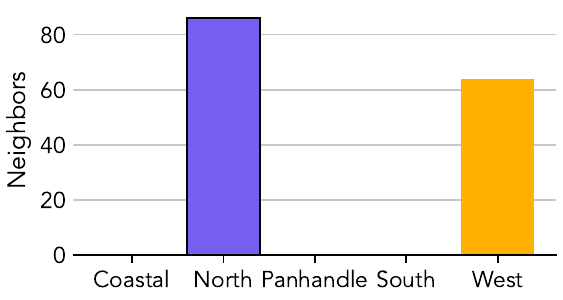}
            \put(0, 56){\textbf{d}}
        \end{overpic}\\
        \begin{overpic}[width=.3\textwidth, trim = {0cm, 0cm, 0cm, 0cm}, clip]{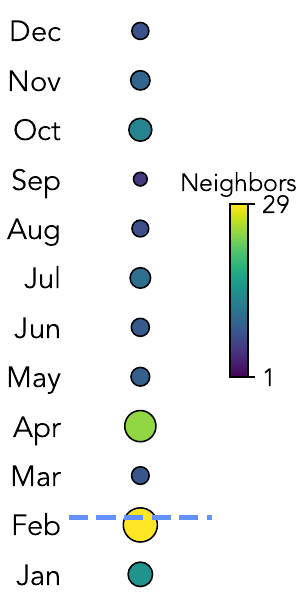}
            \put(-10, 95){\textbf{e}}
        \end{overpic}
    \end{minipage}
    \caption{\textbf{Functional nearest neighbors for ERCOT regional wind 2 Day-Ahead (DA) forecast updating.} North ERCOT region on $d_\star=$~Feb 19 at forecast update time $\tau = 12$pm with a $\cN^{(1)}$ neighborhood. \emph{Left column:} Blue curve is the realized $f_\star$ and solid black is the original 2DA forecast $e_\star$. \emph{Panel a:} Enhanced functional boxplot for the 2DA-forecast neighbors $e_1, \ldots, e_\kappa$. \emph{Panel b:} Realized capacity factor neighbors $f_1, \ldots, f_\kappa$. \emph{Panel c:} Fused curves $\hat{\mu}_1, \ldots, \hat{\mu}_\kappa$  ranked by Modified Band Depth. \emph{Panel d:} Distribution of neighboring curves  selected from ERCOT’s North ($k=0$) and West ($k=1$) regions, black edges mark the target region (North). \emph{Panel e:} Temporal distribution of nearest neighbors by month, the blue line indicates the day $d_\star$.}
    \label{fig:wind_region_neighbors}
\end{figure}

\clearpage
\addtocontents{toc}{\protect\setcounter{tocdepth}{2}}   

\setcounter{section}{0}
\setcounter{figure}{0}
\setcounter{table}{0}
\setcounter{equation}{0}
\setcounter{lemma}{0}          

\renewcommand{\thesection}{S\arabic{section}}
\renewcommand{\thesubsection}{S\arabic{section}.\arabic{subsection}}
\renewcommand{\thefigure}{S\arabic{figure}}
\renewcommand{\thetable}{SI\arabic{table}}
\renewcommand{\theequation}{S\arabic{equation}}

\renewcommand{\thepage}{SI\arabic{page}}
\setcounter{page}{1}

\begin{center}
    {\Large \textbf{Supplementary Information}}\\[1.0em]
    {\large Nearest Neighbor Functional Curve Completion for Intraday Probabilistic Revisions of Renewable Generation Forecasts}\\[0.8em]
    {Guillermo Terr\'en-Serrano and Michael Ludkovski}
\end{center}

\vspace{1.5em}
\renewcommand{\contentsname}{Supplementary Contents}
\tableofcontents          
\vspace{1.5em}

\clearpage

\section{Supplementary Note: Theoretical properties of the functional neighborhood}
\label{sec:theory}

\begin{proposition}[Scale invariance and identifiability]
\label{prop:invariance}
    Fix $\lambda = (\lambda_f, \lambda_e, \lambda_d) \in (0,\infty)^3$ and suppose ties in $\{r_n\}$ are broken by a fixed rule independent of $\lambda$. Then for every $c > 0$, the rescaled parameter $c\lambda$ yields:
    \begin{enumerate}
        \item[(i)] the same Stage~I neighborhood $\mathcal{N}_{\kappa_0}$, the same set $\mathcal{A}_\kappa$, and the same collection of fused forecasts $\{\hat\mu_i\}$;
        \item[(ii)] focal curve weights of the softmax form
        \begin{equation}
            \label{eq:softmax}
            \omega'_i(c\lambda) \;=\;
            \frac{\exp(-\tilde{r}_i/\zeta)}
               {\sum_{j \in \mathcal{A}_\kappa} \exp(-\tilde{r}_j/\zeta)},
            \qquad
            \tilde{r}_i = \max\{d_{f,i},\ \rho_e\, d_{e,i},\ \rho_d\, d_{d,i}\},
        \end{equation}
        with ratios $\rho_e = \lambda_e/\lambda_f$, $\rho_d = \lambda_d/\lambda_f$ and temperature $\zeta := (c\lambda_f)^{-1}$.
        \item[(iii)] The parameterization by $(\lambda_f, \lambda_e, \lambda_d)$ is not identifiable: the predictive distribution is constant along each ray $\{c\lambda : c > 0\}$ up to the value of $\zeta=(c\lambda_f)^{-1}$.
    \end{enumerate}
\end{proposition}
\begin{proof}
    For (i), $r_n(c\lambda) = \max_c\{c\lambda_c d_{c,n}\} = c\, r_n(\lambda)$ by positive homogeneity of the maximum, and multiplication by $c > 0$ preserves the ordering of $\{r_n\}$; hence the $\kappa_0$ nearest neighbors are unchanged, and Stage~II and the fusion step do not involve $\lambda$. For (ii), write $r_i(c\lambda) = c\lambda_f\, \tilde{r}_i$ with $\tilde{r}_i$ as in Eq.~\eqref{eq:softmax}; substituting $\omega_i = \exp(-\tilde{r}_i/\zeta)$ into the normalization $\omega'_i = \omega_i/\sum_j \omega_j$ gives the stated softmax form.

    Stage~II depends only on $\mathcal{N}_{\kappa_0}$ and on Haversine distances, neither of which involves $\lambda$. Similarly, the fusion step of Eqs.~\eqref{eq:fusion-rejection}--\eqref{eq:blended-fcst} involves only parameters disjoint from $\lambda$.
\end{proof}

Consequently, the predictive distribution depends on $(\lambda_f, \lambda_e, \lambda_d)$ only through the ratios $(\rho_e, \rho_d)$, which govern neighbor \emph{selection}, and the temperature $\zeta > 0$, which governs neighbor \emph{weighting}. The reparameterization $(\rho_e, \rho_d, \zeta) \in (0,\infty)^3$ removes the non-identifiability and separates the two roles of the length-scales: which channel dominates selection, and how peaked the resulting weights are.

This is the rationale for the parameterization adopted in the hyperparameter tuning of Section~\ref{sec:assess}: the length-scale $\lambda_f$ (equivalently, the temperature $\zeta = 1/\lambda_f$) is fixed a priori, and the Bayesian optimization searches over the identifiable ratios $(\rho_e, \rho_d)$. Searching the raw triple $(\lambda_f, \lambda_e, \lambda_d)$ instead would introduce a redundant direction along which the aggregated $D_{\mathrm{KS}}$ objective is exactly flat up to the value of $\zeta$. The next lemma summarizes the impact of $\zeta$.



\begin{lemma}
\label{lemma:temperature}
    Let $\tilde{r}_{(1)} < \tilde{r}_{(2)} \le \cdots$ order the dissimilarities on $\mathcal{A}_\kappa$, assume $\tilde{r}_{(1)}$ is unique, and define the effective ensemble size $\ess(\zeta) = \bigl(\sum_{i \in \mathcal{A}_\kappa} \omega'_i(\zeta)^2\bigr)^{-1}$, consistent with the effective sample size appearing in the bandwidth rule of Eq.~\eqref{eq:wbandwith}, since $\bigl(\sum_i \omega'^2_i\bigr)^{-1} = \bigl(\sum_i \omega_i\bigr)^2 / \sum_i \omega_i^2$. Then:
    \begin{enumerate}
        \item[(i)] as $\zeta \to 0^+$, $\omega'$ converges to a point mass on the nearest neighbor,  $\hat\mu_{\text{focal}}$ converges to the single best neighbor (a $1$-nearest-neighbor rule) and $\ess(\zeta) \downarrow 1$;
        \item[(ii)] as $\zeta \to \infty$, $\omega'$ converges to the uniform weight $1/\kappa$, $\hat\mu_{\text{focal}}$ converges to the unweighted ensemble mean and $\ess(\zeta) \uparrow \kappa$.
    \end{enumerate}
\end{lemma}

\begin{proof}[Proof of Lemma \ref{lemma:temperature}]
    The limits of $\omega'$ and $\hat\mu_{\text{focal}}$ follow from Eq.~\eqref{eq:softmax}. For ESS, set $\gamma = 1/\zeta$ and $p_i(\gamma) \propto e^{-\gamma \tilde{r}_i}$ 
    Then
    \[
      \frac{d}{d\gamma} \log \sum_i p_i(\gamma)^2
      \;=\; 2\bigl(\E_p[\tilde{r}] - \E_q[\tilde{r}]\bigr),
    \]
    where $q_i \propto e^{-2\gamma\tilde{r}_i}$ is the exponentially tilted distribution. Since tilting by $e^{-\gamma\tilde{r}}$ shifts mass toward smaller $\tilde{r}$, $\E_q[\tilde{r}] \le \E_p[\tilde{r}]$, so $\sum_i p_i^2$ is nondecreasing in $\gamma$, hence nonincreasing in $\zeta = 1/\gamma$, and $\ess = (\sum_i p_i^2)^{-1}$ is nondecreasing in $\zeta$. 
\end{proof}

\bigskip


\begin{proof}[Proof of Lemma \ref{lem:LIE}]
    $u_{\vartheta,\eta}(s)$ is affine and strictly decreasing in $s$ ($du/ds = -\eta/\vartheta < 0$). ELU is nondecreasing (derivative $1$ on $\{u \ge 0\}$, $e^u > 0$ on $\{u < 0\}$), so $\mathrm{ELU}(u_{\vartheta,\eta}(s))$ is nonincreasing in $s$, hence $\delta_{\nu,\vartheta,\eta}(s) = \nu(\eta - \mathrm{ELU}(u_{\vartheta,\eta}(s)))/(\eta+1)$ is nondecreasing in $s$. At $s = \tau$, $u = \eta \ge 0$ so $\mathrm{ELU}(\eta) = \eta$, giving $\delta(\tau) = 0$. As $s \to \infty$, $u \to -\infty$ so $\mathrm{ELU}(u) \to e^u - 1 \to -1$, giving $\delta \to \nu(\eta+1)/(\eta+1) = \nu$. At $s = \tau + \vartheta$, $u = 0 = \mathrm{ELU}(0)$, giving the stated transition value. Monotonicity plus the two limits give $\delta \in [0,\nu]$ throughout.
\end{proof}

\clearpage



\section{Supplementary Note: Baselines}
\label{sec:baselines}

We implemented two comparator forecasting models to provide a baseline for the performance of the proposed functional probabilistic revision.


\paragraph{ARIMA.}

The first baseline is a dynamic harmonic regression based on an ARIMA time series model for each asset, which completes the observed CF $f_\star(s)$ without ever seeing the day-ahead forecast $e_\star(s)$ \cite{hyndman2014forecasting}.

A seasonal ARIMA with period $T = 288$ is impractical, so we use Fourier harmonic terms to model the daily cyclic mean trend and then ARIMA for the residuals,
\begin{equation}
  f_\star(t) = \sum_{r=1}^{R} \left[ a_r \sin \frac{2 \pi r t}{T}
      + \tilde{a}_r \cos \frac{2 \pi r t}{T} \right] + n(t),
  \qquad
  n(t) \sim \mathrm{ARIMA}(p,d,q).
\end{equation}
We estimate the parameters, including the innovation variance $\nu^2$, jointly by maximum likelihood on the last $60$ days of the training year (2017). The differencing order $d$ is selected by unit-root testing, and $(p,q)$ are selected by the Akaike Information Criterion with the small-sample correction (AICc) \cite{hyndman2008automatic}. We select the number of harmonics $R$ using the weighted interval score with rolling-origin cross-validation on the training year, and we fix the observed history length to $H = 3 \times T$ (3 days prior to the forecast updated time). We apply the resulting time series model in the testing year (2018) without refitting.

Given $\tau$, the fitted model is run over the observed history up to $\tau$ by Kalman filtering \cite{durbin2012time}, so that the partial curve $f_\star(s), s \in \cI_{<\tau}$ is embedded in the state at the forecast update time $\tau$. We then simulate forward to end of the day (midnight) via the ARIMA recursion, with the innovation shocks $\varepsilon(s) \sim \mathcal{N}(0, \nu^2)$ generating probabilistic trajectories. The simulations are clipped to $[0,1]$ and the marginal quantiles and envelopes are based on an ensemble $M_k(s)$ of $K=1000$ trajectories.

\paragraph{Bootstrapped Re-anchored DA forecast (BRDA).} 

The second baseline is the probabilistic upgrade of the day-ahead forecast. Specifically, we re-anchor the DA forecast on the last observed CF and then bootstrap with historical residuals,  which is a standard model previously used in wind power forecasts \cite{madsen2005standardizing}. 
We use the exponential weight function $w (s) = \exp \left( \beta (\tau - s) \right)$ to blend persistence of intra-day forecast deviation ($f_\star(\tau)-e_\star(\tau)$) and the original DA forecast $e_\star(s)$ according to 
\begin{equation}
  \hat{e}_\star (s)
    = w (s) \left(f_\star (\tau) - e_\star (\tau) \right)
      + e_\star (s),
\end{equation}
The decay rate $\beta = 0.231/12$ (1~h), with the lead $s - \tau$, is fixed a priori.

To obtain a predictive distribution, we use $\hat{e}_\star(s)$ as the point forecast and fuse it with historical forecast-actuals curves:
\begin{equation}
  M_k(s) = 0 \vee (\hat{e}_\star (s) + d_{i_k}(s)) \wedge 1, \qquad s \in \cI_{> \tau}
\end{equation}
where $d_i (s) = f_i (s) - \hat{e}_i(s)$ is the realized deviation of curve $i$. We build an ensemble of $K=1000$ such bootstrapped (sampled with replacement) and clipped curves to obtain BRDA marginal quantiles and functional envelopes. 

\clearpage


\section{Supplementary Note: Enhanced functional boxplot}
\label{sec:boxplot}

The functional boxplot \cite{sun2011functional} extends the classical boxplot to functional data, with each element a curve rather than a scalar. The \emph{enhanced} variant displays the deepest curve $\hat{\mu}_{\mathrm{deep}}(s)$ of Eq.~\eqref{eq:central-curve} as the functional median, together with the $25\%$, $50\%$ and $75\%$ central regions and the \emph{fences}. The three central regions are functional envelopes of the form~\eqref{eq:functional-envelope}, evaluated at the fixed levels,
\begin{equation}
    \label{eq:boxplot-regions}
    \mathcal{R}^{\mathrm{bxp}}_{\alpha} \left( s \right) :=
    \left[ \min_{\hat{\mu}_i \in \mathcal{H}_{\alpha}} \hat{\mu}_i \left( s \right), \
           \max_{\hat{\mu}_i \in \mathcal{H}_{\alpha}} \hat{\mu}_i \left( s \right) \right], \quad \alpha \in \{0.75,\, 0.5,\, 0.25\},
    \;
\end{equation}
where $\mathcal{H}^{\mathrm{bxp}}_{\alpha}$ collects the $\lceil (1 - \alpha) \cdot \, \kappa \rceil$ deepest curves. The regions in functional boxplot  describe the depth ordering of the fused curves at conventional proportions and carry no coverage claim. The fences are obtained by inflating that region by a  $1.5$ factor, so that $\mathcal{J}(s) := \bigl[\, l(s) - 1.5 w(s), \; u(s) + 1.5 w(s)\bigr]$, where $[\,l(s), u(s)\,] := \mathcal{R}^{\mathrm{bxp}}_{0.5}(s)$ for the $50\%$ central region and $w(s) := u(s) - l(s)$ for its pointwise width.

\clearpage


\section{Supplementary Note: De-biasing 48h-ahead forecasts}
\label{sec:48h-ahead_debias}

Because ERCOT 48h-ahead forecasts exhibit substantial systematic bias that varies year-to-year, it is essential to de-bias them correctly. To this end, we fit separate linear regression models for each region and each calendar year. Let $f_{p,d,h,r} (t)$ denote the observed generation, and $e_{p,d,h,r} (t)$ the 2DA forecast curves for region $r\in \{1,\ldots,5\}$, period $p\in \{2023, \ 2024, \ 2025\}$, day-of-year $d$, and hour $t\in \{ 1, \ldots, 24 \}$. 
\begin{equation}
    f_{p, d, h, r} (t) = \alpha_{p, r} (t) + \beta_{p, r} (t) \cdot e_{p, d, h, r} (t) + \epsilon_{p, d, h, r} (t).
\end{equation}
We include the forecast hour index $h \in \{1, \ldots, 48\}$ to capture the systematic bias in the forecasted hour $h$. The forecast index $h$ represents the 48 forecasts for hour $t$ in day $d$. Notice that each 2DA forecasts provide 2 forecasts for each hour $t$. The bias-corrected forecast is then obtained as
\begin{equation}
    \hat{e}_{p, d, h, r} (t) = \hat{\alpha}_{p, r} (t) + \hat{\beta}_{p, r} (t) \cdot e_{p, d, h, r} (t).
\end{equation}

\clearpage


\section{Supplementary Figures}

\begin{figure}[htb!]
    \centering
    \begin{overpic}[width=.27\textwidth, trim = {0cm, 0cm, 0cm, 0cm}, clip]{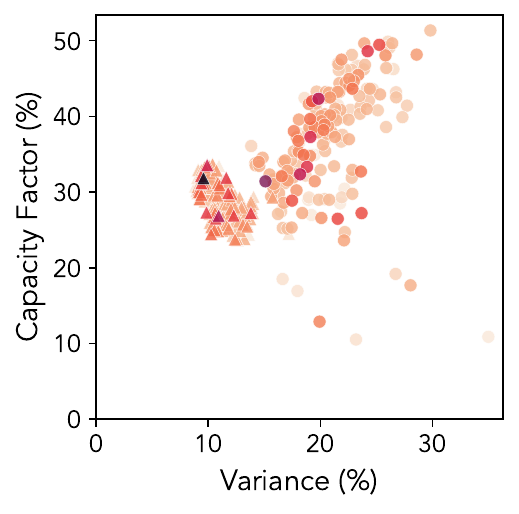}
      \put(0, 95){\textbf{a}}
    \end{overpic}
    \begin{overpic}[width=.29\textwidth, trim = {0cm, 0cm, 0cm, 0cm}, clip]{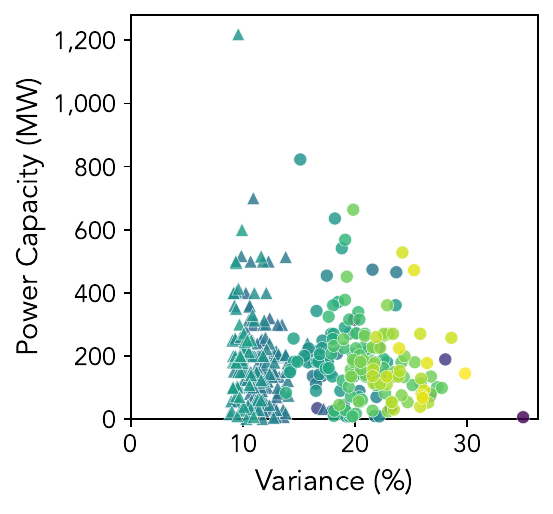}
      \put(0, 89){\textbf{b}}
    \end{overpic}
    \begin{overpic}[width=.27\textwidth, trim = {0cm, 0cm, 0cm, 0cm}, clip]{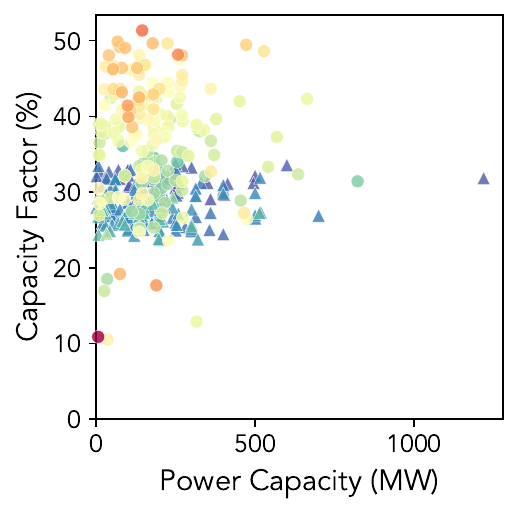}
      \put(0, 95){\textbf{c}}
    \end{overpic}
    \begin{minipage}{.45\textwidth} 
        \centering
        \begin{overpic}[width=.925\textwidth, trim = {0cm, 0cm, 0cm, 0cm}, clip]{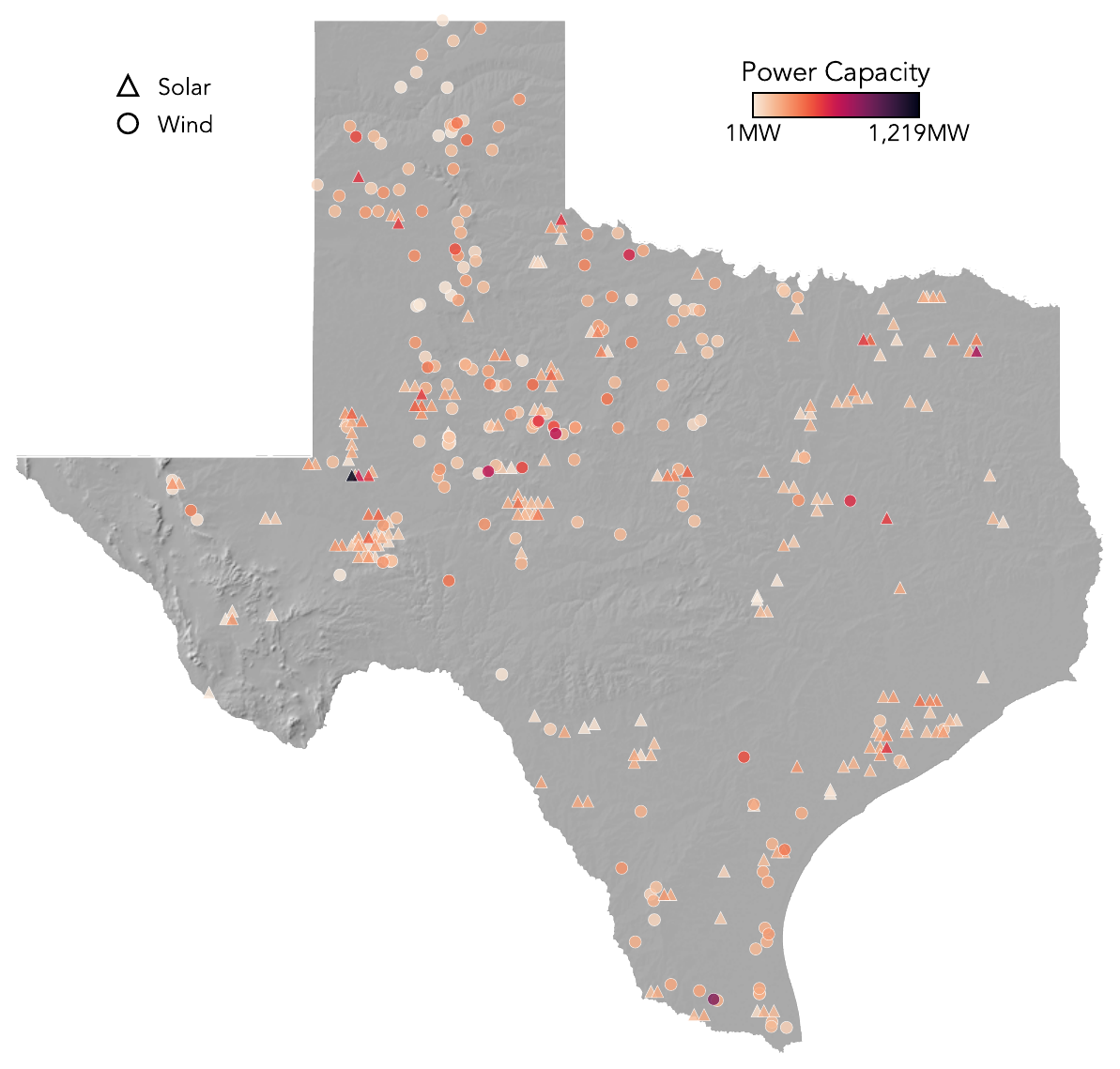}
            \put(0, 90){\textbf{d}}
        \end{overpic}
        \begin{overpic}[width=.925\textwidth, trim = {0cm, 0cm, 0cm, 0cm}, clip]{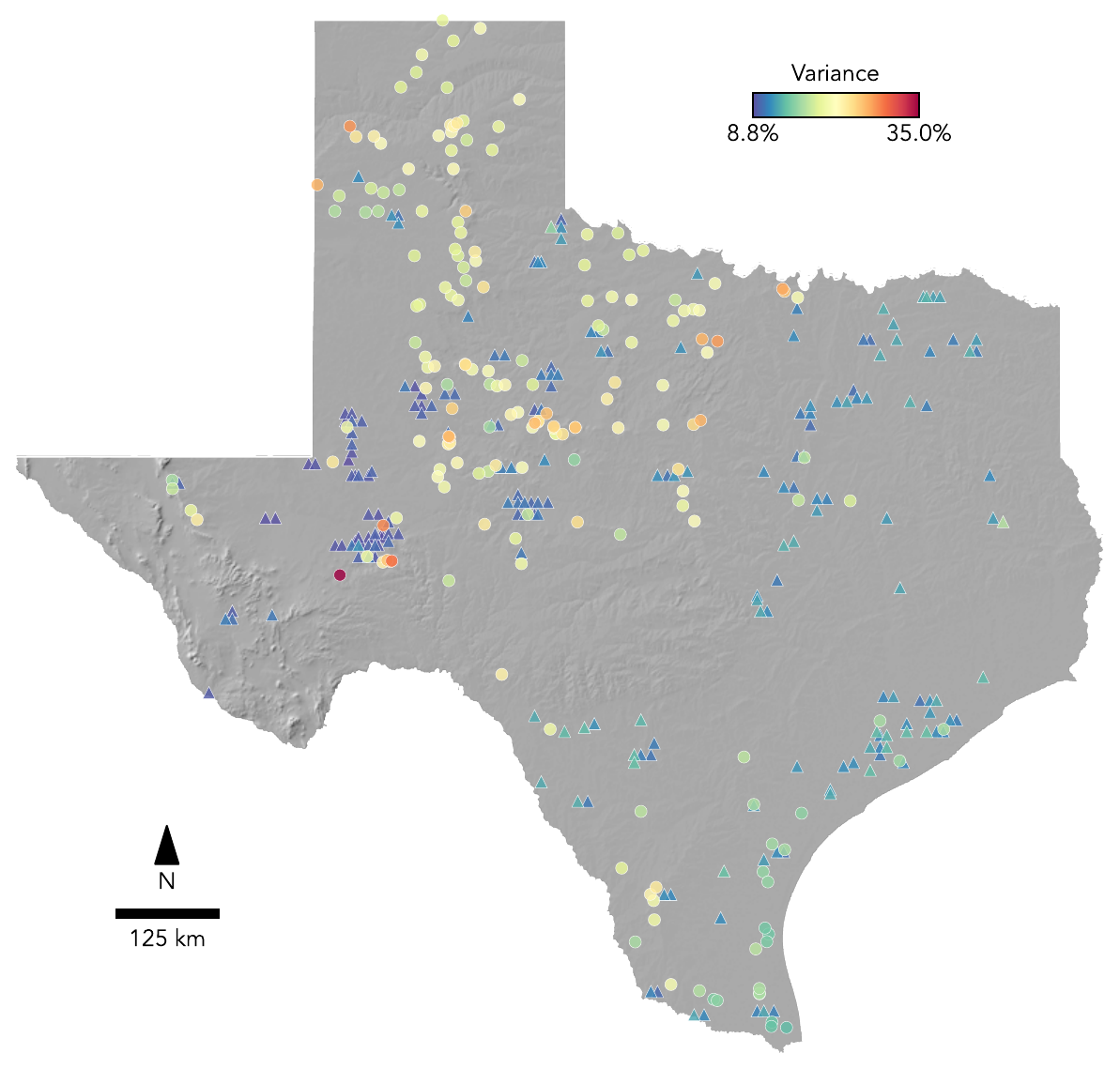}
            \put(0, 90){\textbf{e}}
        \end{overpic}
    \end{minipage}
    \begin{minipage}{.45\textwidth} 
        \centering
        \begin{overpic}[width=.925\textwidth, trim = {0cm, 0cm, 0cm, 0cm}, clip]{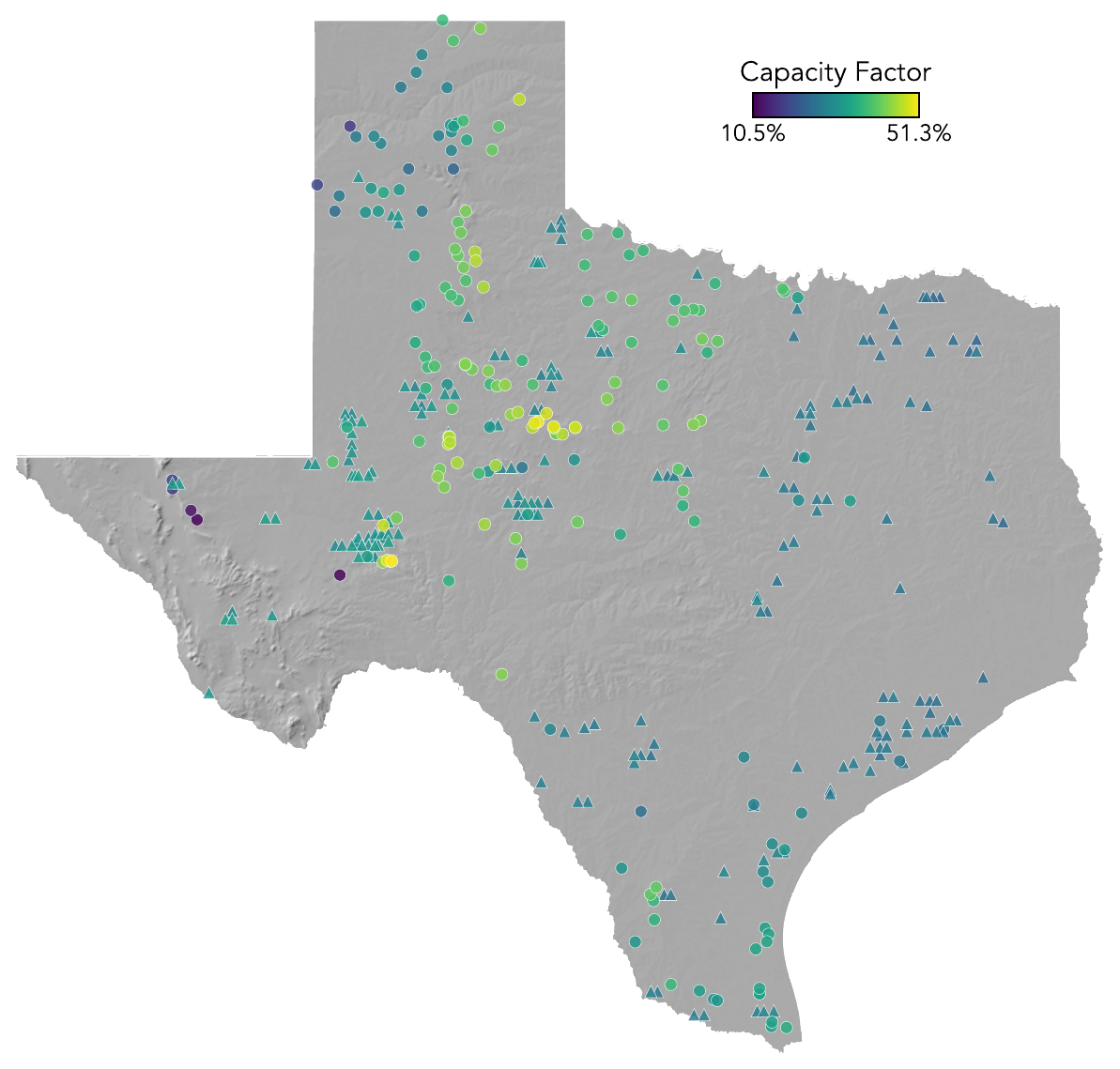}
            \put(0, 90){\textbf{f}}
        \end{overpic}
        \begin{overpic}[width=1\textwidth, trim = {0cm, -3.75cm, 0cm, -2.5cm}, clip]{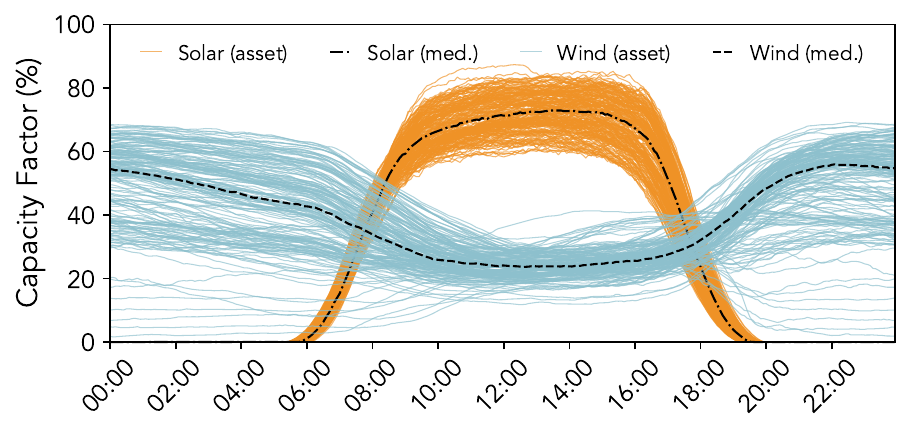}
          \put(-.5, 70){\textbf{g}}
        \end{overpic}
    \end{minipage}
    \caption{\textbf{Average capacity factor (CF), forecasting error, and power capacity for the wind and solar assets in the Texas-7k grid.} Relations between: (a) power capacity and RMSE shaded by average CF, (b) average CF and RMSE shaded by power capacity, and (c) average CF and power capacity shaded by RMSE. (d, g) Average hourly CF of all generation curves in each wind farm, (brighter indicates a higher average CF). (f) DA forecast hourly root mean squared error (RMSE) for each wind farm (brighter color, greater RMSE). (e) AC Installed power capacity (i.e.,~nameplate capacity) in each wind farm, darker represents larger power capacity.}
    \label{fig:dataset}
\end{figure}

\begin{figure}[htb!]
    \centering
       \begin{overpic}[width=.8\textwidth, trim = {0cm, 0cm, 0cm, 0cm}, clip]{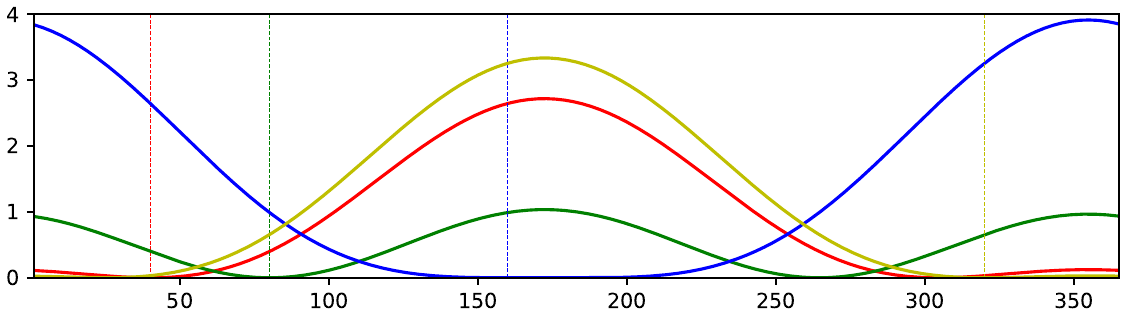}
       \end{overpic}
    \caption{ \textbf{Sun elevation distance} $\|d_\star -d \|_S$ in Eq.~\eqref{eq:elevation_dist} as a function of $d$. We illustrate with $d_\star=40$ (Feb 9) in red, $d_\star=80$ (March 21) in green, $d_\star=160$ (June 9) in blue, and $d_\star=320$ (Nov 16) in yellow.}
    \label{fig:SED}
\end{figure}

\begin{figure}[htb!]
    \centering
       \begin{overpic}[width=.8\textwidth, trim = {0cm, 0cm, 0cm, 0cm}, clip]{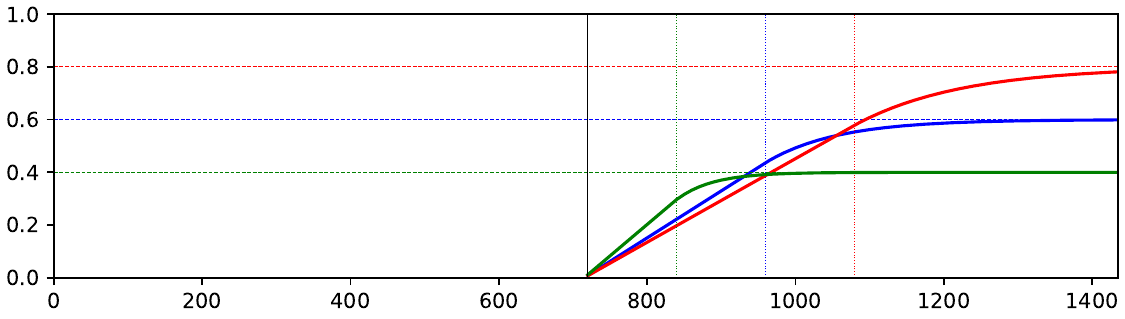}
       \end{overpic}
    \caption{Examples of the \textbf{Linear inverse exponential (LIE) function} at 12pm ($\tau=720$) $\delta_{\vartheta,\nu, \eta}(s)$ for $s>720$. We illustrate with 3 parameter configurations $\vartheta= 2,\nu = 0.4, \eta = 2.5$ in green, $\vartheta = 4,\nu = 0.6, \eta = 2.5$ in blue, and $\vartheta = 6,\nu = 0.8, \eta = 2.5$ in red. The dashed horizontal lines indicate the horizontal asymptotes at respective $\nu$ levels.}
    \label{fig:LIE}
\end{figure}

\clearpage

\begin{figure}[htb!]
    \centering
    \includegraphics[width=.45\textwidth, trim = {0cm, 0cm, 0cm, 0cm}, clip]{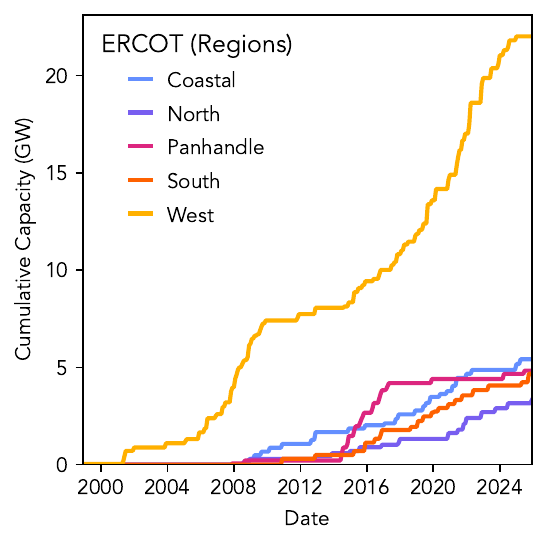}
    \caption{\textbf{Wind generation capacity in ERCOT.} Monthly operating wind capacity aggregated by commissioning vintage from Sep 1999 to May 2025 in the five regions that comprise ERCOT's grid. \label{fig:ercot-monthly-capacity}}
\end{figure}

\clearpage



\section{Supplementary Tables}

\begin{table}[htb!]
    \centering
    \scriptsize
    \caption{\textbf{$D_{KS}$ scores for probabilistic revisions of wind asset generation with de-biased DA forecasts.} Kolmogorov--Smirnov (KS) statistics of the PIT evaluated across each two-hour lead-time interval $s \in [T_0, T_1]$ over the forecast horizon $\mathcal{I}_{>\tau}$ at update times $\tau=$~6am, 12pm, 6pm. We compare our algorithm with two non-functional baselines: ARIMA, and Bootstrapped Re-anchored DA (BRDA) forecast only. Lower is better.}
    \begin{adjustbox}{max width=\linewidth}
        \renewcommand{\arraystretch}{1.25}

\begin{tabular}{rr ccc ccc ccc}
    \toprule
    \multicolumn{2}{c}{\textbf{Lead Interval}}
      & \multicolumn{3}{c}{\textbf{Our Algorithm}}
      & \multicolumn{3}{c}{\textbf{ARIMA}}
      & \multicolumn{3}{c}{\textbf{BRDA}} \\
    \cmidrule(lr){1-2} \cmidrule(lr){3-5} \cmidrule(lr){6-8} \cmidrule(lr){9-11}
    $T_{0}$ & $T_{1}$
      & \multicolumn{3}{c}{$\boldsymbol{\tau}$}
      & \multicolumn{3}{c}{$\boldsymbol{\tau}$}
      & \multicolumn{3}{c}{$\boldsymbol{\tau}$} \\
    \cmidrule(lr){3-5} \cmidrule(lr){6-8} \cmidrule(lr){9-11}
      &      & 6am & 12pm & 6pm & 6am & 12pm & 6pm & 6am & 12pm & 6pm \\
    \midrule
    6am  & 8am  & 0.049 & ---   & ---   & 0.193 & ---   & ---   & 0.034 &   --- &   --- \\
    8am  & 10am & 0.028 & ---   & ---   & 0.147 & ---   & ---   & 0.056 &   --- &   --- \\
    10am & 12pm & 0.026 & ---   & ---   & 0.132 & ---   & ---   & 0.081 &   --- &   --- \\
    12pm & 2pm  & 0.041 & 0.018 & ---   & 0.122 & 0.180 & ---   & 0.073 & 0.092 &   --- \\
    2pm  & 4pm  & 0.026 & 0.037 & ---   & 0.103 & 0.087 & ---   & 0.087 & 0.098 &   --- \\
    4pm  & 6pm  & 0.025 & 0.020 & ---   & 0.069 & 0.082 & ---   & 0.068 & 0.072 &   --- \\
    6pm  & 8pm  & 0.027 & 0.014 & 0.013 & 0.032 & 0.068 & 0.108 & 0.040 & 0.041 & 0.039 \\
    8pm  & 10pm & 0.031 & 0.023 & 0.018 & 0.102 & 0.185 & 0.214 & 0.050 & 0.049 & 0.045 \\
    10pm & 12am & 0.030 & 0.042 & 0.031 & 0.115 & 0.172 & 0.188 & 0.054 & 0.051 & 0.052 \\
    \bottomrule
\end{tabular}

    \end{adjustbox}
    \label{tab:wind_asset-ks-unbiased}
\end{table}

\clearpage

\begin{table}[htb!]
    \centering
    \caption{\textbf{Asset-level probabilistic revision tuning hyperparameters.} Estimated optimal hyperparameters for the wind assets under de-biased and raw DA forecasts at update times $\tau=$~6am, 12pm, and 6pm, and solar assets under de-biased DA forecasts (at $\tau=$~10am, 11am, 12pm, and 2pm). The upper (resp.~lower) block lists the hyperparameters fine-tuned via BO (resp.~held fixed across $\tau$).}
    \begin{adjustbox}{max width=\linewidth}
        \setlength{\tabcolsep}{2.5pt} 
        \renewcommand{\arraystretch}{1.25}
        \begin{tabular}{rc ccc ccc cccc}
    \toprule
    & & \multicolumn{3}{c}{\textbf{Wind (de-biased)}}
      & \multicolumn{3}{c}{\textbf{Wind (as-is)}}
      & \multicolumn{4}{c}{\textbf{Solar (de-biased)}} \\
    \cmidrule(lr){3-5} \cmidrule(lr){6-8} \cmidrule(lr){9-12}
    \textbf{Tuning Hyperparameter} & &
    \textbf{6am} & \textbf{12pm} & \textbf{6pm} &
    \textbf{6am} & \textbf{12pm} & \textbf{6pm} &
    \textbf{10am} & \textbf{11pm} & \textbf{12pm} & \textbf{2pm} \\
    \midrule
    Trust level                          & $\delta$
        & 9.7\% & 20.2\% & 0\%
        & 10\% & 10.9\% & 19.7\%
        & 15\% & 7.1\% & 7.7\%  & 0\% \\
    Trust horizon                        & $\vartheta$
        & 6~h & 10~h & 12~h
        & 1~h & 2~h  & 4~h
        & 7~h & 12~h & 7~h & 10~h \\
    DA forecast similarity weight        & $\lambda_e$
        & 21.8 & 11.4  & 13.5
        & 1.36 & 0.21  & 0.36
        & 2.47 & 1.56 & 0.003 & 0.009 \\
    Seasonal similarity weight           & $\lambda_d$
        & 4.38  & 0.   & 0.98
        & 0.15  & 0.83 & 0.17
        & 0.853 & 1.17 & 43.4 & 90.6 \\
    kNN step I neighborhood size         & $\kappa_0$
        & 150   & 1,815 & 150
        & 150   & 150   & 2,494
        & 2,315 & 1,139 & 298 & 2,042 \\
    \midrule
    Partially observed curve similarity weight & $\lambda_f$
        & \multicolumn{6}{c}{40} & \multicolumn{4}{c}{100} \\
    \cmidrule(lr){3-8} \cmidrule(lr){9-12}
    Probability of fusion                      & $p_{\mathrm{fusion}}$
        & \multicolumn{6}{c}{0.75} & \multicolumn{4}{c}{0.9} \\
    \cmidrule(lr){3-8} \cmidrule(lr){9-12}
    Partially observed curve forgetting rate & $\varepsilon_f$
        & \multicolumn{10}{c}{0.125~h} \\
    DA forecast past decay rate              & $\varepsilon_{e^-}$
        & \multicolumn{10}{c}{1~h} \\
    DA forecast future decay rate                     & $\varepsilon_{e^+}$
        & \multicolumn{10}{c}{4~h} \\
    kNN step II neighborhood size              & $\kappa$
        & \multicolumn{10}{c}{150} \\
    \bottomrule
\end{tabular}
    \end{adjustbox}
    \label{tab:asset_hyperparameters}
\end{table}

\clearpage

\begin{table}[htb!]
    \centering
    \scriptsize
    \caption{\textbf{Solar asset  intra-day point forecast metrics.} Root mean squared error (RMSE), mean absolute error (MAE), and mean bias error (MBE) of the point forecasts at forecast update times $\tau \in \{120, 132, 144, 168\}$ (i.e.,~$\tau=$~10am, 11am, 12pm, and 2pm). The estimators are the weighted median curve $\hat{\mu}_{\text{med}}(s)$, the deepest curve $\hat{\mu}_{\text{deep}}(s)$ ranked by Modified Band Depth, and the focal curve $\hat{\mu}_{\text{focal}}(s)$. Errors are computed over the forecast horizon $\mathcal{I}_{>\tau}$ and expressed in capacity factor units. Updating is based on fusing with de-biased DA forecasts.}
    \setlength{\tabcolsep}{6pt} 
    \renewcommand{\arraystretch}{1.25}
    \begin{tabular}{ll rrr}
    \toprule
    $\boldsymbol{\tau}$ & \textbf{Estimator} & \textbf{RMSE} & \textbf{MAE} & \textbf{MBE} \\
    \midrule    
    \multirow[c]{3}{*}{10am}
      & $\hat{\mu}_{\mathrm{mar}}$   & 0.143          & \textbf{0.087} & \textbf{\phantom{$-$}0.002} \\
      & $\hat{\mu}_{\mathrm{deep}}$  & 0.146          & 0.089          & \phantom{$-$}0.003 \\
      & $\hat{\mu}_{\mathrm{focal}}$ & \textbf{0.142} & 0.092          & \phantom{$-$}0.004 \\
    \midrule
    \multirow[c]{3}{*}{11am}
      & $\hat{\mu}_{\mathrm{mar}}$   & 0.131          & \textbf{0.078} & $-$0.002 \\
      & $\hat{\mu}_{\mathrm{deep}}$  & 0.134          & 0.079          & \textbf{\phantom{$-$}0.000} \\
      & $\hat{\mu}_{\mathrm{focal}}$ & \textbf{0.130} & 0.082          & \phantom{$-$}0.001 \\
    \midrule
    \multirow[c]{3}{*}{12pm}
      & $\hat{\mu}_{\mathrm{mar}}$   & \textbf{0.115} & \textbf{0.065} & $-$0.008 \\
      & $\hat{\mu}_{\mathrm{deep}}$  & 0.117          & 0.067          & $-$0.007 \\
      & $\hat{\mu}_{\mathrm{focal}}$ & 0.117          & 0.070          & \textbf{\phantom{$-$}0.002} \\
    \midrule
    \multirow[c]{3}{*}{2pm}
      & $\hat{\mu}_{\mathrm{mar}}$   & 0.101          & 0.053          & $-$0.006 \\
      & $\hat{\mu}_{\mathrm{deep}}$  & \textbf{0.099} & \textbf{0.052} & $-$0.005 \\
      & $\hat{\mu}_{\mathrm{focal}}$ & 0.101          & 0.056          & \textbf{\phantom{$-$}0.001} \\
    \bottomrule
\end{tabular}
    \label{tab:solar_asset-error-unbiased}
\end{table}

\clearpage

\begin{table}[htb!]
    \centering
    \caption{\textbf{Solar asset unbiased DA forecast probabilistic revision confidence band scores.} Functional interval score (FIS), functional coverage score (FCS), and simultaneous coverage probability (SCP) at significance levels $\alpha \in \{0.1, 0.2\}$ and forecast update times $\tau \in \{120, 132, 144, 168\}$ (i.e., $\tau=$~10am, 11am, 12pm, and 2pm). The confidence regions are constructed from the adjusted functional envelope $\mathcal{R}^{\mathrm{adj}}_\alpha$, the projection envelope $\mathcal{R}^{\mathrm{prj}}_\alpha$, the unadjusted functional envelope  $\mathcal{R}^{\mathrm{fun}}_\alpha$, and the weighted marginal quantiles $\mathcal{R}^{\mathrm{mar}}_\alpha$. For each construction, we report the metric ``score'' and the fraction $k_\alpha$ of fused forecasts contained in the band.}
    \begin{adjustbox}{max width=\linewidth}
        \setlength{\tabcolsep}{2.5pt} 
        \renewcommand{\arraystretch}{1.25}
        \begin{tabular}{ll cccccc cccccc cccccc}
    \toprule
    & & \multicolumn{6}{c}{$\boldsymbol{\overline{\mathrm{FIS}}({\alpha})}$}
      & \multicolumn{6}{c}{$\boldsymbol{\overline{\mathrm{FCS}}({\alpha})}$}
      & \multicolumn{6}{c}{$\boldsymbol{\overline{\mathrm{SCP}}({\alpha})}$} \\
    \cmidrule(lr){3-8} \cmidrule(lr){9-14} \cmidrule(lr){15-20}
    & & \multicolumn{2}{c}{$\boldsymbol{\mathcal{R}^{\mathrm{adj}}_{\alpha}}$}
      & \multicolumn{2}{c}{$\boldsymbol{\mathcal{R}^{\mathrm{prj}}_{\alpha}}$}
      & $\boldsymbol{\mathcal{R}^{\mathrm{fun}}_{\alpha}}$ & $\boldsymbol{\mathcal{R}^{\mathrm{mar}}_{\alpha}}$ 
      & \multicolumn{2}{c}{$\boldsymbol{\mathcal{R}^{\mathrm{adj}}_{\alpha}}$}
      & \multicolumn{2}{c}{$\boldsymbol{\mathcal{R}^{\mathrm{prj}}_{\alpha}}$}
      & $\boldsymbol{\mathcal{R}^{\mathrm{fun}}_{\alpha}}$ & $\boldsymbol{\mathcal{R}^{\mathrm{mar}}_{\alpha}}$ 
      & \multicolumn{2}{c}{$\boldsymbol{\mathcal{R}^{\mathrm{adj}}_{\alpha}}$}
      & \multicolumn{2}{c}{$\boldsymbol{\mathcal{R}^{\mathrm{prj}}_{\alpha}}$}
      & $\boldsymbol{\mathcal{R}^{\mathrm{fun}}_{\alpha}}$ & $\boldsymbol{\mathcal{R}^{\mathrm{mar}}_{\alpha}}$ \\
    \cmidrule(lr){3-4} \cmidrule(lr){5-6} \cmidrule(lr){7-7} \cmidrule(lr){8-8}
    \cmidrule(lr){9-10} \cmidrule(lr){11-12} \cmidrule(lr){13-13} \cmidrule(lr){14-14}
    \cmidrule(lr){15-16} \cmidrule(lr){17-18} \cmidrule(lr){19-19} \cmidrule(lr){20-20}
    $\boldsymbol{\tau}$ & $\boldsymbol{\alpha}$
    & score & $k_\alpha$ & score & $k_\alpha$ & score & score
    & score & $k_\alpha$ & score & $k_\alpha$ & score & score
    & score & $k_\alpha$ & score & $k_\alpha$ & score & score \\
    \midrule
    \multirow[c]{2}{*}{10am} & 0.1 & 0.493 & 0.80 & 0.573 & 1.00 & 0.495 & \textbf{0.490} & \textbf{0.902} & 0.60 & 0.891 & 0.95 & 0.970 & 0.912 & \textbf{0.651} & 1.00 & 0.312 & 1.00 & 0.514 & 0.243 \\
     & 0.2 & 0.419 & 0.55 & 0.455 & 0.90 & 0.446 & \textbf{0.408} & \textbf{0.796} & 0.32 & 0.795 & 0.50 & 0.951 & 0.849 & \textbf{0.651} & 1.00 & 0.312 & 1.00 & 0.420 & 0.129 \\
    \midrule
    \multirow[c]{2}{*}{11am} & 0.1 & 0.447 & 0.85 & 0.545 & 1.00 & 0.450 & \textbf{0.446} & 0.887 & 0.55 & \textbf{0.899} & 1.00 & 0.969 & 0.915 & \textbf{0.624} & 1.00 & 0.270 & 1.00 & 0.498 & 0.241 \\
     & 0.2 & 0.379 & 0.60 & 0.418 & 0.95 & 0.400 & \textbf{0.369} & 0.781 & 0.28 & \textbf{0.799} & 0.60 & 0.950 & 0.852 & \textbf{0.624} & 1.00 & 0.270 & 1.00 & 0.391 & 0.116 \\
    \midrule
    \multirow[c]{2}{*}{12pm} & 0.1 & \textbf{0.380} & 0.80 & 0.471 & 1.00 & 0.386 & 0.384 & 0.891 & 0.50 & \textbf{0.892} & 0.95 & 0.971 & 0.921 & \textbf{0.615} & 1.00 & 0.342 & 1.00 & 0.474 & 0.236 \\
     & 0.2 & 0.318 & 0.50 & 0.371 & 0.75 & 0.343 & \textbf{0.315} & \textbf{0.776} & 0.22 & \textbf{0.776} & 0.38 & 0.956 & 0.866 & \textbf{0.615} & 1.00 & 0.342 & 1.00 & 0.379 & 0.118 \\
    \midrule
    \multirow[c]{2}{*}{2pm} & 0.1 & \textbf{0.313} & 0.85 & 0.361 & 1.00 & 0.315 & 0.336 & \textbf{0.903} & 0.50 & 0.885 & 0.65 & 0.976 & 0.914 & \textbf{0.722} & 1.00 & 0.463 & 1.00 & 0.557 & 0.229 \\
     & 0.2 & \textbf{0.263} & 0.55 & 0.287 & 0.75 & 0.281 & 0.265 & \textbf{0.800} & 0.22 & 0.782 & 0.25 & 0.963 & 0.866 & \textbf{0.722} & 1.00 & 0.463 & 1.00 & 0.456 & 0.109 \\
    \bottomrule
\end{tabular}
    \end{adjustbox}
    \label{tab:solar_unbiased_asset_envelope}
\end{table}

\clearpage

\begin{table}[htb!]
    \centering
    \scriptsize
    \caption{\textbf{$D_{KS}$ scores for probabilistic revisions of solar asset generation  with de-biased DA forecasts.} Kolmogorov--Smirnov (KS) statistic of the PIT evaluated across each 1-hour interval of lead times $s \in [T_0, T_1]$ increments over the forecast horizon $\mathcal{I}_{>\tau}$ at forecast update times $\tau=$~10am, 12pm, 2pm. Lower is better.}
    \setlength{\tabcolsep}{6pt} 
    \renewcommand{\arraystretch}{1.25}
    \begin{tabular}{rrccc}
    \toprule
    \multicolumn{2}{c}{\textbf{Lead Interval}} & \multicolumn{3}{c}{\textbf{Update Time} $\boldsymbol{\tau}$} \\
    \cmidrule(lr){1-2} \cmidrule(lr){3-5} 

    $T_{0}$ & $T_{1}$ &    10am &   12pm &   2pm \\
    \midrule
      10am & 11pm & 0.022 &   --- &   --- \\
      11pm & 12pm & 0.033 &   --- &   --- \\
      12pm &  1pm & 0.028 & 0.035 &   --- \\
      1pm  &  2pm & 0.029 & 0.027 &   --- \\
      2pm  &  3pm & 0.036 & 0.039 & 0.042 \\
      3pm  &  4pm & 0.095 & 0.086 & 0.041 \\
    \bottomrule
\end{tabular}

    \label{tab:solar_asset-ks-unbiased}
\end{table}

\clearpage

\begin{table}[htb!]
    \centering
    \caption{\textbf{Tuned hyperparameters for wind zonal generation with 2DA de-biased forecasts.} Estimated optimal hyperparameters for update times $\tau=12$am, 12pm+1D, and 12pm+1D, across zonal neighborhoods $\cN_k, k=0,1,2$. The upper block list the hyperparameters fine-tuned across forecast update times $\tau$ via BO. The lower block lists the hyperparameters held fixed.}
    \begin{adjustbox}{max width=\linewidth}
        \setlength{\tabcolsep}{2.5pt} 
        \renewcommand{\arraystretch}{1.25}
\begin{tabular}{rc ccc ccc ccc}
    \toprule
    & & \multicolumn{3}{c}{$\boldsymbol{\cN^{(0)}}$}
      & \multicolumn{3}{c}{$\boldsymbol{\cN^{(1)}}$}
      & \multicolumn{3}{c}{$\boldsymbol{\cN^{(2)}}$} \\
    \cmidrule(lr){3-5} \cmidrule(lr){6-8} \cmidrule(lr){9-11}
    \textbf{Tuning Hyperparameter} & &
    \textbf{12pm} & \textbf{12am+1D} & \textbf{12pm+1D} &
    \textbf{12pm} & \textbf{12am+1D} & \textbf{12pm+1D} &
    \textbf{12pm} & \textbf{12am+1D} & \textbf{12pm+1D} \\
    \midrule
    Trust level                          & $\delta$
        & 58.2\% & 43.7\% & 0.0\%
        & 62.3\% & 45.5\% & 40.6\%
        & 44.8\% & 39.9\% & 30.8\% \\
    Trust horizon                        & $\vartheta$
        & 2~h & 2~h & 13~h
        & 2~h & 2~h & 3~h
        & 2~h & 2~h & 2~h \\
    2DA forecast similarity weight        & $\lambda_e$
        & 10   & 10 & 1.87
        & 6.78 & 10 & 10
        & 10   & 10 & 10 \\
    Seasonal similarity weight           & $\lambda_d$
        & 0.0855 & 0.0533 & 0.0004
        & 0.6193 & 0.0009 & 0.0007
        & 0.1292 & 0.0152 & 0.0001 \\
    \midrule
    Partially observed curve forgetting rate            & $\varepsilon_f$
        & \multicolumn{9}{c}{0.5~h} \\
    2DA forecast past decay rate                     & $\varepsilon_{e^-}$
        & \multicolumn{9}{c}{1~h} \\
    2DA forecast future decay rate                     & $\varepsilon_{e^+}$
        & \multicolumn{9}{c}{5~h} \\
    Partially observed curve similarity weight & $\lambda_f$
        & \multicolumn{9}{c}{10} \\
    kNN step II neighborhood size              & $\kappa$
        & \multicolumn{9}{c}{150} \\
    Probability of fusion                      & $p_{\mathrm{fusion}}$
        & \multicolumn{9}{c}{1} \\
    \bottomrule
\end{tabular}

    \end{adjustbox}
    \label{tab:wind_zone-hyperparameters-unbiased}
\end{table}

\clearpage

\begin{table}[htb!]
    \centering
    \scriptsize
    \caption{\textbf{Wind zonal generation point forecast metrics.} Root mean squared error (RMSE), mean absolute error (MAE), and mean bias error (MBE) at forecast update times $\tau \in \{12, 24, 36\}$ (i.e.,~12pm, 12am+1D, and 12pm+1D), averaged over the forecast horizon $\mathcal{I}_{>\tau}$ and expressed in capacity factor units. The point estimators are the weighted median curve $\hat{\mu}_{\text{med}}$, the deepest curve $\hat{\mu}_{\text{deep}}$ ranked by Modified Band Depth, and the focal curve $\hat{\mu}_{\text{focal}}$.  }
    \setlength{\tabcolsep}{4pt} 
    \renewcommand{\arraystretch}{1.25}

\begin{tabular}{ll rrr}
    \toprule
    $\boldsymbol{\tau}$ & \textbf{Estimator} & \textbf{RMSE} & \textbf{MAE} & \textbf{MBE} \\
    \midrule
    \multirow[c]{3}{*}{12pm} & $\hat{\mu}_{mar}(s)$ & 0.131 & 0.106 & 0.014 \\
     & $\hat{\mu}_{deep}(s)$ & 0.142 & 0.115 & 0.012 \\
     & $\hat{\mu}_{focal}(s)$ & \textbf{0.128} & \textbf{0.105} & \textbf{0.004} \\
     \midrule
    \multirow[c]{3}{*}{12am+1D} & $\hat{\mu}_{mar}(s)$ & 0.106 & \textbf{0.088} & 0.007 \\
     & $\hat{\mu}_{deep}(s)$ & 0.116 & 0.096 & 0.008 \\
     & $\hat{\mu}_{focal}(s)$ & \textbf{0.105} & \textbf{0.088} & \textbf{0.001} \\
     \midrule
    \multirow[c]{3}{*}{12pm+1D} & $\hat{\mu}_{mar}(s)$ & 0.088 & 0.074 & 0.005 \\
     & $\hat{\mu}_{deep}(s)$ & 0.094 & 0.079 & 0.007 \\
     & $\hat{\mu}_{focal}(s)$ & \textbf{0.086} & \textbf{0.073} & \textbf{0.0} \\
    \bottomrule
\end{tabular}
    \label{tab:wind_zone-error_c1-unbiased}
\end{table}

\clearpage

\begin{table}[htb!]
    \centering
    \caption{\textbf{Wind zonal generation confidence band scores with $1^{st}$ order neighborhood $\cN^{(1)}$.} Functional interval score (FIS), functional coverage score (FCS), and simultaneous coverage probability (SCP) at significance levels $\alpha \in \{0.1, 0.2\}$ and forecast update times $\tau \in \{12, 24, 36\}$ (i.e.,~12pm, 12am+1D, and 12pm+1D). The confidence regions are constructed from the adjusted functional envelope $\mathcal{R}^{\mathrm{adj}}_\alpha$, the projection envelope $\mathcal{R}^{\mathrm{prj}}_\alpha$, the unadjusted functional envelope  $\mathcal{R}^{\mathrm{fun}}_\alpha$, and the weighted marginal quantiles $\mathcal{R}^{\mathrm{mar}}_\alpha$. For each construction, we report the metric ``score'' and the fraction $k_\alpha$ of fused forecasts contained in the band.}
    \begin{adjustbox}{max width=\linewidth}
        \setlength{\tabcolsep}{2.5pt} 
        \renewcommand{\arraystretch}{1.25}

\begin{tabular}{ll cccccc cccccc cccccc}
    \toprule
    & & \multicolumn{6}{c}{$\boldsymbol{\overline{\mathrm{FIS}}({\alpha})}$}
      & \multicolumn{6}{c}{$\boldsymbol{\overline{\mathrm{FCS}}({\alpha})}$}
      & \multicolumn{6}{c}{$\boldsymbol{\overline{\mathrm{SCP}}({\alpha})}$} \\
    \cmidrule(lr){3-8} \cmidrule(lr){9-14} \cmidrule(lr){15-20}
    & & \multicolumn{2}{c}{$\boldsymbol{\mathcal{R}^{\mathrm{adj}}_{\alpha}}$}
      & \multicolumn{2}{c}{$\boldsymbol{\mathcal{R}^{\mathrm{prj}}_{\alpha}}$}
      & $\boldsymbol{\mathcal{R}^{\mathrm{fun}}_{\alpha}}$ & $\boldsymbol{\mathcal{R}^{\mathrm{mar}}_{\alpha}}$ 
      & \multicolumn{2}{c}{$\boldsymbol{\mathcal{R}^{\mathrm{adj}}_{\alpha}}$}
      & \multicolumn{2}{c}{$\boldsymbol{\mathcal{R}^{\mathrm{prj}}_{\alpha}}$}
      & $\boldsymbol{\mathcal{R}^{\mathrm{fun}}_{\alpha}}$ & $\boldsymbol{\mathcal{R}^{\mathrm{mar}}_{\alpha}}$ 
      & \multicolumn{2}{c}{$\boldsymbol{\mathcal{R}^{\mathrm{adj}}_{\alpha}}$}
      & \multicolumn{2}{c}{$\boldsymbol{\mathcal{R}^{\mathrm{prj}}_{\alpha}}$}
      & $\boldsymbol{\mathcal{R}^{\mathrm{fun}}_{\alpha}}$ & $\boldsymbol{\mathcal{R}^{\mathrm{mar}}_{\alpha}}$ \\
    \cmidrule(lr){3-4} \cmidrule(lr){5-6} \cmidrule(lr){7-7} \cmidrule(lr){8-8}
    \cmidrule(lr){9-10} \cmidrule(lr){11-12} \cmidrule(lr){13-13} \cmidrule(lr){14-14}
    \cmidrule(lr){15-16} \cmidrule(lr){17-18} \cmidrule(lr){19-19} \cmidrule(lr){20-20}
    $\boldsymbol{\tau}$ & $\boldsymbol{\alpha}$
    & score & $k_\alpha$ & score & $k_\alpha$ & score & score
    & score & $k_\alpha$ & score & $k_\alpha$ & score & score
    & score & $k_\alpha$ & score & $k_\alpha$ & score & score \\
    \midrule
    \multirow[c]{2}{*}{12pm} & 0.1 & 0.546 & 0.55 & 0.548 & 0.70 & 0.571 & \textbf{0.497} & 0.891 & 0.55 & 0.894 & 0.65 & 0.965 & \textbf{0.898} & \textbf{0.760} & 1.00 & 0.691 & 1.00 & 0.649 & 0.322 \\
    & 0.2 & 0.478 & 0.32 & 0.474 & 0.38 & 0.541 & \textbf{0.439} & 0.787 & 0.30 & \textbf{0.793} & 0.35 & 0.950 & \textbf{0.807} & \textbf{0.760} & 1.00 & 0.691 & 1.00 & 0.554 & 0.137 \\
    \midrule
    \multirow[c]{2}{*}{12am+1D} & 0.1 & 0.479 & 0.55 & 0.467 & 0.65 & 0.510 & \textbf{0.432} & 0.902 & 0.60 & \textbf{0.901} & 0.65 & 0.966 & 0.898 & \textbf{0.835} & 1.00 & 0.778 & 1.00 & 0.750 & 0.430 \\
    & 0.2 & 0.410 & 0.30 & 0.399 & 0.32 & 0.478 & \textbf{0.374} & \textbf{0.799} & 0.35 & 0.793 & 0.35 & 0.950 & 0.811 & \textbf{0.789} & 0.95 & 0.778 & 1.00 & 0.671 & 0.219 \\
    \midrule
    \multirow[c]{2}{*}{12pm+1D} & 0.1 & 0.414 & 0.70 & 0.410 & 0.90 & 0.418 & \textbf{0.389} & \textbf{0.897} & 0.70 & 0.892 & 0.75 & 0.950 & 0.875 & \textbf{0.862} & 1.00 & 0.829 & 1.00 & 0.760 & 0.541 \\
    & 0.2 & 0.358 & 0.40 & 0.349 & 0.50 & 0.384 & \textbf{0.329} & \textbf{0.789} & 0.45 & 0.783 & 0.45 & 0.925 & 0.783 & \textbf{0.760} & 0.90 & 0.769 & 0.95 & 0.668 & 0.331 \\
    \bottomrule
\end{tabular}
    \end{adjustbox}
    \label{tab:wind_zone-envelope_c1-unbiased}
\end{table}

\end{document}